\documentclass[hidelinks,onefignum,onetabnum]{siamart250106} 

\usepackage{lipsum}
\usepackage{amsfonts}
\usepackage{graphicx}
\usepackage{mathtools}
\usepackage{epstopdf}
\usepackage{enumitem}
\usepackage{tikz}
\usepackage{algorithmic}
\ifpdf
  \DeclareGraphicsExtensions{.eps,.pdf,.png,.jpg}
\else
  \DeclareGraphicsExtensions{.eps}
\fi

\newsiamremark{remark}{Remark}
\newsiamremark{hypothesis}{Hypothesis}
\crefname{hypothesis}{Hypothesis}{Hypotheses}
\newsiamthm{claim}{Claim}

\headers{A Unified Perspective on the Dynamics of Transformers}{V. Castin, P. Ablin, J. A. Carrillo, G. Peyré}

\title{A Unified Perspective \\ on the Dynamics of Deep Transformers
}

\author{
Valérie Castin\thanks{CNRS and Ecole Normale Supérieure PSL, Paris, France. Emails: valerie.castin@ens.psl.eu (corresponding author), gabriel.peyre@ens.psl.eu.}
\and
Pierre Ablin\thanks{Apple, Paris, France. Email: p\_ablin@apple.com.}
\and José A. Carrillo\thanks{Mathematical Institute, University of Oxford, Oxford OX2 6GG, UK. Email: carrillo@maths.ox.ac.uk.}
\and Gabriel Peyré\footnotemark[1]}

\usepackage{amsopn}
\DeclareMathOperator{\diag}{diag}

\DeclareMathOperator\supp{\mathrm{Supp}}

\DeclareMathOperator\tr{\mathrm{Tr}}
\DeclareMathOperator\var{\mathrm{Var}}

\newcommand{\N}{\mathbb{N}}
\newcommand{\dd}{\,\mathrm{d}}

\newcommand{\R}{\mathbb{R}}

\newcommand{\ccal}{{\cal C}}
\newcommand{\scal}{{\cal S}}

\newcommand{\pcal}{{\cal P}}
\newcommand{\bcal}{{\cal B}}
\newcommand{\xcal}{{\cal X}}
\newcommand{\ycal}{{\cal Y}}

\newcommand{\dcal}{{\cal D}}
\newcommand{\fcal}{{\cal F}}
\newcommand{\ncal}{{\cal N}}

\newcommand{\acal}{{\cal A}}

\newcommand{\olsi}[1]{\,\overline{\!{#1}}}
\newcommand{\bbar}{\olsi{B}}

\newcommand{\lcal}{{\cal L}}
\newcommand{\qcal}{{\cal Q}}

\newcommand{\beq}{\begin{equation}}
\newcommand{\eeq}{\end{equation}}
\newcommand{\norm}[1]{\left \lVert #1 \right \rVert}
\newcommand{\prt}[1]{\left ( #1 \right )}

\DeclareRobustCommand{\revision}[1]{\textcolor{blue}{#1}}
\DeclareRobustCommand{\revision}[1]{#1}

\newcommand{\modu}[1]{\left \lvert #1 \right \rvert }

\DeclareMathOperator\softmax{\mathrm{softmax}}
\DeclareMathOperator\lip{\mathrm{Lip}}

\DeclareMathOperator\KL{\mathrm{KL}}

\newcommand{\satt}[1]{^{(\mathrm{#1})}}

\newcommand*{\dbbar}[1]{\bar{\bar{#1}}}

\newsiamthm{defi}{Definition}
\newsiamthm{prop}{Proposition}

\ifpdf
\hypersetup{
  pdftitle={A Unified Perspective on the Dynamics of Deep Transformers},
  pdfauthor={V. Castin, P. Ablin, J. A. Carrillo, G. Peyré}
}
\fi

\begin{document}

\maketitle

\begin{abstract}
    Transformers, which are state-of-the-art in most machine learning tasks, represent the data as sequences of vectors called tokens.
    This representation is then exploited by the attention function, which learns dependencies between tokens and is key to the success of Transformers.
    However, the iterative application of attention across layers induces complex dynamics that remain to be fully understood.  
    To analyze these dynamics, we identify each input sequence with a probability measure and model its evolution as a Vlasov equation called the Transformer PDE, whose velocity field is nonlinear in the probability measure. Our first set of contributions focuses on compactly supported initial data. We show that the Transformer PDE is well-posed and is the mean-field limit of an interacting particle system, thus generalizing and extending previous analyses to several variants of self-attention: multi-head attention, $\revision{\ell^2}$ attention, Sinkhorn attention, Sigmoid attention, and masked attention—leveraging a conditional Wasserstein framework.
    In a second set of contributions, we are the first to study non-compactly supported initial conditions, by focusing on Gaussian initial data. Again for different types of attention, we show that the Transformer PDE preserves the space of Gaussian measures, which allows us to analyze the Gaussian case theoretically and numerically to identify typical behaviors. 
    This Gaussian analysis captures the evolution of data anisotropy through a deep Transformer.
    In particular, we highlight a clustering phenomenon that parallels previous results in the non-normalized discrete case.
\end{abstract}



\section{Introduction}

Transformers, introduced in \cite{vaswani2017attention}, are extremely successful deep learning models, which have reached the state of the art in a wide variety of tasks, from natural language processing
to computer vision.
A key feature of Transformers is that each data point (image, sentence...) is mapped to a \emph{sequence of vectors} $(x_1, \dots, x_n)\in (\R^d)^n$, called \emph{tokens}, before being processed by the model.
Each sequence of tokens is then processed by a succession of layers containing a self-attention block and a multi-layer perceptron (MLP), interleaved with a layer normalization (LayerNorm) operation.
The core component of the Transformer architecture is really self-attention: contrary to the MLP and LayerNorm, which are applied token-wise, self-attention makes all tokens interact, which allows the model to learn complex dependencies between them.
In this work, we aim to model the evolution of tokens as they go through a deep Transformer, and to identify typical behaviors to improve our understanding of how Transformers process data.
We consider a simplified model with only attention blocks, i.e., without MLPs and layer normalizations; note that adding LayerNorm in the analysis has been studied in \cite{geshkovski2023mathematical, burger2025analysismeanfieldmodelsarising} and leads to very different dynamics, as tokens are constrained to evolve on a sphere.

\

\paragraph{Variants of self-attention}
The most widespread version of self-attention is the one introduced in \cite{vaswani2017attention}, which we call Softmax self-attention, but several variants have been proposed over the years to improve specific behaviors of self-attention, such as its regularity \cite{kim2021lipschitz, sander2022sinkformers} or its computational complexity \cite{katharopoulos2020transformers, wortsman2023replacing, ramapuram2024theory}.
We adapt our analysis to each of the following variants, which we present in more detail in Section \ref{sec:att_variants}: $\revision{\ell^2}$ self-attention \cite{kim2021lipschitz}, self-attention without Softmax \cite{schlag2021linear}, ReLU \cite{wortsman2023replacing}, Sigmoid \cite{ramapuram2024theory} self-attention and Sinkhorn self-attention \cite{sander2022sinkformers}.
\revision{Sinkhorn attention can be interpreted as a way to restore a Wasserstein gradient flow structure through symmetrization—in a related spirit, \cite{agarwal2024iterated} exploits this observation in the opposite direction, viewing Transformers with a finite number of layers as a novel discretization of Wasserstein gradient flows and studying the convergence properties of the resulting scheme.}
Finally, several self-attention blocks are typically linearly combined to obtain multi-head attention, and possibly masked in Transformer decoders \cite{vaswani2017attention}—we include these two cases in our analysis, building on the framework introduced in \cite{castin2024smooth} for masked self-attention.

\

\paragraph{Transformers as interacting particle systems}
Transformers are deep models: several residual self-attention layers are stacked (typically $10$ to $100$ \cite{lin2022survey}) and interleaved with non-linear operations, which results in an involved architecture whose properties are only partially understood.
To uncover the mechanisms by which Transformers process data, recent papers \cite{sander2022sinkformers, geshkovski2023emergence, geshkovski2023mathematical} have proposed to model the behavior of tokens going through the Transformer architecture as an interacting particle system, ruled by a system of ordinary differential equations (ODEs)—thus adapting the theory of neural ODEs \cite{chen2018neural, weinan2017proposal, haber2017stable} to Transformers.
This allows for a mathematical study of the dynamics of tokens $(x_1(t), \dots, x_n(t))$ as they are processed by a Transformer, where the time variable $t$ corresponds to layer position in the model.
Let us focus on the simplest of these models, introduced in \cite{sander2022sinkformers}, where only residual single-head self-attention layers are stacked, and assume that the parameters of the attention layers are continuous functions of time.
Denoting $(x_1(0), \dots, x_n(0))\in (\R^d)^n$ the input tokens, the dynamics of tokens as they go through the model read
\begin{equation}
\label{eq:ode_system}
    \phantom{1\le i\le n}~~~~~~~~~~~~\dot x_i(t) = \Gamma_{t, X(t)}(x_i(t)) ~~~~~~~~~~~~1\le i\le n
\end{equation}
where $X(t)\coloneqq (x_1(t), \dots, x_n(t))$ and $\Gamma_{t, X}$ is the attention vector field, defined for $X\in (\R^d)^n$ and three time-dependent matrices $Q(t), K(t) \in \R^{k\times d}$ and $V(t)\in \R^{d\times d}$ with $k\le d$ as
$$\Gamma_{t, X}\colon x\in \R^d \mapsto \sum_{j=1}^n p_j(t, x) V(t)x_j\in \R^d$$
where 
$$p_j(t, x) \coloneqq \exp(Q(t)x\cdot K(t)x_j) / \sum_{\ell=1}^n \exp(Q(t)x \cdot K(t)x_\ell)$$
for Softmax (i.e., vanilla) self-attention (see Section \ref{sec:att_variants} for the other types of attention), $\cdot$ denoting the Euclidean scalar product in $\R^d$.
The authors in \cite{geshkovski2023emergence} provide a thorough study of the dynamics (\ref{eq:ode_system}) in the case of time-independent matrices $Q, K, V$, showing the emergence of clusters after properly rescaling tokens. \revision{Clustering is also investigated in the context of low-rank adaptation in \cite{koubbi2024impactloraemergenceclusters}.}

\

\paragraph{The Transformer PDE}
The dynamics (\ref{eq:ode_system}) can be generalized as the following PDE on probability measures \cite{sander2022sinkformers, zhong2022neural, geshkovski2023emergence, burger2025analysismeanfieldmodelsarising}:
\begin{equation} 
\label{eq:transformer_pde_trad}
    \partial_t\mu + \mathrm{div}(\mu \Gamma_\mu) = 0,
\end{equation}
where tokens are now represented by a probability measure $\mu \in \pcal(\R^d)$ and
$$\Gamma_\mu \colon x\in \R^d \mapsto \int Vy\, \kappa_\mu(x, y)  \dd\mu(y)$$
with $Q(t), K(t)\in \R^{k\times d}$, $V(t)\in \R^{d\times d}$ and $\kappa_\mu(x, y)\coloneqq \exp(Qx\cdot Ky) / \int \exp(Qx\cdot Ky) \dd \mu(y)$ for Softmax self-attention (see Section \ref{sec:att_variants} for the other types of attention).
We do not write explicitly the dependence of $\kappa_\mu$ on $t$ to lighten notation.
The case (\ref{eq:ode_system}) of discrete initial data $(x_1(0), \dots, x_n(0))$ corresponds to plugging the empirical measure $\frac1n \sum_{i=1}^n \delta_{x_i(0)}$ in Equation (\ref{eq:transformer_pde_trad}).
This viewpoint allows us to take possibly continuous probability measures as initial data, thus generalizing the Transformer model to inputs with infinitely many tokens.
The well-posedness of the evolution (\ref{eq:transformer_pde_trad}), which is a Vlasov-type equation \cite{vlasov1968vibrational, dobrushin1979vlasov} or an aggregation-type equation \cite{Go03,BCL09}, is a non-trivial problem, as the velocity field $\Gamma_\mu$ is typically non-linear in $\mu$ and associated with a kernel $\kappa_\mu(x, y)$ that is unbounded, since we do not normalize dynamics as opposed to \cite{burger2025analysismeanfieldmodelsarising}, and cannot be rewritten as a function of $\modu{x - y}$.
This rules out more traditional approaches, presented for instance in \cite{dobrushin1979vlasov,canizo2011well,CCH14,motsch2014heterophilious}.
In the case of Softmax self-attention, when the initial data is constrained in a compact set and $Q, K, V$ are constant over time (i.e., across layers), Equation (\ref{eq:transformer_pde_trad}) is the mean-field limit of the dynamics (\ref{eq:ode_system}) \cite{geshkovski2023emergence}, in the sense that Equation (\ref{eq:transformer_pde_trad}) is well-posed for compactly supported initial data, and satisfies a stability estimate of the form
$$
W_2(\mu(t), \nu(t)) \le C(t, R_0) W_2(\mu_0, \nu_0)
$$
where $R_0$ is the radius of any compact set containing the support of $\mu_0$ and $\nu_0$, and $W_2$ is the Wasserstein distance between measures \cite{V03,santambrogio2015optimal}.
In the first part of this work (Section \ref{sec:well_posedness}), we extend this result to time-dependent parameters and to the variants of self-attention mentioned above, \revision{including masked self-attention}.
We are also the first to study Equation (\ref{eq:transformer_pde_trad}) for non-compactly supported initial data, by focusing on the case of Gaussian probability measures (Section \ref{sec:gaussian}).
The Gaussian case is special, as we show that Equation (\ref{eq:transformer_pde_trad}) preserves Gaussians for several self-attention variants mentioned above.
This property is observed in several algorithms in sampling and optimization \cite{GHLS20,CV21,CHSV22}.
This allows us to summarize the evolution on Gaussians as two matrix ordinary differential equations (ODEs) connecting the expectation and the covariance matrix of the Gaussian data.
\revision{This explicit characterization of the dynamics then allows us to identify a variety of possible behaviors, and highlight a \emph{clustering} phenomenon, in addition to studying well-posedness.}

\revision{
To exemplify this and gain intuition on the possible behaviors of solutions to~\eqref{eq:transformer_pde_trad}, consider the case where the initial condition
$\mu_0=\mathcal{N}(0,\Sigma_0)$ is a centered Gaussian with covariance $\Sigma_0 \succ 0$, and assume for simplicity that 
$Q=K=\mathrm{Id}$ and $V=\varepsilon\,\mathrm{Id}$.  As shown in Section~\ref{sec:gaussian}, the Gaussian structure is preserved along the flow and  $\mu(t)=\mathcal{N}(0,\Sigma(t))$, where the covariance satisfies the Riccati ODE $\dot \Sigma = 2\varepsilon\, \Sigma^2$ in the case of Softmax self-attention. 
Its closed form solution is $ \Sigma(t)=(\Sigma_0^{-1}-2\varepsilon t\, \mathrm{Id} )^{-1}$.
If $\varepsilon<0$, the solution is global in time and $\Sigma(t)\to 0$ as $t\to\infty$. 
In contrast, if $\varepsilon>0$, the maximal eigenvalue of the covariance blows up in finite time $t_{\max}=(2\varepsilon \lambda_{\max}(\Sigma_0) )^{-1}$.
These two behaviors are illustrated in Figure \ref{fig:gaussian_fig_for_intro}.
For general $Q,K,V$, possibly depending on $t$, the covariance dynamics become more complex, and a more detailed analysis is carried out in Section~\ref{sec:gaussian}.
}

\begin{figure}
    \centering
    \includegraphics[width=0.7\linewidth]{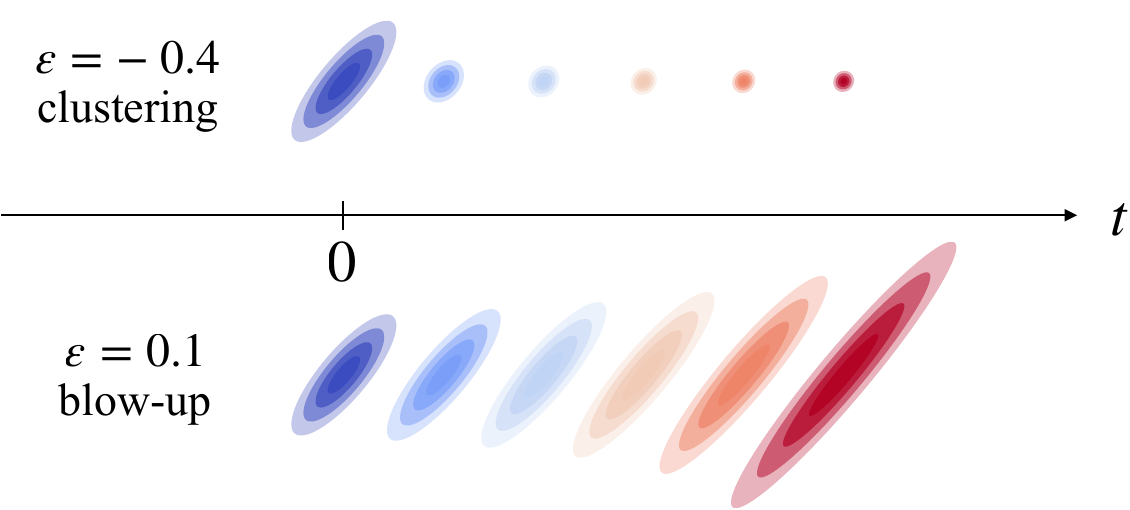}
    \caption{\revision{Covariance evolution of a centered Gaussian measure processed by the Softmax Transformer PDE, with $Q=K=\mathrm{Id}$ and $V=\varepsilon\,\mathrm{Id}$, for $\varepsilon = -0.4$ (upper plot) and $\varepsilon = 0.1$ (lower plot) with the same initial measure. The covariance matrix follows the ODE $\dot \Sigma = 2\varepsilon \Sigma^2$. Choosing $\varepsilon < 0$ leads to $\Sigma(t)\to_{t\to +\infty} 0$, i.e., all the mass clusters at a single point, while $\varepsilon > 0$ leads to a finite-time divergence of the maximal eigenvalue of $\Sigma(t)$.}}
    \label{fig:gaussian_fig_for_intro}
\end{figure}

\

\paragraph{Gradient flow structure of the PDE}
To equip a partial differential equation with a gradient flow structure can be a useful tool to prove convergence and properties of the limit.
Several works have tackled this question for the Transformer PDE.
\cite{sander2022sinkformers} show that Equation (\ref{eq:transformer_pde_trad}) is not a Wasserstein gradient flow, whereas the Sinkformer PDE is, allowing them to highlight a diffusive behavior when the parameter $\varepsilon$ associated to Sinkhorn tends to zero.
\cite{geshkovski2023mathematical} study an alternative to the Transformer PDE where tokens are of unit norm, thus modelling the effect of LayerNorm, and leverage the gradient flow structure of this projected dynamics to prove the emergence of one single cluster when $t\to \infty$.
They also introduce a modified metric on $(\mathbb{S}^{d-1})^n$, which equips the original dynamics (\ref{eq:ode_system}) with a non-Euclidean gradient flow structure, and is a particular case of the more general framework introduced in \cite{li2021hessian}.
In Section \ref{sec:gradient_flow}, we propose a generalization of this metric from tokens to probability measures, formally showing that it equips the Transformer PDE (\ref{eq:transformer_pde_trad}) associated with Softmax self-attention with a gradient flow structure.
Adapting metrics to convert similar nonlocal PDEs to gradient flows has been a very fruitful approach in different settings \cite{CLSS10,CDDW24,BEHFS25}. In fact, this has very recently been developed for the normalized dynamics similar to \eqref{eq:transformer_pde_trad} associated to \eqref{eq:ode_system} on the sphere in \cite{burger2025analysismeanfieldmodelsarising}.
In the case of Gaussian initial data (Section \ref{sec:gaussian}), the gradient flow structure of the Sinkformer PDE \cite{sander2022sinkformers} leads to a Bures-Wasserstein gradient flow. 
The Bures-Wasserstein distance, first introduced in the field of quantum information geometry \cite{bures1969extension}, appears in several recent works in machine learning \cite{lambert2022variational, delon2020wasserstein, leger2024nonnegative}.

\

\noindent Our main contributions can be summarized as follows.
\begin{itemize}
    \item[i)] If $\Gamma_\mu$ is the velocity field of one of the following variants of self-attention, in its single-head \emph{or} multi-head form, and in its masked \emph{or} unmasked form: Softmax self-attention, $\revision{\ell^2}$ self-attention, Sinkhorn self-attention, Sigmoid self-attention, then we show that for any compactly supported initial data $\mu \in \pcal_c(\R^d)$, the evolution (\ref{eq:transformer_pde_trad}) is well-posed.
    The masked self-attention case leverages the conditional optimal transport framework introduced in \cite{castin2024smooth}, \revision{and requires that the position marginal of the initial distribution has a Dirac mass at 0}.
    We also derive a stability estimate with respect to the initial data, which proves that Equation (\ref{eq:transformer_pde_trad}) is the mean-field limit of the dynamical system (\ref{eq:ode_system}). (Section \ref{sec:well_posedness})
    
    \item[ii)] For Softmax, $\revision{\ell^2}$, Sinkhorn and linear self-attention, we show that initial data that are Gaussians stay Gaussian along the dynamics (\ref{eq:transformer_pde_trad}). This allows us to derive explicit ODEs on the mean and the covariance matrix of these Gaussian solutions. For Softmax self-attention, we demonstrate under certain assumptions on the parameters that the limiting covariance of the solution is rank-deficient, which mimics the clustering phenomenon observed in \cite{geshkovski2023emergence}. (Sections \ref{subsec:trad_gaussian}, \ref{subsec:L2_gaussian} and \ref{subsec:sink_self_attention}). 
    
    \item[iii)] With a numerical study of the covariance ODEs in the Gaussian case, we point out that rank-deficiency of the limiting covariance generically holds even without the assumptions of our theoretical analysis, for Softmax self-attention as well as $\revision{\ell^2}$ self-attention. Moreover, we identify a range of typical behaviors in dimension 2, divided in three categories: (i) convergence, (ii) divergence at $t\to +\infty$ and (iii) divergence in finite time. We prove that $\revision{\ell^2}$ self-attention cannot lead to finite-time divergence, contrary to Softmax self-attention. (Section \ref{subsec:experiments})
    
    \item[iv)] As a side contribution, we observe that Gaussian-kernel drifting models \cite{deng2026generative, cao2026gradient} are driven by the difference of two normalized Gaussian-kernel fields, each of which is the field that appears in $\revision{\ell^2}$ self-attention for a suitable choice of parameters. The Gaussian analysis therefore yields closed ODEs for drifting between Gaussian distributions and clarifies the mechanism in this simple regime. (Appendix \ref{appsec:drifting_models})
    
    \item[v)] Finally, we generalize results in \cite{geshkovski2023mathematical} by introducing a twisted metric on the space of probability measures, which equips the Softmax Transformer PDE (\ref{eq:transformer_pde_trad}) with a gradient flow structure. We prove the non-geodesic convexity of the associated functional for the non-normalized dynamics. This complements very recent results in \cite{burger2025analysismeanfieldmodelsarising} for the normalized dynamics. We also reformulate the covariance ODE for Sinkhorn attention in the Gaussian case as a Bures-Wasserstein gradient flow. (Section \ref{sec:gradient_flow})
\end{itemize}
The code for our figures is at \texttt{github.com/vcastin/2025-transformers-PDEs}.

\subsection{Notations}

$\N^*$ is the set of positive natural numbers.
We denote $\cdot$ the Euclidean scalar product and $\modu{\cdot}$ the Euclidean norm.
$\norm{\cdot}_2$ is the associated operator norm.
$\pcal(\R^d)$ (resp. $\pcal_c(\R^d)$) is the set of probability measures (resp. compactly supported probability measures) on $\R^d$.
For all $p\ge 1$, the $p$-Wasserstein space (resp. distance) is denoted $\pcal_p(\R^d)$ (resp. $W_p$).
The rank of a matrix $M$ is denoted $\mathrm{rk}\, M$.
We denote $B_R$ or $\bbar(0, R)$ the closed ball of center 0 and of radius $R$.
The support of a probability distribution $\mu$ is denoted $\supp \mu$.
The Gaussian distribution of mean $\alpha \in \R^d$ and covariance $\Sigma \in \R^{d\times d}$ is denoted $\ncal(\alpha, \Sigma)$, and $\ncal(0, \Sigma)$ is denoted $G_\Sigma$.
The set of $d\times d$ real symmetric matrices is denoted $\mathcal S_d$ and the subset of $\mathcal S_d$ of positive semidefinite (resp. positive definite) matrices is denoted $\mathcal S_d^{+}$ (resp. $\mathcal S_d^{++}$).
For any invertible matrix $A\in \R^{d\times d}$, we denote $A^{-\top}\coloneqq (A^{-1})^\top$.
We denote $\ccal^\infty(\R^d, \R)$ the space of infinitely differentiable functions, and $\ccal_0^\infty(\R^d, \R)$ the space of $\ccal^\infty$ functions that tend to 0 at infinity.

\section{Some Variants of Self-Attention}
\label{sec:att_variants}

Several variants of self-attention and masked self-attention \cite{vaswani2017attention} have been introduced over the years to improve some properties of the attention map, such as its \revision{regularity} \cite{kim2021lipschitz, sander2022sinkformers} or its computational complexity \cite{katharopoulos2020transformers, wortsman2023replacing, ramapuram2024theory}.
Our approach provides a unified framework to compare the behavior of a Transformer model for each variant: we model any Transformer model (Encoder-only or Decoder-only), which is the composition of several residual attention layers interleaved with projections and multi-layer perceptrons—that we do not include here, by a PDE of the form
\begin{equation}
\label{eq:continuity_equation_part_1}
    \partial_t \mu + \mathrm{div}(\mu \Gamma_\mu) = 0
\end{equation}
where the velocity field $\Gamma_\mu\colon \R^\delta \to \R^\delta$ depends on the data $\mu \in \mathcal{P}(\R^\delta)$ and on the chosen variant of self-attention.
For unmasked self-attention, $\delta$ is equal to the dimension $d$ of each token.
In the case of masked self-attention, we have $\delta = d + 1$ following \cite{castin2024smooth}, as explained in Subsection \ref{subsec:masked}.
This approach generalizes the case of a finite number of tokens, which corresponds to $\mu$ being an empirical measure $\frac1n \sum_{i=1}^n \delta_{x_i}$.
Indeed, casting this formula for $\mu$ in Equation (\ref{eq:continuity_equation_part_1}) leads back to the discrete ODE system (\ref{eq:ode_system}).

\medskip
Here, we list all the attention variants covered by our framework.
Let $k, d \in \N^*$ with $k\le d$ and matrices $Q, K\in \R^{k\times d}$ and $V \in \R^{d\times d}$.
We denote $A\coloneqq K^\top Q$.

\subsection{Softmax (Single-Head) Self-Attention}
\label{subsec:trad_self_attention}
We call Softmax self-attention the original definition of self-attention in \cite{vaswani2017attention}.
For any integer $n\in \N^*$ and any vectors $x_1, \dots, x_n \in \mathbb R^d$, Softmax self-attention with parameters $(Q, K, V)$ maps the sequence $X = (x_1, \dots, x_n) \in (\mathbb R^d)^n$ to
\begin{equation}
\label{eq:discrete_attention}
    f\satt{SM}(x_1, \dots, x_n) \coloneqq \Big ( V {\textstyle \sum_{j=1}^n} P_{ij} x_j \Big )_{1 \le i \le n} \in (\R^d)^n,
\end{equation}
\begin{equation*}
    \text{with}\quad P_i \coloneqq \softmax \left ( (Qx_i\cdot K x_j)_{1\le j\le n} \right ),
\end{equation*}
where
$\softmax(w) \coloneqq ( \exp(w_i)/{\sum_{j=1}^n \exp(w_j)}  )_{1\le i\le n}.$
Note that, as we do not consider training, we absorbed the usual normalizing factor $1 / \sqrt{k}$ in $Q$ and $K$.
Now, Equation (\ref{eq:discrete_attention}) can be rewritten as
\begin{equation} 
\label{eq:discrete_pushforward}
    f(X) = (\Gamma_X\satt{SM}(x_1),\dots, \Gamma_X\satt{SM}(x_n)),
\end{equation}
with
$$\Gamma_X\satt{SM}\colon x\in \R^d \mapsto \frac{\sum_{i=1}^n  e^{Qx\cdot K x_i}Vx_i}{\sum_{i=1}^n e^{Qx\cdot K x_i}}.$$
Equation (\ref{eq:discrete_pushforward}) naturally leads to the PDE viewpoint (\ref{eq:continuity_equation_part_1}) by defining, for any probability measure $\mu \in \pcal(\R^d)$ \revision{such that $\mathbb E_{y\sim\mu}[\lvert y\rvert e^{Qx\cdot Ky}]<+\infty$ for all $x\in\R^d$}, the velocity field 
$$\Gamma_\mu\satt{SM} \colon x\in \R^d \mapsto \frac{\int Vy\, e^{Qx\cdot Ky} \dd \mu(y) }{ \int e^{Qx\cdot Ky} \dd \mu(y)}.$$
This generalization of Softmax self-attention to probability measures appears first in \cite{vuckovic2020mathematical}, and later, in this specific form, in \cite{sander2022sinkformers}.
Note that Softmax self-attention does not depend separately on $Q$ and $K$, but only on the product $A\coloneqq K^\top Q\in \R^{d\times d}$.
Remarkably, $\Gamma_\mu\satt{SM}$ is well-defined when $\mu$ is Gaussian, and even has a closed form (see Section \ref{subsec:trad_gaussian}).

\subsection{Linear Self-Attention}
\label{subsec:linear}
Several variants of self-attention, called \emph{linear} variants, reduce the computational complexity of attention from quadratic to linear in the number of tokens $n$ \cite{katharopoulos2020transformers, choromanski2020rethinking, peng2021random, shen2021efficient, schlag2021linear}.
The simplest version of linear attention replaces the Softmax by the identity in Equation (\ref{eq:discrete_attention}), and is widely used in the field of in-context learning \cite{fournier2023practical, von2023uncovering, mahankali2023one, sander2024transformers}.
It is defined as
\begin{equation*}
    f\satt{lin}(x_1,\dots, x_n) \coloneqq (\Gamma_X\satt{lin}(x_1),\dots, \Gamma_X\satt{lin}(x_n)),
\end{equation*}
with
$$\Gamma_X\satt{lin}\colon x\in \R^d \mapsto  V\frac1n \sum_{j=1}^n x_jx_j^\top  A x,$$
which corresponds to the velocity field
$$\Gamma_\mu\satt{lin}\colon x\in \R^d\mapsto \int Vy (Qx\cdot Ky) \dd\mu(y).$$
Let us point out that a slight modification of this attention without Softmax \cite{schlag2021linear} can be obtained as a second-order expansion of Softmax self-attention in the regime where $A\coloneqq K^\top Q$ is very small.
Denoting $f_\varepsilon\satt{SM}$ the Softmax attention map with parameters $(\varepsilon Q, K, V / \varepsilon)$, we have for any $X=(x_1, \dots, x_n)\in (\R^d)^n$ that
$$f\satt{SM}_\varepsilon(X)_i = \frac{1}{\varepsilon n}\sum_{j=1}^n Vx_j + V\prt{\frac1n \sum_{j=1}^n x_jx_j^\top - \frac{1}{n^2} \sum_{\ell=1}^n x_\ell\sum_{\ell=1}^n x_\ell^\top} A x_i + o_{\varepsilon \to 0}(\varepsilon),$$
with $A\coloneqq K^\top Q$.
We denote
\begin{equation*}
    f\satt{lin, \varepsilon}(x_1,\dots, x_n) \coloneqq (\Gamma_X\satt{lin, \varepsilon}(x_1),\dots, \Gamma_X\satt{lin, \varepsilon}(x_n)),
\end{equation*}
with
$$\Gamma_X\satt{lin, \varepsilon}\colon x\in \R^d \mapsto \frac{1}{n\varepsilon}\sum_{j=1}^n Vx_j + V\prt{\frac1n \sum_{j=1}^n x_jx_j^\top - \frac{1}{n^2} \sum_{\ell=1}^n x_\ell\sum_{\ell=1}^n x_\ell^\top} A x$$
the variant of linear attention obtained by expanding the Softmax attention map.
The second sum corresponds exactly to attention without Softmax.
Note that this attention variant has the following simple expression for any \emph{compactly supported or Gaussian} probability measure $\mu$ of expectation $\alpha$ and covariance $\Sigma$:
$$\Gamma_\mu\satt{lin, \varepsilon}\colon x\in \R^d\mapsto \frac1\varepsilon V\alpha + V\Sigma Ax.$$
This expression is the same as for Softmax self-attention up to a rescaling of $\alpha$ (see Subsection \ref{subsec:trad_gaussian}).
Therefore, all the results we state for \revision{Softmax} self-attention in the Gaussian case also hold for this particular choice of linear self-attention.

\subsection[L2 Self-Attention]{$\revision{\ell^2}$ Self-Attention}
For any integer $n\in \N^*$ and vectors $x_1, \dots, x_n \in \mathbb R^d$, $\revision{\ell^2}$ self-attention with parameters $(Q, K, V)$ maps the sequence $X = (x_1, \dots, x_n) \in (\mathbb R^d)^n$ to
\begin{equation*}
    f\satt{\revision{\ell^2}}(x_1, \dots, x_n) \coloneqq \Big ( V {\textstyle \sum_{j=1}^n} P_{ij} x_j \Big )_{1 \le i \le n} \in (\R^d)^n,
\end{equation*}
with
$\quad P_{ij} \coloneqq e^{-\modu{Qx_i - Kx_j}^2} / \sum_{\ell = 1}^n e^{-\modu{Qx_i - Kx_\ell}^2}.$
\revision{This $\ell^2$} self-attention was introduced in \cite{kim2021lipschitz} for its Lipschitz continuity.
\revision{When $Q = K = V = I_d$, it is akin to a mean-shift algorithm step \cite{fukunaga1975estimation}.}
Following the same approach as for traditional self-attention, we define the velocity field associated with $\revision{\ell^2}$ self-attention, for any compactly supported probability measure $\mu \in \pcal_c(\R^d)$, as
\begin{equation*} 
\Gamma_\mu\satt{\revision{\ell^2}} \colon x\in \R^d \mapsto \frac{\int Vy\, e^{-\modu{Qx - Ky}^2} \dd \mu(y) }{ \int e^{-\modu{Qx - Ky}^2} \dd \mu(y)}.\end{equation*}
Note that $\Gamma_\mu\satt{\revision{\ell^2}}$ depends on $Q, K$ and not only on their product $A = K^\top Q$.
Like Softmax self-attention, it can be computed in closed form when $\mu$ is Gaussian (Section \ref{subsec:L2_gaussian}).

\subsection{Sinkhorn Self-Attention}
Sinkhorn self-attention has been introduced in \cite{sander2022sinkformers}.
We only give its definition on probability measures for conciseness.
For any compactly supported or Gaussian probability measure $\mu$, the velocity field associated to Sinkhorn self-attention with the parameter $\varepsilon > 0$ and the cost $c_\varepsilon(x,y)\coloneqq \frac{1}{2\varepsilon}\modu{Qx - Ky}^2 $ is defined as
$$\Gamma_{\mu,\varepsilon}\satt{sink} \colon x\in \R^d \mapsto \frac1\varepsilon \int Vy \, \kappa_{\mu,\varepsilon}^\infty(x, y) \dd \mu(y)$$
where $\kappa_{\mu,\varepsilon}^\infty$ is obtained by performing the Sinkhorn-Knopp algorithm on $\kappa_{\mu, \varepsilon}^0\coloneqq e^{-c_\varepsilon}$, i.e. $\kappa_{\mu,\varepsilon}^\infty(x, y)$ is the limit of the following sequence:
\begin{equation}
\label{eq:sinkhorn_iterations}
    \kappa_{\mu,\varepsilon}^{\ell +1}(x,y) = \begin{cases}
        \frac{\kappa_{\mu,\varepsilon}^\ell(x,y)}{\int \kappa_{\mu,\varepsilon}^\ell(x,y') \dd \mu(y')} \mbox{ if $\ell$ is even,} \\
        \frac{\kappa_{\mu,\varepsilon}^\ell(x,y)}{\int \kappa_{\mu,\varepsilon}^\ell(x',y) \dd \mu(x')} \mbox{ if $\ell$ is odd.}
    \end{cases}
\end{equation}
This algorithm—which, in its discrete version, normalizes the kernel matrix $\kappa$ iteratively row-wise and column-wise—outputs a bistochastic kernel: $\int \kappa_{\mu,\varepsilon}^\infty(z, y) \dd \mu(z) = \int \kappa_{\mu,\varepsilon}^\infty(x, z) \dd \mu(z) = 1$ for all $x, y\in \R^d$.

The Sinkhorn algorithm has deep connections with entropic optimal transport (EOT).
Indeed, we have the following reinterpretation of Sinkhorn self-attention.
Let $\mu$ and $\nu$ be two probability measures on $\R^d$, either supported in a compact set or Gaussian.
Keeping the notation $c_\varepsilon(x,y)\coloneqq \frac{1}{2\varepsilon}\modu{Qx - Ky}^2 $, consider the entropic optimal transport problem
\begin{equation} 
\label{eq:eot}
    OT_\varepsilon(\mu, \nu) \coloneqq \min_{\pi\in \Pi(\mu, \nu)} \int c_\varepsilon(x,y)\dd \pi(x,y) + \KL(\pi \| \mu\otimes \nu),
\end{equation}
where $\Pi(\mu, \nu)$ is the set of couplings between $\mu$ and $\nu$ and $\KL(\alpha\| \beta)\coloneqq \int_{\mathcal X}\log\left ( \frac{\dd \alpha}{\dd \beta}\right ) \dd \alpha$ is the Kullback-Leibler divergence, which is infinite if $\alpha$ is not absolutely continuous with respect to $\beta$.
Assuming that $\mu = \nu$ and denoting $\dd\pi^*(x,y)\coloneqq \kappa_{\mu,\varepsilon}^\infty(x,y) \dd(\mu\otimes \mu)$ where $\kappa_{\mu,\varepsilon}^\infty$ is obtained as above with Sinkhorn, it is well-known \cite{nutz2021introduction, janati2020entropic} that the coupling $\pi^*$ is the unique solution of the entropic optimal transport problem (\ref{eq:eot}). 
Therefore, $\kappa_{\mu, \varepsilon}^\infty$ is the density of the optimal coupling $\pi^*$ with respect to the measure $\mu\otimes \mu$.
Note that, like for Softmax self-attention, $\kappa_{\mu, \varepsilon}^\infty$ only depends on the product $A = K^\top Q$ (see \cite{sander2022sinkformers}).
We also study the Gaussian case for Sinkhorn self-attention in Section \ref{subsec:sink_self_attention}.

\subsection{Unnormalized Self-Attention}
A natural family of self-attention variants consists of suppressing the normalization factor in the definition of $\Gamma$, and possibly changing the exponential that appears in the Softmax for a smoother function such as identity \cite{schlag2021linear} (which falls back on linear attention, Section \ref{subsec:linear}), ReLU \cite{wortsman2023replacing} or Sigmoid \cite{ramapuram2024theory}.
In Section \ref{sec:well_posedness}, we mention the following variants:
\begin{align*} 
&\Gamma_\mu\satt{exp}\colon x\in \R^d \mapsto \int Vy \ e^{Qx\cdot Ky} \dd \mu(y),\\
&\Gamma_\mu\satt{ReLU}\colon x\in \R^d \mapsto \int Vy \ \mathrm{ReLU}(Qx\cdot Ky) \dd \mu(y),\\
&\Gamma_\mu\satt{\sigma}\colon x\in \R^d \mapsto \int Vy \ \mathrm{\sigma}(Qx\cdot Ky) \dd \mu(y),
\end{align*}
where $\mathrm{ReLU}\colon x\in \R\mapsto \max(x, 0)$ and $\sigma \colon z\in \R\mapsto (1 + e^{-z})^{-1}$ is the Sigmoid function.
It is easy to see that none of these attention variants preserve the Gaussianity of the data, as $\Gamma_\mu$ is not an affine function when $\mu$ is Gaussian.

\subsection{Masked Self-Attention}
\label{subsec:masked}
Softmax self-attention, introduced at the beginning of Section \ref{subsec:trad_self_attention}, is used in Encoders, which appear in Encoder-only \cite{devlin2018bert} and Encoder-Decoder Transformers \cite{vaswani2017attention}.
In Decoders, however, what is used is masked self-attention, which takes into account the sequential nature of the inputs.
In its discrete version, masked self-attention maps the sequence $X = (x_1, \dots, x_n)\in (\R^d)^n$ to $(f\satt{m}(X)_1, \dots, f\satt{m}(X)_n)\in (\R^d)^n$ such that
\begin{equation} 
\label{eq:masked_att_discrete} f\satt{m}(X)_i \coloneqq f\satt{SM}(x_1, \dots, x_i),
\end{equation}
where $f\satt{SM}$ is Softmax self-attention, defined in Equation (\ref{eq:discrete_attention}).
Similarly, one can define a masked attention map for any discrete self-attention variant.
Masked self-attention is crucial for next-token prediction tasks in NLP and time series: at test time, it is used to generate data in an autoregressive fashion, i.e., each newly generated token contributes to the computation of the next one; and during training, it prevents the model from cheating in its next-token prediction task by making use of tokens it has to predict.

Equation (\ref{eq:masked_att_discrete}) cannot be directly generalized to probability measures, as masked attention makes use of the order of tokens, which is lost when representing a sequence $(x_1, \dots, x_n)$ by the associated empirical measure $\frac1n\sum_{i=1}^n \delta_{x_i}$.
To include masked self-attention in our unified framework, we leverage a contribution in \cite{castin2024smooth}, which extends masked self-attention to probability measures by adding a \emph{position coordinate} in $[0,1]$ to the input space: this coordinate specifies the order of tokens.
\revision{Then, an input to masked self-attention becomes a probability measure $\bar \mu$ on the product space $[0,1]\times \R^d$.}
Let $d\in \mathbb N^*$ and $Q, K\in \R^{k\times d}$ and $V\in \R^{d\times d}$.
The velocity field associated with masked Softmax self-attention is defined, for any compactly supported probability measure $\bar \mu \in \mathcal{P}_c([0,1]\times \R^d)$ \revision{such that $\bar\mu(\{0\}\times \R^d) > 0$}, as
\begin{equation}\label{eq:masked}\Gamma_{\bar \mu}\satt{m}(\sigma, x) \coloneqq \prt{0, \frac{\int_{[0, \sigma]\times \R^d} Vy \, e^{Ax\cdot y}  \dd \bar \mu(\tau, y)}{ \int_{[0, \sigma]\times\R^d} e^{Ax\cdot y} \dd \bar \mu(\tau, y)}}\end{equation}
with $A\coloneqq K^\top Q$.
In what follows, we call \emph{position marginal} the first marginal of $\bar \mu$, defined as
\begin{equation}
\label{eq:time_marginal}\dd\theta(\sigma)\coloneqq \int_{x\in\R^d} \dd\bar\mu(\sigma, x).\end{equation}
$\theta$ is therefore a probability measure on $[0, 1]$.
\revision{The assumption $\bar\mu(\{0\}\times \R^d) > 0$ is equivalent to saying that $\theta$ has a Dirac mass at zero. 
Adding a Dirac mass at zero together with a null token vector is equivalent to the off-by-one attention correction and the attention sink mechanism~\cite{miller2023attentionoffbyone,xiao2024attentionsinks,agarwal2025gpt}. 
Note that these techniques absorb a portion of the attention probability mass, preventing pathological concentration of attention on spurious tokens and improving robustness to outliers and long-context degradation. 
}
We call \emph{space marginal} and denote $\mu$ the marginal of $\bar \mu$ in $\R^d$, defined as
\begin{equation*}
    \dd \mu(x) \coloneqq \int_{\sigma\in[0, 1]}\dd\bar \mu(\sigma, x).
\end{equation*}
$\mu$ is therefore a probability measure on $\R^d$.
\revision{
\begin{remark}Note that the assumption $\bar\mu(\{0\}\times \R^d) > 0$ ensures that $\Gamma_{\bar \mu}\satt{m}(\sigma, x)$ is well-defined for $\sigma = 0$. It can be relaxed to the weaker assumption that $\dd \bar\mu(\sigma, x)$ disintegrates as $\dd \bar\mu^\sigma(x) \dd\theta(\sigma)$ for $(\bar\mu^\sigma)_{\sigma\in[0,1]}$ a weak-$\star$ continuous family of probability measures—which is for example the case when $\bar\mu$ is absolutely continuous with respect to the Lebesgue measure on $[0,1]\times \R^d$ \cite{furuya2024transformers}.
However, our proof of well-posedness for masked self-attention (Theorem \ref{thm:mask_compact_support}) only holds under the assumption $\bar\mu(\{0\}\times \R^d) > 0$, which is why we keep this assumption in the definition, for simplicity.
\end{remark}
}
For a finite number of tokens, Equation (\ref{eq:masked_att_discrete}) can be cast in the framework (\ref{eq:masked}) by choosing $\bar \mu = \frac1n \sum_{i=1}^{n} \delta_{((i-1)/n, x_i)}$, for instance.
Equation (\ref{eq:masked}) then provides a natural generalization of masked Softmax self-attention to continuous distributions of tokens.
Moreover, this idea, originally applied to Softmax self-attention, can be used for any of the attention variants mentioned above.

Note that with this choice of velocity field, the position marginal of $\bar \mu(t)$ stays constant along the dynamics (\ref{eq:continuity_equation_part_1}).
However, Equation (\ref{eq:masked}) does not allow for a study of the Gaussian case—e.g., when the space marginal is Gaussian—as the masked Transformer PDE does not preserve this structure.

\subsection{Multi-Head Self-Attention}
The attention functions presented above provide different ways of learning dependencies—for example, semantic dependencies—between tokens.
In practice, to increase the expressive power of attention and learn dependencies at different scales in the text or image, several \emph{attention heads} with different parameters are linearly combined to obtain what is called multi-head attention.
If $\Gamma$ is the velocity field of a (possibly masked) attention function—Softmax, linear, $\revision{\ell^2}$, Sigmoid..., the associated multi-head velocity field takes the form
$$\Gamma_\mu\satt{MH}\coloneqq \sum_{h=1}^H \Gamma_\mu^{(h)}$$
for $H$ the number of heads, which must divide $d$, and where the parameters of $\Gamma_\mu^{(h)}$ are $Q^{(h)}, K^{(h)}\in \R^{d/H\times d}$ and $V^{(h)}\in \R^{d\times d}$.
Note that multi-head attention usually involves matrices $W^{(h)}$ that multiply each term $\Gamma_\mu^{(h)}$ for $1\le h\le H$.
As we do not consider training, we absorb $W^{(h)}$ in $V^{(h)}$, so that our matrices $V^{(h)}$ are $d\times d$. 
Moreover, it is practical to represent the list of matrices $(Q^{(h)}, K^{(h)})_{1\le h\le H}\in (\R^{d/H\times d})^H$ by two square matrices $Q, K\in \R^{d\times d}$, where $Q$ (resp. $K$) is obtained by stacking the $Q^{(h)}$ (resp. $K^{(h)}$) row-wise:
\begin{equation}\label{eq:MH_representation} Q \coloneqq \begin{pmatrix}
    Q^{(1)} \\ \vdots \\ Q^{(H)}
\end{pmatrix} \quad \mbox{and} \quad K \coloneqq \begin{pmatrix}
    K^{(1)} \\ \vdots \\ K^{(H)}
\end{pmatrix}.\end{equation}
\revision{As multi-head attention is a \emph{linear} combination of single attention heads, we show that well-posedness and Gaussian preservation, when they hold for a single-headed variant of self-attention, also hold for its multi-headed version.}

\section{Well-Posedness for a Compactly Supported Initial Condition}
\label{sec:well_posedness}

We first investigate the behavior of the Transformer PDE
\begin{equation}
\label{eq:PDE_part2_intro}
    \partial_t\mu + \mathrm{div}(\mu \Gamma_\mu) = 0
\end{equation}
when the initial condition is compactly supported, for all variants of self-attention presented in Section \ref{sec:att_variants}.
For any $p\ge 1$, we will be looking for solutions of Equation (\ref{eq:PDE_part2_intro}), as curves $\mu \in \ccal([0, T], \pcal_p(\R^d))$, continuous with the topology induced by the Wasserstein distance $W_p$ on $\pcal_p(\R^d)$—which makes it a complete space.
We say that $\mu \in \ccal([0, T], \pcal_p(\R^d))$ is a (weak) solution of Equation (\ref{eq:PDE_part2_intro}) with initial data $\mu_0$ if for any $\ccal^\infty$ function $\psi\in \ccal_0^\infty([0, +\infty)\times \R^d)$ that tends to zero at infinity we have
\begin{equation*}
    \int_0^T \int_{\R^d}\prt{\frac{\partial \psi}{\partial t}  + \Gamma_\mu \cdot \nabla_x \psi} \dd\mu(t)\dd t= \int_{\R^d}\psi(T, \cdot)\dd \mu(T)  - \int_{\R^d}\psi(0, \cdot)\dd \mu_0.
\end{equation*}
We focus on proving the well-posedness of the PDE (\ref{eq:PDE_part2_intro}), i.e., existence and uniqueness of a global solution, and deriving a stability estimate.
The case of Softmax self-attention has already been studied in the particular case of constant parameters, while the other cases are completely new.
Our proof, which relies on a fixed-point argument, is different and technically simpler than the one in \cite{geshkovski2023emergence} since it relies on classical Dobrushin type estimates \cite{dobrushin1979vlasov,canizo2011well}.

\subsection{Unmasked Self-Attention}

Let us first study Equation (\ref{eq:PDE_part2_intro}) for our attention variants in their unmasked form.
We have the following result.

\begin{theorem}
\label{thm:compact_support}
    Let $d, k\in \N^*$ with $k\le d$ and $p\ge 1$.
    Let $Q, K\colon [0, +\infty) \to \R^{k\times d}$ and $V \colon [0, +\infty) \to \R^{d\times d}$ be three \revision{integrable} maps, modeling the evolution of parameters $Q, K, V$ across layers of the Transformer.
    We set $\varepsilon = 1$ for Sinkformer attention, by absorbing the $\varepsilon$ in $Q, K, V$.
    Let $\mu_0 \in \pcal_c(\R^d)$ be a compactly supported initial condition.
    Then, for any choice of velocity field in $\{\Gamma\satt{SM}, \Gamma\satt{\revision{\ell^2}}, \Gamma\satt{sink}, \Gamma\satt{\sigma} \}$ \revision{in its single-head \emph{or} multi-head version}, and denoting $\Gamma_\mu(t, \cdot) \coloneqq \Gamma_{\mu(t)}(\cdot)$ the attention map associated to $Q(t), K(t), V(t)$ (with the convention (\ref{eq:MH_representation}) and $k=d$ for multi-head attention), the Transformer PDE
    \begin{equation} 
    \label{eq:transf_eq}
        \partial_t \mu + \mathrm{div}(\mu \Gamma_\mu) = 0
    \end{equation}
    with initial condition $\mu_0$ has a unique global weak solution $\mu \in \ccal([0,+\infty), \pcal_c(\R^d))$, with $\pcal_c(\R^d)$ equipped with the $p$-Wasserstein distance $W_p$.
    Moreover, let $R_0$ be the smallest radius such that $\supp \mu_0 \subset B_{R_0}$, and define 
    $$R(t) \coloneqq \exp\left(\int_0^t\norm{V(s)}_2 \dd s\right) R_0$$
    for $t\ge 0$, \revision{where $\lVert \cdot \rVert_2$ is the spectral norm}.
    Then, the solution $\mu$ satisfies 
    $$ \supp \mu(t) \subset B_{R(t)}$$
    for all $t\ge 0$.
    Finally, we have the following stability estimate.
    For all $R_0>0$ and $t>0$, there exists a constant $C(t, R_0)$ depending only on $t$, $R_0$ \revision{and $Q,K,V$} (and on the choice of $\Gamma$), such that for any initial conditions $\mu_0$ and $\nu_0$ supported in $B_{R_0}$, and denoting $\mu$ and $\nu$ the associated global solutions of Equation (\ref{eq:transf_eq}), we have
    $$ W_p(\mu(t),\nu(t)) \le C(t,R_0) W_p(\mu_0,\nu_0).$$
\end{theorem}

Therefore, for compactly supported initial data, the Transformer PDE is well-posed: it has a unique global solution on $[0, T]$ with initial data $\mu_0$, whose support typically grows exponentially with time.
Moreover, it generalizes the discrete system (\ref{eq:ode_system}), which corresponds to the initial data $\mu_0\coloneqq \frac1n \sum_{i=1}^n \delta_{x_i(0)}$, and
is the mean-field limit of this interacting particle system on any compact time interval $[0, T]$: if $(\mu^\eta(0))_{\eta\in\N^*}$ is a sequence of empirical measures such that $W_p(\mu^\eta(0), \mu(0))\to_{\eta\to+\infty}0$, then the distance $W_p(\mu^\eta(t), \mu(t))$ tends to $0$ as well when $\eta \to +\infty$ for any $0\le t\le T$; and this holds for Softmax, $\revision{\ell^2}$, Sinkhorn and Sigmoid self-attention and their multi-head version.
However, the constant $C(t, R_0)$ is exponential in $R(t)$, which is itself exponential in $t$.
When $t$ becomes large, the stability estimate is therefore very loose.

\begin{proof}
Let us detail the proof of Theorem \ref{thm:compact_support} in the simpler case where the parameters $Q, K, V$ are constant over time.
The general case is deferred to Appendix \ref{appsubsec:varying_params}.

\paragraph{Step 1: defining the flow $\phi_t(\mu)(x)$}
For all $T>0$, denote $\ccal_{\mathrm{co}}([0,T], \pcal_c(\R^d))$ the set of equi-compactly supported curves, i.e., of continuous curves $\mu \in \ccal([0,T], \pcal_c(\R^d))$ such that for any compact time interval $[s, t]\subset [0,T]$, there exists a compact set $K\subset \R^d$ verifying: $\forall \tau \in [s,t],\ \supp(\mu(\tau)) \subset K.$
Let $\mu \in \ccal_{\mathrm{co}}([0,T], \pcal_c(\R^d))$ and $x\in \R^d$, and consider the Cauchy problem
\beq
\label{eq:freezed_cauchy}
\begin{cases}
    \dot r(t) = \Gamma_{\mu}(t, r(t)) \mbox{ in }0\le t \le T\\
    r(0) = x.
\end{cases}.
\eeq
We prove the following estimates for each \revision{single-headed} self-attention type (see Appendix \ref{appsubsec:estimates})\revision{—multi-head estimates, stated in Lemma \ref{applem:mh_estimates}, are the same up to changing the constant factors, so that the arguments still hold.}
If $\supp \mu \subset B_R$ \revision{and $\supp \nu \subset B_R$}, then
\begin{enumerate}[label=(\roman*)]
    \item $\sup_{x\in \R^d}\modu{\Gamma_\mu(x)} \le \norm{V}_2 R$, \label{eq:ref_1}
    \item $\sup_{x\in \R^d}\norm{D_x\Gamma_\mu}_2 \le \norm{V}_2  \norm{A}_2 R^2$,\label{eq:ref_2}
    \item $\modu{\Gamma_\mu(x) - \Gamma_\nu(x)} \le   c(\modu{x}, R) W_p(\mu, \nu), $\label{eq:ref_3}
\end{enumerate}
where $c(\modu{x}, R)$ is a continuous function that depends on $\modu{x}, R$ and $Q, K, V$.
Then, traditional Cauchy theory tells us that problem (\ref{eq:freezed_cauchy}) has a unique global solution \revision{in $[0, T]$}.
Indeed, denote $R>0$ a radius such that $\supp \mu(t) \subset B_R$ for all $0\le t \le T$.
Such a radius exists because $\mu$ is equi-compactly supported.
Equations \ref{eq:ref_2} and \ref{eq:ref_3} ensure that the map $(t, x) \in [0,T]\times \R^d \mapsto \Gamma_{\mu(t)}(x)$ is continuous.
Thanks to Equation \ref{eq:ref_2}, we also have, for all $t\in [0,T]$:
$$\modu{\Gamma_{\mu(t)}(x) - \Gamma_{\mu(t)}(y)} \le \norm{V}_2 \norm{A}_2 R^2 \modu{x - y}.$$
Finally, Equation \ref{eq:ref_1} ensures that $\Gamma_{\mu(t)}$ grows less than linearly with $\modu{x}$—in fact, it is bounded with respect to $\modu{x}$—so that there is a unique global solution.
We can therefore define the associated flow $\phi_t(\mu)(x)$, such that the solution $r$ of Problem (\ref{eq:freezed_cauchy}) satisfies $r(t) = \phi_t(\mu)(x)$.

\paragraph{Step 2: \revision{local in-time} fixed point argument}
Let us set any initial condition $\bar \mu_0 \in \pcal_c(\R^d)$, and a time $T>0$ to be chosen later.
Define the map
\begin{align*}
    \fcal \colon \ccal_{\mathrm{co}}([0,T], \pcal_c(\R^d)) &\to \ccal_{\mathrm{co}}([0,T], \pcal_c(\R^d)) \\
     \mu &\mapsto \phi_t(\mu)_\sharp \bar \mu_0
\end{align*}
and equip the space $\ccal_{\mathrm{co}}([0,T], \pcal_c(\R^d))$ with the distance
$$\dcal_{p,T}(\mu,\nu)\coloneqq \max_{0\le t\le T}W_p(\mu(t),\nu(t)).$$
We look for a complete space $X$ preserved by $\fcal$ and such that $\fcal$ is a contraction on $X$, to apply a fixed point argument and show the existence and uniqueness of a solution to Equation (\ref{eq:transf_eq}) in $[0, T]$.
Let $X$ be the set of curves $\mu \in \ccal_{\mathrm{co}}([0,T], \pcal_c(\R^d))$ satisfying
$$\supp \mu(t) \subset \bbar(0, e^{\norm{V}_2t}R_0) $$
for all $0\le t\le T$, where $R_0$ is the smallest radius such that $\supp \bar \mu_0 \subset B_{R_0}$.
We show (see Appendix \ref{appsubsec:technical_lemmas}) that $\fcal(X)\subset X$, that $(X, \dcal_{p, T})$ is complete and that
$$\dcal_{p, T}(\fcal(\mu), \fcal(\nu)) \le f(T, R_0) \dcal_{p, T}(\mu, \nu)$$
for any $\mu, \nu\in X$, where $f(T, R_0)$ is a continuous function such that $f(T, R_0) \to_{T\to 0} 0$.
Let us then choose $T>0$ small enough to have $f(T, R_0) < 1$, so that $\fcal$ becomes a strict contraction.
According to Banach fixed-point theorem, $\fcal$ has a unique fixed point $\mu$.
Equivalently, Equation (\ref{eq:transf_eq}) with initial data $\bar \mu_0$ has a unique solution in $[0, T]$.

\paragraph{Step 3: constructing a global solution}
By repeating Step 2 with the updated initial condition $\phi_T(\mu)_\sharp \bar \mu_0$, we can extend $\mu$ to further times, and so on.
However, the time interval that we add at each step depends on $R_0$, which grows as we iterate the argument—and the length of the time interval shrinks accordingly.
Assume by contradiction that this method does not allow us to extend $\mu$ beyond some finite limiting time $T_\mathrm{lim}$.
It is easy to check that $\mu$ satisfies, for all $t\in [0, T_\mathrm{lim})$:
$$\supp \mu(t) \subset \bbar(0, e^{\norm{V}_2t}R_0),$$
where $R_0$ is the smallest radius such that $\supp \bar \mu_0 \subset B_{R_0}$.
Denote $R(t)\coloneqq e^{\norm{V}_2t}R_0$ for all $0\le t\le T_\mathrm{lim}$, and choose $T'>0$ small enough so that
$$f(T', R(T_\mathrm{lim})) < 1.$$
By using the same arguments as above, we can extend the restriction $\mu_{\lvert [0, T_\mathrm{lim} - T'/2]}$ to the time interval $[0, T_\mathrm{lim} + T'/2]$, which is a contradiction.
Therefore, $\mu$ can be extended to arbitrarily large times by iterating the fixed point argument, and so there exists a unique global solution to Problem (\ref{eq:transf_eq}) with initial condition $\bar \mu_0$.

\paragraph{Step 4: stability estimates}
We derive the stability estimates in Appendix \ref{apppar:stab_estim} adapting the traditional Dobrushin's method \cite{dobrushin1979vlasov,Go03,CCH14}, building on Equations \ref{eq:ref_1}, \ref{eq:ref_2}, \ref{eq:ref_3}.
\end{proof}

\revision{\begin{remark}Theorem \ref{thm:compact_support} does not cover the unnormalized attentions $\Gamma_\mu\satt{exp}$, $\Gamma_\mu\satt{lin}$ and $\Gamma_\mu\satt{ReLU}$.
Indeed, for these three types of attention and assuming that $\mu$ is supported in $B_R$, the velocity field $\Gamma$ cannot be bounded by $R$ up to a constant factor.
We rather have that
$\sup_{x\in B_R}\modu{\Gamma_\mu\satt{exp}(x)} \le \norm{V}_2 R e^{\norm{A}_2 R^2}$,
$\sup_{x\in B_R}\modu{\Gamma_\mu\satt{lin}(x)} \le \norm{V}_2\norm{A}_2 R^3$,
$\sup_{x\in B_R}\modu{\Gamma_\mu\satt{ReLU}(x)} \le \norm{V}_2\norm{A}_2 R^3$,
and when $V = I_d$, each bound is reached for a suitable Dirac measure.
Therefore, Lemma \ref{lem:support_control} does not hold, and we have no guarantee that $\mu$ stays compactly supported across the Transformer PDE dynamics, so that our proof technique is not applicable.
\end{remark}}

\subsection{Masked Self-Attention}

\revision{The case of masked self-attention is special: indeed, with the framework introduced in Section \ref{subsec:masked}, no upper bound of the form 
$$\modu{\Gamma_{\bar \mu}\satt{m}(\sigma, x) - \Gamma_{\bar \nu}\satt{m}(\sigma, x)} \le C(x) W_2(\bar \mu, \bar \nu)$$ 
can hold  (see Appendix \ref{appsubsec:masked}), even when the supports of $\mu$ and $\nu$ are constrained in a compact set and have the same position marginal (defined in Equation (\ref{eq:time_marginal})).}
A key remark to circumvent this issue is to notice that if $\bar \mu(t)$ is the solution of Equation (\ref{eq:continuity_equation_part_1}) initialized at $\bar\mu_0\in \mathcal{P}_c([0,1]\times \R^d)$, then the position marginal of $\bar \mu(t)$ is constant over time.
For any probability distribution $\theta\in \mathcal{P}([0, 1])$, let us denote $\pcal^\theta_c([0,1]\times \R^d)$ the set of compactly supported probability measures whose position marginal is equal to $\theta$.
The idea is to equip this space with an alternative to the Wasserstein distance, the conditional Wasserstein distance.

\begin{defi}[Conditional Wasserstein distance \cite{hosseini2023conditional}]
\label{def:conditional_OT}
    Let $d\in \mathbb N^*$ and $\bar \mu, \bar \nu \in \mathcal{P}_c^\theta([0,1]\times \R^d)$ \revision{such that $\theta(\{0\}) > 0$}.
    The conditional Wasserstein distance between $\bar\mu$ and $\bar \nu$ is defined as
    $$d(\bar \mu, \bar \nu) \coloneqq \int_0^1 \revision{W_1}(\bar\mu^\tau, \bar\nu^\tau) \dd \theta(\tau)$$
    where we have written
    $$\dd \bar \mu (\tau, x) =: \dd \theta(\tau) \dd \bar\mu^\tau (x)\quad \mbox{and}\quad \dd \bar \nu (\tau, x) =: \dd \theta(\tau) \dd \bar\nu^\tau (x)$$
    with the disintegration theorem.
\end{defi}

\begin{remark}
    One can also define the conditional $p$-Wasserstein distance for $p\ge 1$ as
    $$d(\bar \mu, \bar \nu) \coloneqq \left (\int_0^1 W_p(\bar\mu^\tau, \bar\nu^\tau)^p \dd \theta(\tau)\right )^{1/p}$$
    with the same notation as in Definition \ref{def:conditional_OT}.
    All our proofs adapt directly to this case, as the estimates rely on a $W_1$ estimate.
\end{remark}

The idea of Definition \ref{def:conditional_OT} is to constrain the transport plans so that they preserve the position marginal.
This allows us to prove the following well-posedness result.

\begin{theorem}
\label{thm:mask_compact_support}
    Let $k, d\in \N^*$.
    Let $Q, K\colon [0, +\infty) \to \R^{k\times d}$ and $V \colon [0, +\infty) \to \R^{d\times d}$ be three \revision{integrable} maps, modeling the evolution of parameters $Q, K, V$ across layers of the Transformer.
    Let $\bar \mu_0 \in \pcal_c([0, 1]\times \R^d)$ be a compactly supported initial condition. \revision{Denote $\theta$ its position marginal, defined as $\dd\theta(\sigma) = \int_{x\in \R^d}\dd\bar\mu_0(\sigma, x),$
    and assume that $\theta(\{0\}) > 0$, i.e., $\theta$ has a Dirac mass at $0$.}
    Let $\Gamma \in \{\Gamma\satt{SM}, \Gamma\satt{\revision{\ell^2}}, \Gamma\satt{sink}, \Gamma\satt{\sigma}\}$ be an unmasked\revision{, single-headed or multi-headed} velocity field associated with the parameters $Q(t), K(t), V(t)$ (with the convention (\ref{eq:MH_representation}) and $k=d$ for multi-head attention), and denote \revision{$\Gamma_{\bar\mu}(t,(\sigma,x)) \coloneqq \Gamma_{\bar \mu(t)}\satt{m}(\sigma,x)$} the corresponding masked attention map.
    Then, the masked Transformer PDE
    \begin{equation} 
    \label{eq:mask_transf_eq}
        \partial_t \bar \mu + \mathrm{div}(\bar\mu \Gamma_{\bar \mu}) = 0
    \end{equation}
    with initial condition $\bar \mu_0$ has a unique global weak solution $\bar \mu \in \ccal([0,+\infty), \pcal_c([0, 1]\times \R^d))$, with $\pcal_c([0, 1]\times \R^d)$ equipped with the conditional Wasserstein distance $d$.
    Moreover, let $R_0$ be the smallest radius such that $\supp \mu_0 \subset B_{R_0}$, where $\mu_0$ is the space marginal of $\bar \mu_0$, and define 
    $$R(t) \coloneqq e^{\int_0^t\norm{V(s)}_2 \dd s}R_0$$
    for $t\ge 0$, \revision{where $\norm{\cdot}_2$ is the spectral norm}.
    Then, the space marginal $\mu$ of the solution $\bar \mu$ satisfies 
    $$ \supp \mu(t) \subset B_{R(t)}$$
    for all $t\ge 0$.
    Finally, we have the following stability estimate.
    For all $R_0>0$ and $t>0$, there exists a constant $C(t, R_0, \revision{\theta(\{0\})})$ depending only on $t$, $R_0$\revision{, $\theta(\{0\})$ and $Q,K,V$} such that for any initial conditions $\bar \mu_0$ and $\bar \nu_0$ supported in $[0, 1]\times B_{R_0}$ with the same position marginal $\theta$, and denoting $\bar \mu$ and $\bar \nu$ the associated global solutions of Equation (\ref{eq:mask_transf_eq}), we have
    $$ d(\bar \mu(t),\bar \nu(t)) \le C(t, R_0, \revision{\theta(\{0\})}) d(\bar \mu_0,\bar \nu_0).$$
\end{theorem}

\revision{We defer the proof to Appendix \ref{appsubsec:masked}.}

\revision{\begin{remark}
    The assumption that the position marginal $\theta$ has a Dirac mass at zero is crucial for the proof to hold. This is due to the following bound:
    $$\modu{\tilde \Gamma_{\bar\mu}\satt{m}(\sigma, x) - \tilde \Gamma_{\bar\nu}\satt{m}(\sigma, x)}\le \frac{c(\modu{x}, R)}{\int_0^\sigma\dd\theta(\tau)}d(\bar\mu, \bar\nu),$$
    for any $\bar\mu, \bar\nu\in \pcal([0,1]\times B_R)$, which is sharp up to a constant factor independent of $\theta$. When $\theta(\{0\}) = 0$, the right-hand side diverges—and, for suitable $\bar\mu$ and $\bar\nu$, the left-hand side as well, which is an obstruction to proving the contractivity of $\fcal$.
    A way to ensure this assumption is always satisfied is to put an artificial and arbitrarily small Dirac mass at $(0, 0_{d})$, and encode actual tokens in $(0,1]\times \R^d$. However, this slightly changes the computation of masked attention—relaxing this assumption is left for future work.
\end{remark}}


\section{Clustering For Gaussian Initial Data}
\label{sec:gaussian}

Let us now tackle the particular case where the initial data is a Gaussian measure.
This is an oversimplified model for real data; still, it provides a simple example of anisotropic data, which allows us to study the evolution of the anisotropy across the dynamics.
Indeed, the Gaussian case has the nice property that several types of unmasked self-attention introduced in Section \ref{sec:att_variants} have a closed form when $\mu$ is Gaussian—more precisely, $\Gamma_\mu$ becomes an affine function.
This implies that Gaussian input measures stay Gaussian during the dynamics $\partial_t\mu + \mathrm{div}(\mu \Gamma_\mu) = 0$, so that we can summarize the evolution as a system of ordinary differential equations linking the expectation and the covariance matrix of $\mu$, which can then be studied more easily than the initial PDE.
This viewpoint is not limited to self-attention layers: Appendix \ref{appsec:drifting_models} records a closely related application to the Gaussian-kernel drifting field introduced in \cite{deng2026generative}. This field is the difference between the target and current normalized Gaussian-kernel attention fields, so the Gaussian computations below also provide closed mean and covariance dynamics for drifting between Gaussian measures.

\subsection{Softmax Self-Attention}
\label{subsec:trad_gaussian}

Let us start with Softmax self-attention.
We have the following closed form for the velocity field $\Gamma_\mu\satt{SM}$ when $\mu$ is a Gaussian probability measure, \revision{which is key for proving the stability of Gaussians and deriving ODEs on their expectation and covariance}.
Recall that $A \coloneqq K^\top Q$.

\begin{lemma}
    \label{lem:gaussian_formula_gamma}
    Let $\mu = \ncal(\alpha, \Sigma)$ be a Gaussian measure on $\R^d$.
    Then for all $x\in \R^d$ it holds
    $$\Gamma_\mu\satt{SM}(x) = V(\alpha + \Sigma A x).$$
\end{lemma}

Note that this covers the case of linear attention as well by absorbing the $\varepsilon$ in the matrix $A$ (see Subsection \ref{subsec:linear}).
When the initial data is Gaussian, we can then rewrite the dynamics as follows.

\begin{prop}
    \label{prop:gaussian_dynamics_trad}
    Let $k, d\in \N^*$.
    Let $Q, K\colon [0, +\infty)\to \R^{k\times d}$ and $V \colon [0, +\infty)\to \R^{d\times d}$ be three continuous maps, modeling the evolution of parameters $Q,K,V$ across layers of the Transformer.
    Consider a Gaussian initial condition $\mu_0 = \ncal(\alpha_0, \Sigma_0)$ with $\alpha_0 \in \R^d$ and $\Sigma_0 \in \R^{d\times d}$ positive definite.
    Then, the Transformer PDE
    $$\partial_t \mu + \mathrm{div}(\mu \Gamma_\mu\satt{SM}) = 0$$
    associated to Softmax self-attention with parameters $Q(t), K(t), V(t)$ has a unique maximal solution $\mu$ on $[0, t_\mathrm{max})$, such that $\mu(t)$ is Gaussian for all $t\in [0, t_\mathrm{max})$.
    Moreover, denoting $\mu(t) \eqqcolon \ncal(\alpha(t), \Sigma(t))$ and $A\coloneqq K^\top Q$, we have
    \begin{equation}
        \label{eq:general_eq_sigma}
        \dot \Sigma = V\Sigma A \Sigma + \Sigma A^\top \Sigma V^\top.
    \end{equation}
    and
    \begin{equation*}
        \dot \alpha = V(I_d + \Sigma A) \alpha.
    \end{equation*}
\end{prop}

The proof is in Appendix \ref{appsubsec:gaussian_proofs}.
In the rest of the section, let us assume to simplify the problem that $Q, K, V$ are constant over time. \revision{Under a commutation assumption on these matrices, we have the following classification of possible behaviors of $\Sigma$.}

\revision{\begin{prop}
    \label{prop:trad_theoretical_result}
    Consider the matrix-valued differential equation
    \beq 
    \label{eq:trad_cov} 
    \dot \Sigma = V\Sigma A \Sigma + \Sigma A^\top \Sigma V^\top.
    \eeq
    Assume that the matrices $V$ and $V^\top$ commute with $\Sigma_0$ and $V A + A^\top V^\top$.
    Then, Equation (\ref{eq:trad_cov}) with initial condition $\Sigma_0 \succ 0$ has a unique maximal solution, defined as
        $$\Sigma(t) = \left (\Sigma_0^{-1} - t(V A + A^\top V^\top) \right )^{-1}$$
    for all $t\ge 0$ such that the matrix $\Sigma_0^{-1} - t (V A + A^\top V^\top)$ is invertible.
    The behavior of the solution \revision{therefore} depends on the sign of the eigenvalues of $VA + A^\top V^\top$ as follows.
    \begin{enumerate}
        \item If $VA + A^\top V^\top \preceq 0$, then Equation (\ref{eq:trad_cov}) has a global solution. Moreover, the matrix $\Sigma(t)$ converges to a limit $\Sigma^* \succeq 0$ satisfying
        $\lambda_i(\Sigma^*) = 0$ for all $i$ such that $\lambda_i(VA + A^\top V^\top) < 0$.
        Therefore, the mass concentrates on an affine subspace of $\R^d$, of dimension equal to the multiplicity of 0 as an eigenvalue of $VA + A^\top V^\top$.
        \item If $VA + A^\top V^\top$ has a positive eigenvalue, then the largest eigenvalue of $\Sigma(t)$ goes to $+\infty$ in finite time.
    \end{enumerate}
\end{prop}}

\revision{
\begin{remark}
    Note that in the even simpler case where $\Sigma_0$ and $VA + A^\top V^\top$ commute, the limit of $\Sigma(t)$ becomes explicit. Indeed, as these matrices are simultaneously diagonalizable, we get 
    $$\lambda_i(\Sigma(t)) = \left (\lambda_i(\Sigma_0)^{-1} - t \lambda_i(VA + A^\top V^\top)\right )^{-1}.$$
    In this case, if $VA + A^\top V^\top \preceq 0$ then $\Sigma(t)$ converges to the matrix $\Sigma^* \succeq 0$ defined as
    $$\lambda_i(\Sigma^*) = \begin{cases} 0 ~~~~~~~~~~~~~~~~~~\mbox{if }\lambda_i(VA + A^\top V^\top) < 0, \\
    \lambda_i(\Sigma_0) ~~~~~~~~~~~\mbox{if }\lambda_i(VA + A^\top V^\top) = 0.\end{cases}$$
\end{remark}
Hence, under the assumption of Proposition \ref{prop:trad_theoretical_result} and for $VA + A^\top V^\top\neq 0$, any Gaussian solution of the Transformer PDE either i) converges when $t\to +\infty$ to a degenerate Gaussian measure, i.e., with a low-rank covariance matrix, or ii) leaves the finite-covariance Gaussian regime in finite time.
Case i) corresponds to a collapse of the mass along the axes associated to zero eigenvalues of the limiting covariance $\Sigma^*$, and can be seen as the continuous parallel of the clustering for discrete tokens highlighted in \cite{vaswani2017attention, geshkovski2023emergence}.
The proof of Proposition \ref{prop:trad_theoretical_result} is in Appendix \ref{appsubsec:gaussian_proofs}.
We complement this analysis with numerical experiments (Section \ref{subsec:experiments}), which show that this behavior goes beyond the assumption of Proposition \ref{prop:trad_theoretical_result}.}

\revision{\begin{remark} As a more intuitive illustration of Proposition \ref{prop:trad_theoretical_result}, let us look at Equation (\ref{eq:general_eq_sigma}) in dimension 1, writing $(s, a, v)$ for $(\Sigma, A, V)$ in this case.
\begin{prop}
\label{prop:dim_1_trad}
    Let $a, v \in \R$ and $s_0 \in \R_+^*$.
    Consider Equation (\ref{eq:trad_cov}) in dimension 1, which reads
    $\dot s = 2av s^2$
    with initial condition $s(0) = s_0$.
    There is a unique maximal solution to this equation, defined as
    $$s(t) = (s_0^{-1} - 2vat)^{-1}$$
    for $t\in [0, t_\mathrm{max})$, where $t_\mathrm{max}\coloneqq (2va s_0)^{-1}$ if $av >0$, and $t_\mathrm{max}\coloneqq +\infty$ if $av\le 0$.
\end{prop}
In dimension 1, the dynamics on Gaussians are therefore simple: the \revision{covariance matrix} either i) goes to zero or ii) blows up, i.e., becomes infinite, in finite time.
Translating these behaviors in terms of the mass distribution induced by the Gaussian measure gives either i) clustering of the mass towards a single Dirac when $t\to +\infty$ or ii) escaping of the mass to infinity in finite time.
\end{remark}}

\revision{To complement the analysis,} the following paragraphs \revision{showcase two properties of} Equation (\ref{eq:general_eq_sigma}).
\revision{Note that in the particular case where $V = I_d$, Equation (\ref{eq:general_eq_sigma}) takes the form of a Riccati equation \cite{wonham1968matrix}.}

\

\paragraph{Rank is preserved \revision{across the dynamics}}
\revision{We show in Appendix \ref{applem:rank_preservation} that Equation (\ref{eq:general_eq_sigma}) preserves the rank of $\Sigma$.
In particular, the case where the initial covariance matrix is of rank 1, which corresponds to a maximally anisotropic distribution of tokens, leads to a simpler differential equation.}

\begin{lemma}
    \label{lem:rank_1_dynamics}
    Let $A,V\in \R^{d\times d}$ and $u_0\in\R^d$.
    Then the maximal solution $\Sigma(t)$ of the equation
    $$\dot \Sigma = V\Sigma A \Sigma + \Sigma A^\top \Sigma V^\top$$
    with initial data $\Sigma_0\coloneqq u_0 u_0^\top$
    is of rank 1 for all times where it is defined.
    Moreover, denoting $\Sigma(t)\eqqcolon u(t)u(t)^\top$ with $u(t)\in \R^d$, we have
    \begin{equation}\label{eq:rank_1_evol}\dot u = (u^\top A u) Vu.\end{equation}
\end{lemma}

\revision{The evolution of $u(t)$ is characterized by a non-linear coupling between the radial dynamics and the angular dynamics: while $V$ determines the flow's direction, the scalar field $u^\top A u$ acts as a non-uniform time-scaling factor. Therefore, unless $u_0$ is a joint eigenvector of $V$ and $A$, or these matrices share a specific algebraic relationship (like commutativity),} finding a closed form for $u(t)$ seems a difficult problem.
\revision{Still, up to an implicit time rescaling, one easily checks that the behavior of the solution is the following.}
\revision{\begin{proposition}
    \label{prop:rescaled_closed_form}
    Let $A,V\in \R^{d\times d}$ and $u_0\in\R^d$.
    Denote $u(t)$ the solution of Equation (\ref{eq:rank_1_evol}).
    We have
    $$u(t) = e^{\tau(t) V}u_0,$$
    where $\tau(t)$ is the time rescaling that is solution of $\dot \tau = u^\top A u$ such that $\tau(0) = 0$.
\end{proposition}}

\revision{Moreover}, from Lemma \ref{lem:rank_1_dynamics} we can easily characterize rank-1 stationary points of the dynamics.

\begin{lemma}
    Let $A, V\in \R^{d\times d}$.
    The rank-1 matrix $uu^\top$ with $u\in \R^d$ is a stationary point of Equation (\ref{eq:general_eq_sigma}) if and only if $u\in \ker V$ or $u^\top (A + A^\top) u = 0$.
\end{lemma}

Note that the set $\{u\in \R^d : u^\top (A + A^\top) u = 0\}$ is the isotropic cone associated to the quadratic form $A + A^\top$.
Its geometry depends on the signature of $A + A^\top$.
It is non-trivial if and only if $A + A^\top$ has zero eigenvalues or eigenvalues of the opposite sign.

\

\paragraph{Stationary points have low rank when $V$ is identity and $A$ is symmetric}
Another interesting case is to look at stationary points of Equation (\ref{eq:general_eq_sigma}) when $V = I_d$ and $A$ is symmetric.
We have seen in Proposition \ref{prop:trad_theoretical_result} that when $A+A^\top \preceq 0$, limiting points $\Sigma^*$ must satisfy
$$\mathrm{rk} \Sigma^* \le \mathrm{rk} A.$$
We generalize this result to any symmetric matrix $A$ as follows (see Appendix \ref{appsubsec:gaussian_proofs} for the proof).

\begin{prop}
\label{prop:stat_points}
    Let $A\in \R^{d\times d}$ be a symmetric matrix and assume that $V = I_d$.
    If a symmetric matrix $\Sigma \in \R^{d\times d}$ is solution of 
    $$V\Sigma A \Sigma + \Sigma A^\top \Sigma V^\top = 0,$$
    then
    $$\mathrm{rk} \Sigma \le \dim \ker A + \min(\sharp\{ \mathrm{positive\ eigenvalues\ of\ }A\}, \sharp\{\mathrm{negative\ eigenvalues\ of\ }A \}).$$
\end{prop}

Therefore, when $A$ is symmetric and $V = I_d$, the stationary points of Equation (\ref{eq:general_eq_sigma}) have low rank—smaller than $d/2$ when $A$ is invertible, for instance.
Note that when $A$ has a positive eigenvalue, we have seen in Proposition \ref{prop:trad_theoretical_result} that the dynamics with a positive definite initialization blow up in finite time. In that case, stationary points are never reached.

\

\subsection{Multi-Head Softmax Self-Attention}
Lemma \ref{lem:gaussian_formula_gamma} allows us to derive the following result for multi-head Softmax self-attention.

\begin{prop}
    Let $d\in \N^*$.
    For \revision{$1\le h\le H$}, let
    \(Q^{\revision{(h)}}, K^{\revision{(h)}} \colon [0, +\infty)\to \R^{d/H\times d}\)
    and \(V^{\revision{(h)}} \colon [0, +\infty)\to \R^{d\times d}\) be
    continuous maps. These maps model the evolution of parameters across layers
    of the Transformer.
    Consider a Gaussian initial condition $\mu_0 = \ncal(\alpha_0, \Sigma_0)$ with $\alpha_0 \in \R^d$ and $\Sigma_0 \in \R^{d\times d}$ positive definite.
    Then, the multi-head Transformer PDE
    $$\partial_t \mu + \mathrm{div}(\mu \Gamma_\mu\satt{MH}) = 0$$
    associated to multi-head self-attention with parameters $(Q^{\revision{(h)}}, K^{\revision{(h)}}, V^{\revision{(h)}})_{1\le h\le H}$ via Equation (\ref{eq:MH_representation}) has a unique maximal solution $\mu$ on $[0, t_\mathrm{max})$, such that $\mu(t)$ is Gaussian for all $t\in [0, t_\mathrm{max})$.
    Moreover, denoting $\mu(t) \eqqcolon \ncal(\alpha(t), \Sigma(t))$ \revision{and $A^{(h)}\coloneqq (K^{(h)})^\top Q^{(h)}$}, we have
    \begin{equation*}
        \dot \Sigma = \sum_{h = 1}^{H} V^{(h)}\Sigma A^{(h)} \Sigma + \Sigma A^{(h)\top} \Sigma (V^{(h)})^\top.
    \end{equation*}
    and
    \begin{equation*}
        \dot \alpha = \sum_{h=1}^H V^{(h)}(I_d + \Sigma A^{(h)}) \alpha.
    \end{equation*}
\end{prop}

We only analyze the dynamics associated with Softmax multi-head attention numerically (see Section~\ref{subsec:experiments}). \revision{Note that one can consider the multi-head version of any of the presented attention variants: if Gaussians are preserved by a single-headed attention variant, then its multi-head version also preserves Gaussians.}

\subsection[L2 Self-Attention]{$\revision{\ell^2}$ Self-Attention}
\label{subsec:L2_gaussian}

In the case of $\revision{\ell^2}$ self-attention, we have the following closed form for the velocity field when $\mu$ is Gaussian.

\begin{lemma}
    Let $\mu = \ncal(\alpha, \Sigma)$ be a Gaussian measure on $\R^d$.
    Then for all $x\in \R^d$ it holds
    $$\Gamma_\mu\satt{\revision{\ell^2}}(x) = V(\Sigma^{-1} + 2K^\top K)^{-1}(\Sigma^{-1}\alpha + 2K^\top Qx).$$
\end{lemma}

As $\Gamma_\mu\satt{\revision{\ell^2}}$ is an affine function when $\mu$ is Gaussian, Gaussianity is preserved by the $\revision{\ell^2}$ Transformer PDE, so we can rewrite it as two coupled matrix ODEs, with the same method as for Softmax self-attention.

\begin{prop}
\label{prop:L2_attention_gaussian_case}
    Let $d\in \N^*$.
    Let $Q, K, V \colon [0, +\infty) \to \R^{d\times d}$ be continuous functions.
    Consider a Gaussian initial condition $\mu_0 = \ncal(\alpha_0, \Sigma_0)$ with $\alpha_0 \in \R^d$ and $\Sigma_0 \in \R^{d\times d}$ positive definite.
    Then, the Transformer equation
    $$\partial_t \mu + \mathrm{div}(\mu \Gamma_\mu\satt{\revision{\ell^2}}) = 0$$
    associated to $\revision{\ell^2}$ self-attention has a unique maximal solution $\mu$, such that $\mu(t)$ is Gaussian for all $t\in [0, t_\mathrm{max})$.
    Moreover, denoting $\mu(t) \eqqcolon \ncal(\alpha(t), \Sigma(t))$ \revision{and $A\coloneqq K^\top Q$}, we have
    \begin{equation}
        \label{eq:L2_covariance}
        \dot \Sigma = 2V\Sigma (I_d + 2K^\top K\Sigma)^{-1} A \Sigma + 2\Sigma A^\top (I_d + 2K^\top K\Sigma)^{-1} \Sigma V^\top.
    \end{equation}
    and
    \begin{equation*}
        \dot \alpha = V (\Sigma^{-1} + 2 K^\top K)^{-1} (\Sigma^{-1} + 2A)\alpha.
    \end{equation*}
\end{prop}
\revision{\begin{remark} Note that a similar result as Proposition \ref{prop:L2_attention_gaussian_case} has been stated in a completely different context in \cite[Page 229]{lacombe1999presentation}, only in dimension 1, assuming $Q = K$ and with a slight modification of the velocity field: $\Gamma_\mu(x) = \nu x + \Gamma_\mu\satt{\ell^2}(x)$. \end{remark}}
Contrary to what happens with Softmax self-attention, we show (see Appendix \ref{appsubsec:gaussian_proofs}) that the dynamics (\ref{eq:L2_covariance}) cannot blow up in finite time.

\begin{lemma}
    \label{lem:well_posed_L2}
    Let $\Sigma_0 \succ 0$.
    The matrix-valued differential equation
    \begin{equation*}
        \dot \Sigma = 2V\Sigma (I_d + 2K^\top K\Sigma)^{-1} A \Sigma + 2\Sigma A^\top (I_d + 2K^\top K\Sigma)^{-1} \Sigma V^\top
    \end{equation*}
    initialized at $\Sigma_0$ has a unique global solution.
\end{lemma}

In this sense, $\revision{\ell^2}$ self-attention is therefore smoother than Softmax self-attention.
This nicely complements—and is connected to—the result in \cite{kim2021lipschitz} showing that $\revision{\ell^2}$ self-attention is globally Lipschitz continuous, contrary to Softmax self-attention.
As the equation on $\Sigma$ is more involved in the $\revision{\ell^2}$ case than in the Softmax self-attention case, let us focus only on dimension 1.
A clustering phenomenon appears as for Softmax self-attention.

\begin{prop}
\label{prop:dim_1_L2}
    Let $q, k, v \in \R$ and $s_0 \in \R_+^*$.
    Denote $a\coloneqq qk$ and consider the differential equation
    $$\dot s = \frac{4avs^2}{1 + 2k^2s}$$
    with initial condition $s(0) = s_0$.
    \begin{enumerate}
        \item If $av > 0$, then $s(t)\to +\infty$ when $t\to +\infty$.
        \item If $av < 0$, then $s(t)\to 0$ when $t\to +\infty$.
    \end{enumerate}
\end{prop}

\begin{proof}
If $av > 0$, then $\dot s > 0$ so $s$ is increasing. Hence, it converges in $(s_0, +\infty]$, and it cannot converge to a finite value as $0\notin [s_0, +\infty)$ is the only stationary point of the equation.

It $av < 0$, then $s$ is decreasing, so it converges in $[-\infty, s_0)$. The only stationary point in this interval is $0$, and $s_0 > 0$, which proves the claim as $s(t)$ is continuous.
\end{proof}

In Proposition \ref{prop:dim_1_L2}, case 1 corresponds, in terms of the Gaussian measure of covariance $s(t)$, to the mass spreading to infinity, while case 2 induces a clustering of the mass into one single Dirac.
The 1-dimensional case for $\revision{\ell^2}$ self-attention is therefore close to Proposition \ref{prop:dim_1_trad} (dimension 1 for Softmax self-attention), where the finite-time blow-up of $s$ is replaced by a divergence when $t\to +\infty$.
Following this remark, we point out numerically in Section \ref{subsec:experiments} that the Softmax and $\revision{\ell^2}$ self-attention dynamics are very close except when the former blows up.
We also show in Section \ref{subsec:experiments} that the clustering behavior observed with Softmax self-attention also occurs with $\revision{\ell^2}$ self-attention.

\subsection{Sinkhorn Self-Attention}
\label{subsec:sink_self_attention}

Let us finally focus on Sinkhorn self-attention.
We start by computing a closed form for the velocity field $\Gamma_{\mu, \varepsilon}^{(\text{sink})}$ when $\mu$ is a Gaussian probability measure.

\begin{lemma}
    \label{lem:expression_gamma_sinkhorn}
    Let $\mu = \ncal(\alpha, \Sigma)$ be a Gaussian measure on $\R^d$.
    Let $Q, K \in \R^{k\times d}$ be two matrices, denote $A \coloneqq K^\top Q$ and assume that $A$ and $\Sigma$ are invertible.
    Then, for all $x\in \R^d$, it holds
    $$\Gamma\satt{sink}_{\mu, \varepsilon}(x) = \frac1\varepsilon V(I_d - A^{-\top}\Sigma^{-1}C)\alpha + \frac1\varepsilon VA^{-\top}\Sigma^{-1}C x,$$
    where
    $$C \coloneqq \Sigma^{1/2} (\Sigma^{1/2} A^\top \Sigma A \Sigma^{1/2} +\frac{\varepsilon^2}{4} I_d)^{1/2}\Sigma^{-1/2} -  \frac\varepsilon2 I_d.$$
\end{lemma}

The proof is in Appendix \ref{appsubsec:gaussian_proofs}.
Building on this result, we obtain with the same method as for Proposition \ref{prop:gaussian_dynamics_trad} two matrix ODEs summarizing the Sinkformer PDE on Gaussian measures.

\begin{prop}
    Let $d\in \N^*$.
    Let $Q, K, V \colon [0, +\infty) \to \R^{d\times d}$ be continuous functions.
    Consider a Gaussian initial condition $\mu_0 = \ncal(\alpha_0, \Sigma_0)$ with $\alpha_0 \in \R^d$ and $\Sigma_0 \in \R^{d\times d}$ positive definite.
    Then, the Sinkformer PDE
    $$\partial_t \mu + \mathrm{div}(\mu \Gamma_{\mu, \varepsilon}\satt{sink}) = 0$$
    associated to Sinkhorn self-attention has a unique maximal solution $\mu$, such that $\mu(t)$ is Gaussian for all $t\in [0, t_\mathrm{max})$.
    Moreover, denoting $\mu(t) \eqqcolon \ncal(\alpha(t), \Sigma(t))$ \revision{and $A\coloneqq K^\top Q$}, we have
    \begin{equation*}
        \dot \Sigma = \frac1\varepsilon V A^{-\top} \Sigma^{-1} C \Sigma + \frac1\varepsilon \Sigma C^\top \Sigma^{-1} A^{-1} V^\top,
    \end{equation*}
    with $C \coloneqq \Sigma^{1/2} (\Sigma^{1/2} A \Sigma A^\top \Sigma^{1/2} +\frac{\varepsilon^2}{4} I_d)^{1/2}\Sigma^{-1/2} -  \frac\varepsilon2 I_d,$
    and
    \begin{equation*}
        \dot \alpha = \frac1\varepsilon V\alpha.
    \end{equation*}
\end{prop}

Remarkably, in this case, and contrary to $\revision{\ell^2}$ and Softmax self-attention, the expectation $\alpha$ and the covariance matrix $\Sigma$ evolve independently.
We have
$$\alpha(t) = e^{tV / \varepsilon}\alpha_0.$$

\begin{remark}
    The fact that the expectation and the covariance evolve independently is not that surprising.
    Indeed, \cite{sander2022sinkformers} show that, under certain assumptions on the parameters, the Sinkformer PDE corresponds to the Wasserstein gradient flow of a functional that takes the form $\phi(\alpha) + \psi(\Sigma)$ on Gaussians, where $\alpha$ and $\Sigma$ are respectively the expectation and the covariance matrix of the Gaussian measure (see Equation (\ref{GaussFunc-sink})).
    Then, the Wasserstein gradient flow of $\phi(\alpha) + \psi(\Sigma)$, which stays in the space of Gaussian measures, corresponds to following a Euclidean gradient flow on $\phi(\alpha)$ and a Bures-Wasserstein gradient flow on $\psi(\Sigma)$ (see Section \ref{subsec:BW_flow}).
\end{remark}

Contrary to the expectation equation, the covariance equation is much more challenging to be analyzed theoretically.
We only consider the 1-dimensional case.

\begin{prop}
    Let $q, k, v\in \R$ and $a\coloneqq qk$.
    Assume that $\varepsilon = 1$, as it can be absorbed in $a$ and $v$.
    The covariance equation associated with Sinkhorn self-attention in dimension 1 reads
    $$\dot s = \frac{v}{a}(\sqrt{4a^2 s^2 +1} - 1).$$
    For any initial data $s_0\in \R_+^*$, this equation has a global solution $s(t)$, such that:
    \begin{enumerate}
        \item if $av > 0$, then $s(t)\to +\infty$ when $t\to +\infty$,
        \item if $av < 0$, then $s(t)\to 0$ when $t\to +\infty$.
    \end{enumerate}
\end{prop}

The proof is the same as for Proposition \ref{prop:dim_1_L2}.
Therefore, in dimension 1, Sinkhorn self-attention is smoother than Softmax self-attention, and induces the same clustering phenomenon when $av < 0$.

\subsection{Experiments}
\label{subsec:experiments}

Our theoretical results about the Gaussian case typically rely on restrictive hypotheses on the weight matrices.
We therefore complement the theoretical analysis with numerical experiments, to understand better the properties of the evolution and of the limiting covariance matrix $\Sigma$ for each type of self-attention, depending on the parameters.
In this section, we call a random matrix a matrix whose coordinates are i.i.d. Gaussian.
We focus on Softmax, $\revision{\ell^2}$ and (Softmax) multi-head attention.
Sinkhorn self-attention is investigated in Figure \ref{appfig:sinkhorn}.

\

\paragraph{Qualitative behaviors for $d=2$}
The 2-dimensional case allows us to visualize the different dynamics.
Covariance matrices are then symmetric positive semidefinite $2\times 2$ matrices, and can be represented in a 3-dimensional space with the following change of coordinates:
$$\begin{pmatrix} a & c \\ c & b \end{pmatrix}\in \scal_2^+ \mapsto (x, y, z) \coloneqq (a-b, 2c, a+b) \in \R^3.$$
With this choice of parametrization, the set of symmetric nonnegative $2\times 2$ matrices becomes the cone of equation $z \ge \sqrt{x^2 + y^2}$ (see Figure \ref{fig:well_posed_dynamics}), and the evolution of a covariance matrix $\Sigma(t)$ is represented as a curve inside this cone.
Note that the boundary of the cone corresponds to degenerate matrices, i.e., of rank 1 or 0.

\begin{figure}
\centering
$$\begin{array}{ccccc}
&\mbox{(a) Conv. to} & \mbox{(b) Conv. to} & \mbox{(c) Conv. to} & \mbox{(d) Conv. to} \\
&\mbox{zero} & \mbox{a line} & \mbox{a plane} & \mbox{two lines}\\
\includegraphics[width=0.16\textwidth]{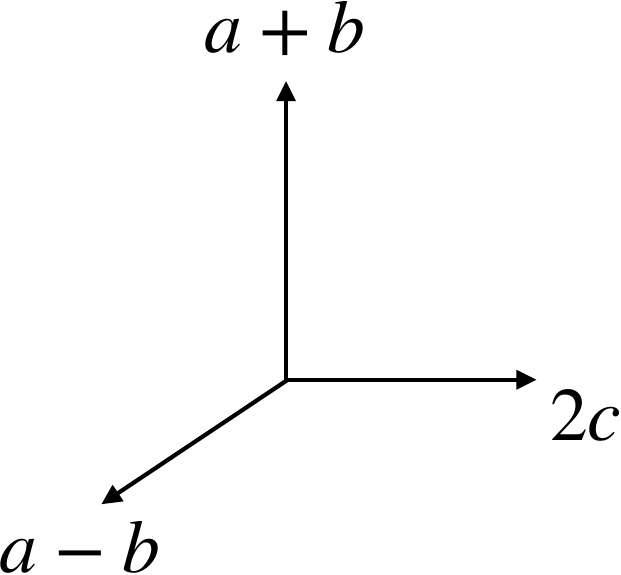}& \includegraphics[width=0.15\textwidth]{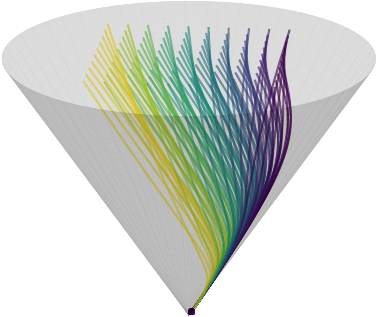} & \includegraphics[width=0.16\textwidth]{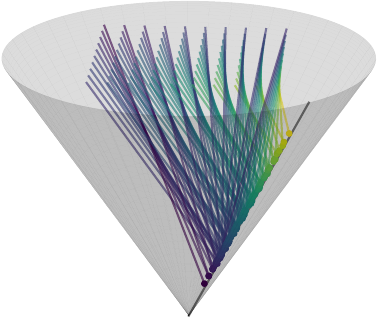}&\includegraphics[width=0.16\textwidth]{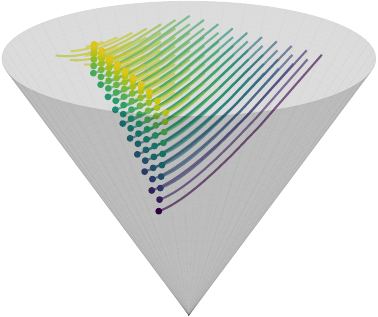}&\includegraphics[width=0.16\textwidth]{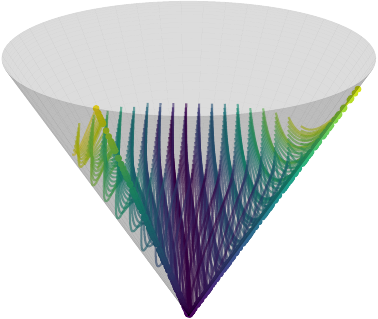}
\end{array}$$
\caption{Evolution of the covariance matrix of a 2-dimensional Gaussian measure that goes through the Transformer PDE. The plots (a), (b), and (d) were obtained with Softmax self-attention, respectively with (a) $V$ random and $A + A^\top \prec 0$, (b) $V=I_2$ and $A + A^\top \preceq 0$ of rank 1 and (d) $V$ and $A$ chosen specifically to obtain this pattern. The plot (c) corresponds to multi-head self-attention with $V = I_2$ and $A + A^\top \preceq 0$ of rank 1.}
\label{fig:well_posed_dynamics}
\end{figure}

\begin{figure}
    \centering
    $$\begin{array}{ccccc}
    \Sigma \mapsto \Sigma / \tr\Sigma&\mbox{(a) Blow-up} & \mbox{(b) Blow-up} & \mbox{(c) Div.} & \mbox{(d) Conv./blow-up} \\
    \includegraphics[width=0.16\textwidth]{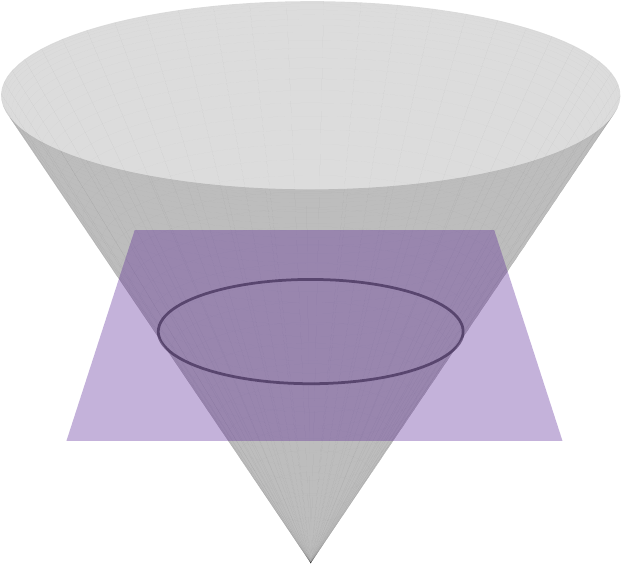}~~~~& \includegraphics[width=0.13\textwidth]{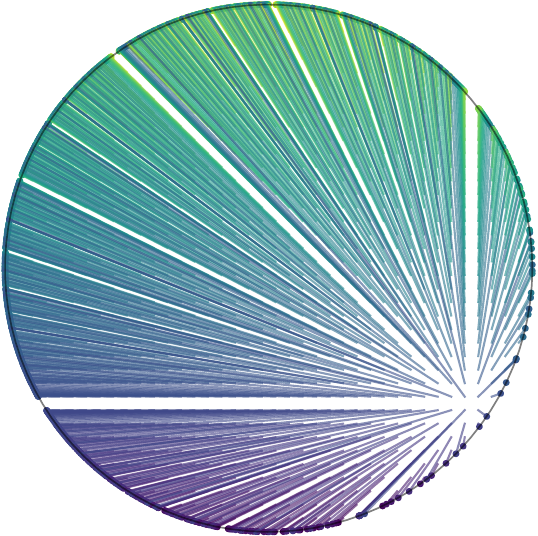} & \includegraphics[width=0.13\textwidth]{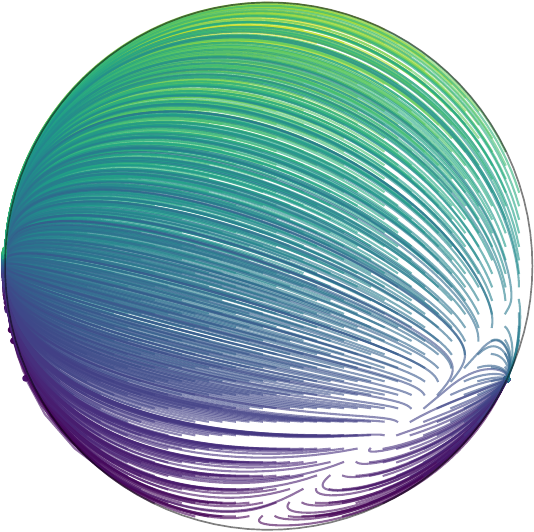}&\includegraphics[width=0.13\textwidth]{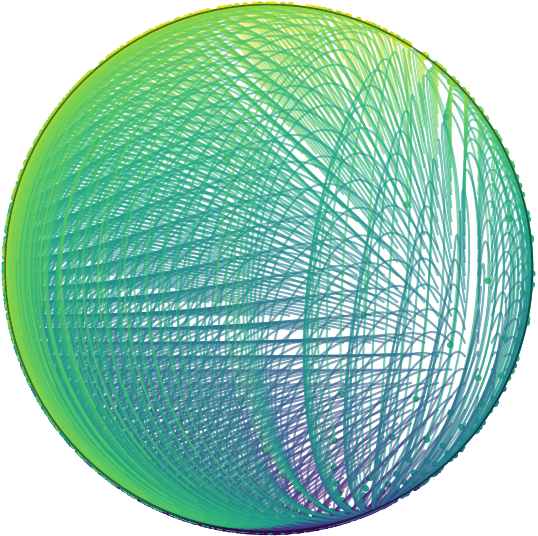}&\includegraphics[width=0.13\textwidth]{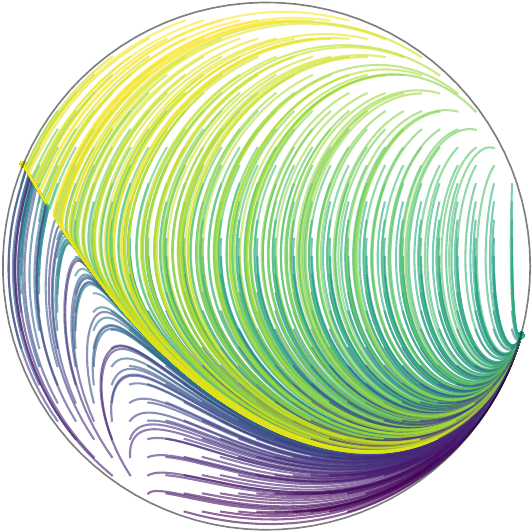}\\
    &\mbox{SM att.} & \mbox{MH att.} & \mbox{$\revision{\ell^2}$ att.} & \mbox{SM att.} 
    \end{array}$$
    \caption{Projection on the set of trace-1 matrices of the dynamics of the covariance matrix of a Gaussian measure that goes through the Transformer PDE, in cases where curves blow up or diverge. The plots (a), (b), and (c) were obtained with the same parameters ($V = I_2$ and $Q,K$ fixed so that $A+A^\top \succ 0$), respectively for Softmax, multi-head Softmax and $\revision{\ell^2}$ self-attention. In (a) and (b), the dynamics explode in finite time, while it is well-posed (but diverging) in (c). Finally, in (d), some of the initializations lead to a finite-time blow-up (purple curves) while others lead to convergence of the covariance matrix (yellow/green curves). (d) was obtained with Softmax self-attention but we observed a very similar behavior with $\revision{\ell^2}$ and multi-head self-attention (see Figure \ref{appfig:both} in Appendix \ref{appsubsec:experiments}).}
    \label{fig:explosion}
\end{figure}

We highlight different behaviors which correspond to different choices of parameters and self-attention types.
In each case, we set our initial points on a two-dimensional grid, orthogonal to the $z$ axis (so all initial matrices have the same trace), and we run a Euler discretization of the ODE satisfied by $\Sigma$, with a fixed step-size $\tau$.
So the equation $\dot \Sigma = g(\Sigma)$ is discretized as
$$\Sigma_{k + 1} = \Sigma_k + \tau g(\Sigma_k)$$
starting from some initial point $\Sigma_0$.
Depending on the parameters and the initial data, we observe a range of different behaviors, illustrated in Figure \ref{fig:well_posed_dynamics} and \ref{fig:explosion}.
Here are some typical evolutions of the covariance matrix.
\begin{enumerate}
    \item Convergence to zero: independently of the initial data, all trajectories converge to the origin. We observe empirically that this is the case when the matrix $A + A^\top$ is negative definite, independently of $V$, and for Softmax, multi-head and $\revision{\ell^2}$ self-attention (see Figure \ref{appfig:to_zero}). However, this is not the case for Sinkhorn self-attention (see Figure \ref{appfig:sinkhorn}). Figure \ref{fig:well_posed_dynamics} (a) was obtained with Softmax self-attention for a random choice of $A$ provided that $A + A^\top \prec 0$, and a random choice of $V$.
    \item Convergence to a line on the boundary: all trajectories converge to a line on the boundary of the PSD cone, which is the case for Softmax and $\revision{\ell^2}$ self-attention when $A + A^\top$ is negative and of rank 1. Figure \ref{fig:well_posed_dynamics} (b) corresponds to this case with $V$ random and Softmax self-attention.
    \item Convergence to a plane: we observe this behavior with \revision{Softmax} MH self-attention when $A + A^\top$ is negative and of rank 1. Figure \ref{fig:well_posed_dynamics} (c) corresponds to this case for $V = I_2$. In that case, limiting covariance matrices have full rank and no clustering occurs, which seems to be specific to multi-head attention.
    \item Divergence in finite or infinite time: some initializations lead to a divergence of at least one of the eigenvalues of $\Sigma$, in finite time or when $t\to + \infty$. In that case, we stop the evolution at some fixed threshold for $\norm{\Sigma}$, and then we plot in 2D the trajectory of the rescaled covariance matrices $\frac{\Sigma}{\tr(\Sigma)}$, which belong to a horizontal slice of the cone corresponding to matrices of trace 1 (see Figure \ref{fig:explosion}).
    The plots (a), (b), and (c) of Figure \ref{fig:explosion} correspond to the same set of parameters, respectively for Softmax, multi-head, and $\revision{\ell^2}$ self-attention.
\end{enumerate}
Therefore, when the dynamics is well-posed (no finite-time explosion of an eigenvalue), it appears to converge generically to a low-dimensional subspace, independently of the attention type.
For Softmax, $\revision{\ell^2}$ and Sinkhorn self-attention, this subspace is included in the boundary of the PSD cone, namely in the set of degenerate nonnegative matrices: this parallels the clustering phenomenon that occurs with a finite number of tokens.
We investigate this aspect in higher dimension in Paragraph \ref{par:histograms}.

Our figures also allow for an empirical comparison of the behaviors with the different types of self-attention.
In particular, we observe that Softmax and $\revision{\ell^2}$ self-attention induce very similar behaviors (see Appendix \ref{appsubsec:experiments}, Figure \ref{appfig:well_posed_dynamics}), except when the former blows up—in that case, $\revision{\ell^2}$ diverges but stays well-posed, according to Lemma \ref{lem:well_posed_L2}, which induces very different trajectories (see for instance Figure \ref{fig:explosion} (a) and (c)).

\begin{figure}[ht!]
    \centering
    \includegraphics[width=0.8\linewidth]{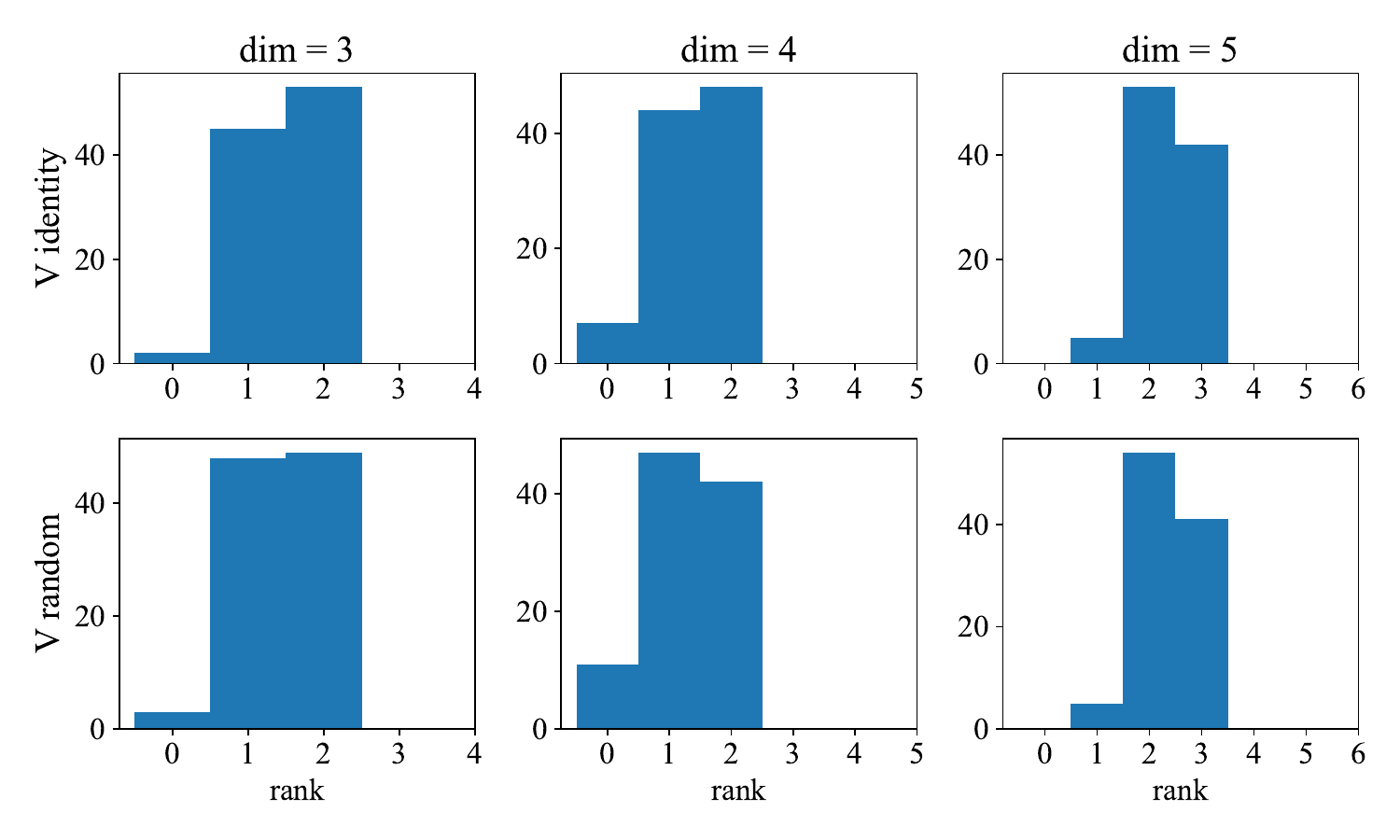}
    \caption{Histogram of the rank of limiting points of the covariance equation for Softmax self-attention, in dimensions 3, 4, and 5. The matrix $V$ has full rank ($V = I_d$ in the upper row and $V$ random and different for each point in the lower row) and the matrix $A$ has rank $\lfloor d / 2\rfloor$, is random negative semidefinite, and is different for each point. Limiting points have low rank (smaller than $\lceil d / 2\rceil$), which parallels the clustering phenomenon observed for discrete tokens.}
    \label{fig:histograms_trad}
\end{figure}

\

\paragraph{Clustering in higher dimension}
\label{par:histograms}
Proposition \ref{prop:trad_theoretical_result} shows that the limiting points of the covariance equation associated with Softmax self-attention have a low rank under some assumptions on the parameters.
To complement this analysis, we highlight numerically that this holds generally, as soon as the dynamics converge, and not only for Softmax self-attention but also for $\revision{\ell^2}$ self-attention.

Figure \ref{fig:histograms_trad} plots, in dimensions $d = 3, 4, 5$, the histogram of the rank of limiting points of the covariance equation for Softmax self-attention, with $Q \in \R^{\lfloor d / 2\rfloor\times d}$ random (different for each point), $K = -Q$ so that $A \coloneqq K^\top Q$ is symmetric negative and of rank $\lfloor d / 2\rfloor$, and $V = I_d$ (upper row) or $V$ random and different for each point (lower row).
The case $V = I_d$ is covered by Proposition \ref{prop:trad_theoretical_result}, which tells us that limiting points have a rank smaller than $\lceil d / 2\rceil$: our numerical results confirm this result.
Although no theoretical result covers the case $V$ random, Figure \ref{fig:histograms_trad} highlights that the bound $\lceil d / 2\rceil$ still holds in that case.
We display a similar figure for $\revision{\ell^2}$ self-attention in Appendix \ref{appsubsec:experiments}.

\section{Viewing the Dynamics as a Gradient Flow}
\label{sec:gradient_flow}

This section aims to connect the Transformer PDE with gradient flows in measure spaces.
Let us first stress that the derivations in this section are only formal, the goal being to gain some intuition on the gradient flow structure of the different Transformer evolutions considered so far, and to highlight, in particular, the structure of the Gaussian case. Making these calculations rigorous is beyond the scope of this paper.
We start with some background on gradient flows with respect to a geodesic distance, which generalize Wasserstein gradient flows.

\subsection{Limiting Minimizing Schemes for Geodesic Distances}
\label{subsec:general_gradient_flows}

We consider a distance $\dcal(\mu, \nu)$ on the space $\pcal(\R^d)$ of probability distributions on $\R^d$. For some function $\fcal\colon \pcal(\R^d)\to \R$, one can define minimizing evolutions on the metric space $\pcal(\R^d)$ starting from $\mu_0\in \pcal(\R^d)$, following \cite{ambrosio2008gradient}, by considering implicit stepping with step-size $\tau > 0$:
\begin{equation}
    \mu_{t+\tau} \in \underset{\mu\in \pcal(\R^d)}{\operatorname{argmin}} \left\{ \frac{1}{2\tau} \dcal(\mu_t, \mu)^2 + \fcal(\mu) \right\}. \label{eq:discr-stepping}
\end{equation}

Under suitable regularity properties on $\fcal$, it is possible to take the limit $\tau \to 0$ and consider the continuous-time evolution $t \mapsto \mu_t$, which solves a PDE.
The aim of this subsection is to formally derive this PDE, called the \emph{gradient flow} of $\fcal$ for the distance $\dcal$, and which takes the form of a continuity equation for some velocity field that depends on $\fcal$.
Note that if $\dcal$ is equal to the Wasserstein distance $W_2$, this is by now a textbook computation which corresponds to the JKO flow construction~\cite{jordan1998variational,ambrosio2008gradient,santambrogio2015optimal} of Wasserstein gradient flows.
We consider the more general case where $\dcal$ is a geodesic distance with the following dynamical formulation:
\begin{equation}
\label{eq:dynamical_formulation}
    \dcal(\mu, \nu)^2 = \inf_{\rho, v}\left\{ \int_0^1\!\!\!\int \mathcal{G}_{\rho_s}(v_s) \cdot v_s \dd \rho_s \dd s : \partial_s\rho + \mathrm{div}(\rho\, \mathcal{G}_\rho (v)) = 0, \mbox{ } \begin{array}{c}
        \rho_0 = \mu \\
        \rho_1 = \nu
    \end{array}  \right\},
\end{equation}
where $\mathcal{G}_{\rho}$ is a linear invertible $L^2_\rho$-self-adjoint operator on vector fields. The classical distance $\dcal = W_2$ corresponds to the choice $\mathcal{G}_\rho = \text{Id}$. We here follow a similar duality strategy to \cite{brenierEOT,CLSS10} to formally rewrite the distance as a saddle point problem for which optimality conditions are easier to obtain. We first observe that we can dilate the time variable $s$ and scale $v$ such that
\begin{equation}\label{eq:dynamical_formulation2}
    \frac1{\tau}\dcal(\mu, \nu)^2 = \inf_{\rho, v}\left\{ \int_0^\tau\!\!\!\int \mathcal{G}_{\rho_s}(v_s) \cdot v_s \dd \rho_s \dd s : \partial_s\rho + \mathrm{div}(\rho\, \mathcal{G}_\rho (v)) = 0, \!\begin{array}{c}
        \rho_0 = \mu \\
        \rho_\tau = \nu
    \end{array} \!\! \right\},
\end{equation}
Using a Lagrangian dual formulation of $\dcal$, we rewrite
\begin{align*}
\mu_{t + \tau} \in \underset{\mu}{\mathrm{argmin}}\ 2 \fcal(\mu) &+ \inf_{\rho, v} \sup_\psi \int_0^\tau\!\!\!\int \mathcal{G}_\rho(v) \cdot v \dd \rho \dd s \\ &- \int_0^\tau \!\!\!\int(\partial_s \psi + \mathcal{G}_\rho(v) \cdot \nabla_x \psi ) \dd \rho \dd s + \int \!\psi(\tau, \cdot) \dd \mu - \int \! \psi(0, \cdot) \dd \mu_t.
\end{align*}
\revision{Here, the test functions $\psi\in C_0^\infty([0,\tau]\times \R^d)$ play the role of Lagrange multipliers to rewrite the distance as an inf-sup problem without constraints. This approach is classical going back to the case of the classical Wasserstein distance $W_2$, see \cite{benamou2000computational,brenierEOT}.}
The optimality condition on $\mu_{t+\tau}$ is
\begin{equation}
\label{eq:optimality_condition}
    2 \delta \fcal(\mu_{t+\tau}) + \psi(\tau, \cdot) = 0,
\end{equation}
where $\delta \fcal$ is the first variation (Fréchet derivative) of $\fcal$, i.e.,
\begin{equation*}
    \fcal(\mu + h \nu) = \fcal(\mu) + h \int \delta \fcal(\mu) \, \mathrm{d}\mu + o(h).
\end{equation*}
Equation (\ref{eq:optimality_condition}) formally gives
\begin{equation}\label{eq:oc_1} \frac12 \nabla_x\psi = - \nabla_x \delta \fcal(\mu), \end{equation}
when $\tau \to 0^+$.
Moreover, we find the following optimality condition for $v$ given by
\begin{equation} \label{eq:oc_2} 
v = \frac12 \nabla_x\psi,
\end{equation}
by taking Fréchet derivatives with respect to $v$ in \eqref{eq:dynamical_formulation2} and using that $\mathcal{G}_\rho$ is a linear invertible $L^2_\rho$-self-adjoint operator on vector fields. Combining Equations (\ref{eq:oc_1}) and (\ref{eq:oc_2}), one can rewrite $\partial_t \rho + \mathrm{div}(\rho \, \mathcal{G}_\rho (v)) = 0$ as the following non-linear advection equation
\begin{equation}\label{eq:pde-gradflow}
    \partial_t \rho + \mathrm{div}(\rho\, \Gamma_\rho) = 0. 
\end{equation}
with $\Gamma_\rho \coloneqq - \mathcal{G}_\rho (\nabla_x[\delta \fcal(\mu)])$. In this sense, we say that the PDE (\ref{eq:pde-gradflow}) is the gradient flow of $\fcal$ for the distance $\dcal$. Notice that for the classical case with $\dcal = W_2$ so that $\mathcal{G}_\mu = \text{Id}$, then $\Gamma_\rho$ is the standard Wasserstein gradient.

\begin{remark}
Classical free energy functionals are of the form
\begin{equation} \label{F}
{\cal F}(\rho) = \int U(\rho)\,\dd x + \int
V(x)\,\rho(x)\dd x + \frac12 \int 
W(x-y)\,\rho(x)\,\rho(y)\dd x\dd y.
\end{equation}
where $U\colon\R^+\to \R$ is a density of
internal energy, $V\colon\R^d\to\R$ is a confinement potential and
$W\colon\R^d\to\R$ is an interaction potential, see for instance \cite{CMV03,V03}. The corresponding variation is given by
$
\delta \fcal(\rho)=U' \left ( \rho \right ) + V + W\ast \rho.
$
Without interaction potential $W=0$, this general family of PDEs, with $\dcal = W_2$ so that $\mathcal{G}_\mu = \text{Id}$, contains well-known models in mathematical physics such as the heat equation, $U(s)=s\log s$ and $V=0$, the porous-medium
and fast-diffusion equations \cite{Vazquez}, $U(s)={s^m}/{(m-1)}$, $m>0$  and $V=0$ and their Fokker-Planck counterparts, 
$U(s)=s\log s$ and $V(x)={|x|^2}/{2}$, and $U(s)={s^m}/{(m-1)}$ and $V(x)={|x|^2}/{2}$ respectively, see
\cite{Otto01} for instance. With nontrivial interaction potentials, it includes many important PDEs in mathematical biology and mathematical physics, such as the Keller-Segel model, see \cite{CCY19} and the references therein, or nonlocal McKean-Vlasov equations, see \cite{CGW23} and the references therein.
\end{remark}

\subsection{Restriction to the Subspace of Gaussians}
\label{subsec:gaussian_restriction}
In Section \ref{sec:gaussian}, we have seen that the Transformer PDE with a Gaussian measure as initial data stays in the space of Gaussians over time, for several variants of self-attention.
For Transformer PDEs that have a gradient flow structure as introduced in Subsection \ref{subsec:general_gradient_flows}, the evolution of the Gaussian case can be inferred from the gradient flow structure, and follows a Riemannian gradient flow.
This subsection details that connection, which allows us to check the Gaussian evolutions obtained in Section \ref{sec:gaussian}.

Let $\dcal$ be a geodesic distance.
We assume that $\dcal$ is translation invariant, in the sense that:
\begin{equation*}
    \dcal(\mu, \nu)^2 = \dcal(\mu_0, \nu_0)^2 + \lvert m(\mu) - m(\nu)\rvert^2,
\end{equation*}
where $\lvert \cdot \rvert$ is the Euclidean norm, $m(\mu) \coloneqq \int x \mathrm{d}\mu(x)$ is the mean of $\mu$, and $\mu_0 = T_\# \mu$ with $T(x) = x - m(\mu)$ (so that $\mu_0$ and $\nu_0$ have zero mean). This is the case for the Wasserstein distance and the twisted distance~(\ref{eq:twisted-distance}) considered below.
The metric $\dcal$ descends to a finite-dimensional metric $D$ on the cone $\mathcal{S}_d^{++}$ of covariances:
\begin{equation*}
    D(\Sigma, \Sigma') \coloneqq \dcal(\mathcal{N}(0, \Sigma), \mathcal{N}(0, \Sigma')),
\end{equation*}
where $\mathcal{N}(\alpha, \Sigma)$ is the Gaussian measure with covariance $\Sigma$ and mean $\alpha$. If the initial data $\mu_{t=0}$ is Gaussian and the vector field $\Gamma_{\mathcal{N}(\alpha, \Sigma)}$ is affine, solutions of Equation (\ref{eq:pde-gradflow}) are Gaussian for all time, as seen in Section \ref{sec:gaussian}: we can write $\mu_t \eqqcolon \mathcal{N}(\alpha(t), \Sigma(t))$, and $t \mapsto (\alpha(t), \Sigma(t))$ follows a Riemannian gradient flow of the finite-dimensional function
\begin{equation*}
    F(\alpha, \Sigma) \coloneqq \fcal(\mathcal{N}(\alpha, \Sigma)),
\end{equation*}
where the Riemannian structure is induced by $D$ as follows. 

Similarly to Equation (\ref{eq:discr-stepping}), the Riemannian flow induced by the gradient flow on the space of Gaussians can be defined via an implicit stepping:
\begin{equation*}
    (\alpha(t+\tau), \Sigma(t+\tau)) \in \underset{\alpha, \Sigma}{\operatorname{argmin}} \, D(\Sigma, \Sigma(t))^2 + \lvert \alpha - \alpha(t)\rvert^2 + 2\tau F(\alpha, \Sigma),
\end{equation*}
where the minimization is restricted to positive semidefinite matrices.
Taking the limit $\tau \to 0$ yields a continuous-time trajectory $t \mapsto (\alpha(t), \Sigma(t))$. This flow satisfies:
\begin{align}
    \label{eq:riemannian-cov}
    \frac{\mathrm{d}\alpha}{\mathrm{d}t} &= -\nabla_\alpha F(\alpha(t), \Sigma(t)), \\
    \frac{\mathrm{d}\Sigma}{\mathrm{d}t} &= -M_\Sigma^{-1} \nabla_\Sigma F(\alpha(t), \Sigma(t)),\nonumber
\end{align}
where the Riemannian metric $M_\Sigma : \mathbb{R}^{d \times d} \to \mathbb{R}^{d \times d}$ is defined as follows:
\begin{equation}
    \label{eq:riemannian_metric}
    D(\Sigma, \Sigma')^2 = \langle M_\Sigma(\Sigma - \Sigma'), \Sigma - \Sigma' \rangle_{\mathbb{R}^{d \times d}} + o(\|\Sigma - \Sigma'\|_F^2).
\end{equation}
The following two subsections detail the gradient flow structure of the Transformer PDE respectively for Softmax and Sinkhorn self-attention, in the general case and the Gaussian case, building on what has been introduced above.

\subsection{Sinkformer: Wasserstein Flow}
\label{subsec:BW_flow}

As shown in \cite{sander2022sinkformers}, for the specific case of Sinkhorn attention layers, and under the assumption that $A = A^\top = -V$ (with $A\coloneqq K^\top Q$), the Sinkformer PDE is a Wasserstein gradient flow, i.e., it satisfies Equation (\ref{eq:pde-gradflow}) for $\dcal = W_2$ (i.e., $\mathcal{G}_\mu = \text{Id}$), and for the functional:
\begin{equation}
\label{eq:sink_energy_functional}
    \fcal_\varepsilon(\mu) \coloneqq -\frac12 \int \kappa_{\mu, \varepsilon}^\infty\log \left (\frac{\kappa_{\mu, \varepsilon}^\infty}{\kappa_{\mu, \varepsilon}^0}\right )\dd (\mu\otimes \mu) + \frac{1}{4\varepsilon}\int (\modu{Qx}^2 + \modu{Kx}^2)\dd \mu(x).
\end{equation}

\begin{remark}
    The second term in Equation (\ref{eq:sink_energy_functional}) does not appear in \cite{sander2022sinkformers}, because the authors define $\kappa^0(x,y)\coloneqq \exp(Qx\cdot Ky/\varepsilon)$, while we choose $\kappa^0(x,y)\coloneqq \exp(-\modu{Qx - Ky}^2 / 2\varepsilon)$.
    These two choices lead to the same $\kappa^\infty$.
\end{remark}

As seen in Section~\ref{subsec:sink_self_attention}, in this case, for a Gaussian $\mu$, the vector field $\Gamma_\mu$ is an affine function. Consequently, one can also consider this evolution as a Riemannian gradient flow over the mean and covariance space, according to Subsection \ref{subsec:gaussian_restriction}. The corresponding Riemannian metric is the so-called Bures-Wasserstein metric, whose inverse reads \cite{malago2018wasserstein}:
\begin{equation*}
    M_\Sigma^{-1} : A \in \mathbb{R}^{d \times d} \mapsto A \Sigma + \Sigma A^\top.
\end{equation*}

When restricted to the space of Gaussian distributions, we show in Appendix \ref{appsubsec:BW_flow} that the Sinkformer evolution (\ref{eq:transformer_pde_trad}) corresponds to the Riemannian flow ODE (\ref{eq:riemannian-cov}) for the function:
\begin{equation}
\label{GaussFunc-sink}
    F_\varepsilon(\alpha, \Sigma) =  \frac{1}{4\varepsilon} \left (- \mathfrak{B}_\varepsilon^2(\Sigma, \Sigma) + \tr(Q\Sigma Q^\top) + \tr(K\Sigma K^\top) + \alpha^\top (Q^\top Q +K^\top K)\alpha \right ),
\end{equation}
where $\mathfrak{B}_\varepsilon^2(\Sigma_1,\Sigma_2)$ is the entropy-regularized Bures distance, defined as
\begin{equation} 
\label{eq:entropy_regularized_bures}
\mathfrak{B}_\varepsilon^2(\Sigma_1, \Sigma_2)\coloneqq 2\varepsilon OT_\varepsilon(\ncal(\alpha_1, \Sigma_1), \ncal(\alpha_2, \Sigma_2)) - \modu{\alpha_1 - \alpha_2}^2,
\end{equation}
where $OT_\varepsilon$ is defined in Equation (\ref{eq:eot}).
Note that these Bures flows have several favorable properties. In particular, the low-rank manifold is closed, meaning that the rank cannot increase during evolution. However, it is possible for the rank to decrease, and even if the initial covariance is full-rank, it typically becomes rank-deficient at the final time, as illustrated numerically in Section~\ref{subsec:experiments}.

\subsection{Softmax Transformer: A Twisted Wasserstein Flow}
\label{subsec:twisted_flow}

For a finite number of tokens, the Transformer evolution can be recast as a gradient flow for a particular non-Euclidean metric \cite{geshkovski2023emergence}.
We generalize this construction to probability measures by defining a twisted Wasserstein distance denoted $d_{A, V}$, associated with:
\begin{equation}
 \label{eq:twisted-distance}
 \mathcal{G}_\mu(v) \colon x\mapsto \frac{Bv(x)}{\int e^{Ax\cdot y}\dd \mu(y)},
\end{equation}
with the notation of Equation (\ref{eq:dynamical_formulation}), and where $B\coloneqq - V A^{-\top}$.
The twisted Wasserstein distance
$$\revision{d_{A, V}(\mu, \nu)^2 = \inf_{\rho, v}\left\{ \int_0^1\!\!\!\int \mathcal{G}_{\rho_s}(v_s) \cdot v_s \dd \rho_s \dd s : \partial_s\rho + \mathrm{div}(\rho\, \mathcal{G}_\rho (v)) = 0, \mbox{ } \begin{array}{c}
        \rho_0 = \mu \\
        \rho_1 = \nu
    \end{array}  \right\}}$$
therefore mimics the dynamical formulation of the Wasserstein distance, but changing the mobility, in the spirit of \cite{CLSS10,li2021hessian}.

Using this metric, we obtain with the computation of Subsection \ref{subsec:general_gradient_flows} that when $B$ is symmetric and positive definite, the Transformer PDE becomes the gradient flow PDE (\ref{eq:pde-gradflow}) for the quadratic interaction functional:
\begin{equation*}
    \fcal(\mu) \coloneqq \frac12 \int e^{Ax\cdot y} \dd \mu(x) \dd \mu(y).
\end{equation*}

When restricted to the space of Gaussian distributions, the Transformer evolution (\ref{eq:transformer_pde_trad}) corresponds to the Riemannian flow ODE (\ref{eq:riemannian-cov}) for the function:
\begin{equation*}
    F(\alpha, \Sigma) \coloneqq  \frac{e^{ \frac12 \alpha^\top ((A + \Sigma^{-1})^\top (\Sigma^{-1} - A \Sigma A^\top)^{-1}(A + \Sigma^{-1}) - \Sigma^{-1})\alpha }}{2\lvert \det(I_d - A \Sigma A^\top \Sigma)\rvert^{1/2}}
\end{equation*}
and the Riemannian metric induced by $d_{A, V}$ on the space of Gaussians via Equation (\ref{eq:riemannian_metric}).
The computation is in Appendix \ref{appparagraph:gaussian_functional_trad}. 

An important open question is to know whether the energy functional $\fcal$ is geodesically convex for the twisted Wasserstein distance.
We provide a first answer by showing that this is not the case under natural sign and commutation assumptions on the parameters.
In future work, it could be interesting to investigate the geodesic convexity of $\fcal$ restricted to smaller classes of measures, or subject to some assumptions on the parameters.

We have the following characterization of the geodesics of $d_{A, V}$ (see Appendix \ref{appparagraph:geodesics_of_twisted_distance}).

\begin{lemma}
    Let $A, V$ be such that \revision{$A$ is symmetric and $B\coloneqq - VA^{-1}$ is symmetric positive definite}.
    \revision{Denote $G(y,z)\coloneqq e^{Ay\cdot z}$, and $G*\mu(y)\coloneqq \int_{\R^d} G(y,z)\dd\mu(z)$ for any probability measure $\mu$.}
    The geodesics $\rho$ of $d_{A, V}$ are characterized by
    \revision{\begin{equation*}
            \partial_s \rho + \mathrm{div} \left ( \frac{B\nabla_x \psi}{2G*\rho}\rho \right ) = 0
            \end{equation*}
            for $\psi$ solving
            \begin{equation*}
            \partial_s \psi + \frac14 \frac{\nabla_x\psi\cdot B\nabla_x\psi}{G*\rho} - \frac14 \int_{\R^d} \frac{\nabla_y\psi\cdot B\nabla_y\psi}{(G*\rho(y))^2}G(y,x)\rho_s(y)\dd y = 0.
    \end{equation*}}
\end{lemma}

Building on this computation, we prove the following negative result (see Appendix \ref{appsubsec:attention_distance}).

\begin{prop}
    \label{prop:non_geo_convexity}
    Let $A, V \in \R^{d\times d}$ such that \revision{$A$ is symmetric} and $B\coloneqq - VA^{-\top}$ is symmetric positive definite. Assume that \revision{$A$ commutes with $B$ and that $V$ has at least one positive eigenvalue.}
    Then $\fcal$ is not geodesically convex for the distance $d_{A, V}$.
\end{prop}

\revision{The absence of geodesic convexity rules out the usual methods for analyzing gradient flows and their long-term behavior. It often signals complex, long-term dynamics—such as metastability or a variety of equilibrium states—similar to what is observed in aggregation equations \cite{MR3143991,MR3067832}. In fact, the non-geodesic convexity is proved nearby Dirac concentrations at single points, suggesting that these points may be saddle points or unstable equilibria.}
\revision{\begin{remark}
    Proposition \ref{prop:non_geo_convexity} and Proposition \ref{prop:trad_theoretical_result} part 2, which can both be seen as negative results, and whose assumptions have a nonempty overlap (for instance $V = I_d, A = -I_d$), capture complementary aspects of the geodesic geometry of the functional: the finite-time blow-up reflects a lack of upper control on the second variation along Wasserstein geodesics (a smoothness issue), whereas the failure of geodesic convexity corresponds to the absence of a lower bound.
\end{remark}}

\section{Conclusion}

This work studies the Transformer PDE, which models the evolution of data that goes through the layers of a deep Transformer model.
For several variants of self-attention, we show that the Transformer PDE is well-posed when the initial data is compactly supported, and we derive a stability estimate with respect to the initial condition. Our framework includes in particular masked self-attention.
We also consider the case of a Gaussian initial condition, which has the useful property of staying Gaussian across the dynamics for several attention variants.
Building on this remark, we show both theoretically and numerically that when the Gaussian evolution converges, the covariance of the limiting Gaussian measure is rank-deficient, which parallels the clustering phenomenon observed with discrete tokens.
Finally, we draw a connection between our framework and gradient flows, by introducing formally a distance on probability measures that equips the Transformer PDE associated with Softmax self-attention with a gradient flow structure.

\section*{Acknowledgments}

\revision{We thank the reviewers for their accurate reading and their excellent suggestions.}
The research of JAC was supported by the Advanced Grant Nonlocal\--CPD (Non\-local PDEs for Complex Particle Dynamics: Phase Transitions, Patterns and Synchronization) of the European Research Council Executive Agency (ERC) under the European Union’s Horizon 2020 research and innovation programme (grant agreement No. 883363).
JAC was also partially supported by EPSRC grant number EP/V051121/1.
JAC was also partially supported by the “Maria de Maeztu” Excellence Unit IMAG, funded by MCIN/AEI/10.13039/501100011033/, reference CEX2020-001105-M.
The work of G. Peyré was supported by the French government under the management of Agence Nationale de la Recherche as part of the “France 2030” program, reference ANR-23-IACL-0008 (PRAIRIE-PSAI). 
The work of G. Peyré and V. Castin was supported by the ERC project WOLF.

\appendix
\section{Background on Entropic Optimal Transport}
\label{appsec:eot_background}

\subsection{Dual Formulation of EOT}
\label{appsubsec:dual_EOT}

Let $\mu$ and $\nu$ be two compactly supported or Gaussian probability measures on $\R^d$.
Recall the entropic optimal transport problem
\begin{equation} 
    \label{eq:entropic_ot}
    OT_\varepsilon(\mu, \nu) \coloneqq \min_{\pi \in \Pi(\mu, \nu)} \int c_\varepsilon(x, y) \dd \pi(x, y) + \KL(\pi \| \mu \otimes \nu),
\end{equation}
introduced in Section \ref{sec:att_variants}, where $c_\varepsilon(x, y)\coloneqq \frac{1}{2\varepsilon}\modu{Qx - Ky}^2$.
It is well-known \cite{janati2020entropic} that problem (\ref{eq:entropic_ot}) admits the following dual formulation
\begin{equation*}
    OT_\varepsilon(\mu, \nu) = \max_{\substack{f \in \lcal^1(\mu) \\ g \in \lcal^1(\nu)}} \int_{\R^d} f \dd\mu + \int_{\R^d} g \dd \nu + 1 - \int_{\R^d \times \R^d} e^{f(x) + g(y) - c_\varepsilon(x, y)} \dd \mu(x) \dd \nu(y),
\end{equation*}
and that a pair of dual potentials $(f, g)$ is optimal if and only if it satisfies the following optimality conditions respectively $\nu$- and $\mu$-almost everywhere:
\begin{equation}
    \label{eq:schrodinger_system}
    \begin{cases}
        f(x) &= - \log \int e^{g(y) -c_\varepsilon(x,y)} \dd \nu(y) \\
        g(y) &= - \log \int e^{f(x) -c_\varepsilon(x,y)} \dd \mu(x).
    \end{cases}
\end{equation}
There always exist solutions $f$ and $g$ that satisfy Equation (\ref{eq:schrodinger_system}) for \emph{every} $(x, y) \in \xcal \times \ycal$ \cite{nutz2021introduction}.
Let us focus on such pairs.
Then, $(f, g)$ is unique modulo the equivalence relation
$$(f_1, g_1) \sim (f_2, g_2) \Leftrightarrow \exists \eta \mbox{ s.t. } f_1 = f_2 + \eta \mbox{ and } g_1 = g_2 - \eta.$$
Let then $(f, g)$ be any couple in the equivalence class.
We call $f$ and $g$ Schrödinger potentials or dual potentials.
They inherit the regularity of $c$, and satisfy the following useful relation:
\begin{equation*}
    \kappa_{\mu, \varepsilon}^\infty(x, y) = e^{f(x) + g(y) - c_\varepsilon(x, y)}.
\end{equation*}
We can then rewrite $OT_\varepsilon(\mu, \nu)$ as follows:
\begin{equation}
    \label{eq:OT_kappa_formulation}
    OT_\varepsilon (\mu, \nu) = \frac{1}{2\varepsilon} \int \modu{Qx - Ky}^2 \kappa_{\mu, \varepsilon}^\infty(x, y)\dd \mu(x)\dd \nu(y) + \int \log(\kappa_{\mu, \varepsilon}^\infty)\kappa_{\mu, \varepsilon}^\infty \dd \mu(x) \dd \nu(y).
\end{equation}
Assume now that $\mu$ and $\nu$ are compactly supported, with a support included in the ball $B_R$.
The dual potentials are both Lipschitz continuous.

\begin{lemma}
    \label{lem:schrödinger_estimates}
    Assume that $\supp \mu$, $\supp \nu \subset B_R$.
    Let $c\colon (x,y)\in (\R^d)^2 \mapsto \frac{1}{2\varepsilon}\modu{Qx - Ky}^2$ and $(f, g)$ a couple of dual potentials satisfying Equation (\ref{eq:schrodinger_system}) everywhere.
    Then
    \begin{equation} \label{appeq:bound_f_g}\sup_{x,y\in B_R} f(x) + g(y) \le \frac1\varepsilon (\norm{Q}_2 + \norm{K}_2)^2 R^2.\end{equation}
    Moreover, $f$ and $g$ are both Lipschitz continuous, with a Lipschitz constant bounded by $\frac1\varepsilon (\norm{Q}_2 + \norm{K}_2)^{3/2} R$.
\end{lemma}

\begin{proof}
    Equation (\ref{appeq:bound_f_g}) derives from \cite{nutz2021introduction}, Lemma 4.9, noticing that $c_\varepsilon(x, y)\le \frac1\varepsilon (\norm{Q}_2 + \norm{K}_2)^2 R^2$ for $x, y\in B_R$.
    Then, according to \cite{nutz2021introduction} Lemma 4.11, the dual potentials are Lipschitz continuous with the same Lipschitz constant as $c_\varepsilon$.
    It is then straightforward to check that $c_\varepsilon$ is $\frac1\varepsilon (\norm{Q}_2 + \norm{K}_2)^{3/2} R$-Lipschitz continuous in both variables.
\end{proof}

Finally, we have the following stability result.

\begin{lemma}[\cite{carlier2022lipschitz}]
    \label{lem:carlier}
    For any compactly supported distributions $\mu, \nu$ on $\R^d$, consider the entropic optimal transport problems $OT_\varepsilon(\mu, \mu)$ and $OT_\varepsilon(\nu, \nu)$ for the cost $c(x, y) \coloneqq \frac{1}{2\varepsilon} \modu{Qx - Ky}^2$, and denote $(f^\mu, g^\mu)$ and $(f^\nu, g^\nu)$ associated Schrödinger potentials.
    Then, there exists a function $C_\mathrm{stab}(R)>0$ depending on $R, Q, K, \varepsilon$ such that for all compactly supported probability measures $\mu$ and $\nu$ on $\R^d$, we have
    $$\inf_{\eta \in \R} \norm{f^\mu - f^\nu - \eta}_\infty + \norm{g^\mu - g^\nu + \eta}_\infty \le C_\mathrm{stab}(R) W_2(\mu, \nu).$$
\end{lemma}

\subsection{EOT Between Gaussians}
We have the following generalization of Theorem 1 in \cite{janati2020entropic}, where we take the cost $\frac{1}{2\varepsilon}\modu{Qx - Ky}^2$ instead of the classical quadratic cost.
Note that a similar result is stated in \cite{bojilov2016matching}, with a generalization to the case $A$ non-invertible.

\begin{theorem}
    \label{appthm:extension_janati}
    Let $\mu = \ncal(\alpha, \Sigma)$ and $\nu = \ncal(\beta, \Omega)$ be two Gaussian measures.
    Let $Q$ and $K$ be two matrices in $\R^{d\times d}$, and let $\varepsilon >0$.
    Denote $A = K^\top Q$ and assume that $A$ is invertible.
    Denote $\pi^*$ the minimizer of the following entropy-regularized optimal transport problem:
    \beq 
    \label{eq:entropic_ot_problem}
    \min_{\pi \in \Pi(\mu, \nu)} \int \frac{1}{2\varepsilon}\modu{Q x - K y}^2 \dd \pi(x, y) + \KL(\pi \| \mu \otimes \nu),
    \eeq
    where $\Pi(\mu, \nu)$ is the set of couplings between $\mu$ and $\nu$, and
    $$\KL(\pi \| \mu \otimes \nu) = \int \log\prt{\frac{\dd \pi}{\dd (\mu\otimes \nu)}} \dd \pi$$
    is the Kullback-Leibler divergence between $\pi$ and $\mu\otimes \nu$, set equal to $+\infty$ if $\pi$ is not absolutely continuous with respect to $\mu \otimes \nu$.
    Then, the optimal coupling $\pi^*$ is a Gaussian measure, given by
    $$\pi^* = \ncal\prt{\begin{pmatrix}
        \alpha \\ \beta
    \end{pmatrix}, \begin{pmatrix}
        \Sigma & A^{-1} C^\top \\ C A^{-\top} & \Omega
    \end{pmatrix}},$$
    with
    $$C \coloneqq \Omega^{1/2} \prt{\Omega^{1/2} A \Sigma A^\top \Omega^{1/2} + \frac{\varepsilon^2}{4}I_d}^{1/2}\Omega^{-1/2} - \frac{\varepsilon}{2} I_d.$$
\end{theorem}

\begin{proof}[Proof of Theorem \ref{appthm:extension_janati}]
    We follow the steps of the proof of \cite[Theorem 1]{janati2020entropic}.
    First of all, we can assume that $\mu$ and $\nu$ are centered (\cite[Lemma 1]{janati2020entropic}).
    Then, using the dual formulation of Problem (\ref{eq:entropic_ot_problem}) introduced in Section \ref{appsubsec:dual_EOT}, we have that
    \beq 
    \label{eq:optimal_potentials}
    \frac{\dd \pi^*}{\dd \mu\otimes \nu} (x,y) = e^{f(x) + g(y) - \frac{\modu{Q x - Ky}^2}{2\varepsilon}},
    \eeq
    where $f,g\colon \R^d \to \R$ are optimal dual potentials, that can be obtained as the limit of the following recursion, starting from a potential $f_0$:
    \beq
        \label{eq:recursion_potentials}
        \begin{gathered}
            g_{n+1}(y) = - \log \int e^{f_n(x) - \frac{\modu{Qx - Ky}^2}{2\varepsilon}} \dd \mu(x) \\
            f_{n+1}(x) = - \log \int e^{g_{n+1}(y) - \frac{\modu{Qx - Ky}^2}{2\varepsilon}} \dd \nu(y)
        \end{gathered}
    \eeq
    for $n\ge 0$.
    Let us choose $f_0$ to be a quadratic form, and show that the iterates of the recursion (\ref{eq:recursion_potentials}) stay in the space of quadratic forms.
    For any symmetric matrix $X \in \R^{d\times d}$, denote
    $$\qcal(X)\colon x\in \R^d \mapsto -\frac{1}{2}x^\top X x.$$
    We have the following Lemma.

    \begin{lemma}
        \label{lem:propagation_quadratic_forms}
        Let $\mu = \ncal(0, \Sigma)$ and $\nu = \ncal(0, \Omega)$, and let $X$ be a symmetric $d\times d$ matrix and $m\in \R$ a constant.
        \begin{enumerate}[label=(\roman*)]
            \item The expression
            $$T_{\mu, Q, K}(\qcal(X) + m)\coloneqq-\log \int e^{-\modu{Qx - Ky}^2/2\varepsilon + \qcal(X)(x) + m} \dd \mu(x)$$
            for $x\in \R^d$ is well-defined if and only if
            $$X' \coloneqq Q^\top Q +\varepsilon \Sigma^{-1} + \varepsilon X \succ 0$$
            and in that case, $T_{\mu, Q, K}(\qcal(X) + m)$ is a quadratic form, equal to $\qcal(Y)$ up to a constant term, where
            $$Y \coloneqq \frac{1}{\varepsilon}(A X'^{-1} A^\top - K^\top K).$$
            \item The expression
            $$T_{\nu, K, Q}(\qcal(X) + m)\coloneqq-\log \int e^{-\modu{Qx - Ky}^2/2\varepsilon + \qcal(X)(y) + m} \dd \nu(y)$$
            for $x\in \R^d$ is well-defined if and only if
            $$X' \coloneqq K^\top K +\varepsilon \Omega^{-1} + \varepsilon X \succ 0$$
            and in that case, it is a quadratic form, equal to $\qcal(Y)$ up to a constant term, where
            $$Y \coloneqq \frac{1}{\varepsilon}(A^\top X'^{-1} A - Q^\top Q).$$
        \end{enumerate}
    \end{lemma}

    With the notation of Lemma \ref{lem:propagation_quadratic_forms}, we can rewrite the recursion (\ref{eq:recursion_potentials}) as
    \begin{gather*}
        g_{n+1}= T_{\mu, Q, K}\prt{f_{n}} \\
        f_{n+1} = T_{\nu, K, Q}\prt{g_{n+1}}.
    \end{gather*}
    Let us initialize it with
    $$f_0 = \qcal(0).$$
    According to Lemma \ref{lem:propagation_quadratic_forms}, for all $n\ge 1$, we can write $f_n$ and $g_n$ as quadratic forms $\qcal(U_n)$ and $\qcal(V_n)$ up to constant terms.
    Assume that $f_0 =\qcal(0)$, and define $f_n$ and $g_n$ with Equation (\ref{eq:recursion_potentials}).
    Then, we can write
    $$f_n = \qcal(U_n) + \mathrm{cst}~~~~~\mbox{and}~~~~~g_n = \qcal(V_n) + \mathrm{cst}$$
    where
    \begin{gather*}
        V_{n+1} = \frac{1}{\varepsilon}\prt{A(Q^\top Q + \varepsilon \Sigma^{-1} + \varepsilon U_n)^{-1}A^\top - K^\top K} \\
        U_{n+1} = \frac{1}{\varepsilon}\prt{A^\top(K^\top K + \varepsilon \Omega^{-1} + \varepsilon V_{n+1})^{-1}A - Q^\top Q}.
    \end{gather*}
    In particular, for all $n\ge 0$, the matrices $Q^\top Q + \varepsilon \Sigma^{-1} + \varepsilon U_n$ and $K^\top K + \varepsilon \Omega^{-1} + \varepsilon V_{n+1}$ are positive definite.
    Indeed, denote $F_0 \coloneqq Q^\top Q + \varepsilon \Sigma^{-1}$ and $G_0 \coloneqq K^\top K + \varepsilon \Omega^{-1}$, and
    \beq
        \label{eq:def_F_n}
        \begin{gathered}
            F_n \coloneqq Q^\top Q + \varepsilon \Sigma^{-1} + \varepsilon U_n, \\
            G_n \coloneqq K^\top K +\varepsilon \Omega^{-1} + \varepsilon V_n
        \end{gathered}
    \eeq
    for all $n\ge 1$.
    We have the following recursion:
    \beq 
        \label{eq:recursion_F_n}
        \begin{gathered}
            F_{n+1} = \varepsilon \Sigma^{-1} + A^\top G_n^{-1} A, \\
            G_{n+1} = \varepsilon \Omega^{-1} + A F_n^{-1} A^\top,
        \end{gathered}
    \eeq
    which proves that $F_n$ and $G_n$ stay positive definite along the iterations.

    \begin{lemma}
        The sequences of matrices $(F_n)_{n\ge 0}$ and $(G_n)_{n\ge 0}$ defined by $F_0 \coloneqq Q^\top Q + \varepsilon \Sigma^{-1}$ and $G_0 \coloneqq K^\top K + \varepsilon \Omega^{-1}$, and
        \begin{gather*}
            F_{n+1} \coloneqq \varepsilon \Sigma^{-1} + A^\top G_n^{-1} A, \\
            G_{n+1} \coloneqq \varepsilon \Omega^{-1} + A F_n^{-1} A^\top
        \end{gather*}
        for all $n\ge 0$ converge towards positive definite matrices $F$ and $G$.
    \end{lemma}

    \begin{proof}
        Denoting
        $$\phi(M) = \varepsilon \Sigma^{-1} + A^\top (\varepsilon \Omega^{-1} + A M^{-1} A^\top)^{-1} A,$$
        for any $M\succ 0$, we have $F_{n+1} = \phi(F_n)$ for all $n\ge 0$.
        A similar computation as in the proof of Proposition 2 \cite{janati2020entropic} shows that the operator norm of the differential of $\phi$ at $M$ is equal to
        $$\norm{D_M\phi}_2 = \norm{A^\top (\varepsilon \Omega^{-1} + A M^{-1}A^\top )^{-1} A M^{-1}}_2^2.$$
        Let us assume for now that $A$ is invertible.
        Then, it holds
        $$A^\top (\varepsilon \Omega^{-1} + A M^{-1}A^\top )^{-1} A M^{-1} = (I_d + \varepsilon M A^{-1}\Omega^{-1}A^{-\top})^{-1}.$$
        Denoting $\lambda_1(X)\ge \dots \ge \lambda_d(X)$ the eigenvalues of any matrix $X$, we have
        \begin{align*}
            \norm{A^\top (\varepsilon \Omega^{-1} + A M^{-1}A^\top )^{-1} A M^{-1}}_2 &= \frac{1}{\lambda_d(I_d + \varepsilon M A^{-1}\Omega^{-1}A^{-\top})} \\
            &\le \frac{1}{1 + \varepsilon \lambda_d(M) \lambda_d(A^{-1})\lambda_d(\Omega^{-1})\lambda_d(A^{-\top})}.
        \end{align*}
        Applying Weyl's inequality to the decomposition
        $$F_{n+1} = \varepsilon \Sigma^{-1} + A^\top (\varepsilon \Omega^{-1} + A F_n^{-1} A^\top)^{-1} A$$
        yields, noticing that the matrix $A^\top (\varepsilon \Omega^{-1} + A F_n^{-1} A^\top)^{-1} A$ is positive:
        $$\lambda_d(F_{n+1}) \ge \frac{\varepsilon}{\lambda_1(\Sigma)},$$
        which is also true for $\lambda_d(F_0)$.
        Moreover, for all $M\succ 0$ such that 
        $$\lambda_d(M) \ge \frac{\varepsilon}{\lambda_1(\Sigma)},$$
        we have
        $$\norm{D_M \phi}_2 \le \frac{1}{1 + \frac{\varepsilon^2}{\lambda_1(\Sigma)\lambda_1(A)^2\lambda_1(\Omega)}} < 1. $$
        An approximation argument shows that this inequality stays true even for non-invertible matrices $A$.
        Therefore, we have bounded the operator norm of the differential of $\phi$ uniformly away from 1 along the trajectory of $(F_n)$, which shows that $F_n$ converges.
        The same method applies to prove convergence of $(G_n)$.

        As $F_n$ and $G_n$ are positive matrices for all $n$, the limits $F$ and $G$ are nonnegative matrices.
        Moreover, taking the limit $n\to \infty$ in Equation (\ref{eq:recursion_F_n}) gives
        \begin{equation}
            \label{eq:link_F_G}
            \begin{gathered}
                F = \varepsilon \Sigma^{-1} + A^\top G^{-1} A, \\
                G = \varepsilon \Omega^{-1} + A F^{-1} A^\top,
            \end{gathered}
        \end{equation}
        which shows that $F$ and $G$ are positive matrices.
    \end{proof}
    By taking the limit $n\to \infty$ in Equations (\ref{eq:def_F_n}) and (\ref{eq:recursion_F_n}), we obtain the following relations:
    \begin{equation}
        \label{eq:link_F_U}
        \begin{gathered}
        F = Q^\top Q + \varepsilon \Sigma^{-1} + \varepsilon U, \\
        G = K^\top K +\varepsilon \Omega^{-1} + \varepsilon V,
        \end{gathered}
    \end{equation}
    where the optimal potentials $f$ and $g$ defined in Equation (\ref{eq:optimal_potentials}) can be written as
    $$f =  \qcal(U) + \mathrm{cst}~~~~~\mbox{and}~~~~~g =  \qcal(V) + \mathrm{cst}.$$
    We can now show that the optimal coupling $\pi^*$ is Gaussian, and write its covariance matrix in terms of $F$ and $G$.

    \begin{lemma}
    \label{lem:expression_H_inverse}
        Let $\pi^*$ be the optimal coupling defined in Theorem \ref{appthm:extension_janati}, and $F$ and $G$ defined in Equation (\ref{eq:link_F_U}).
        Then $\pi^*$ is a Gaussian measure on $\R^d \times \R^d$, whose covariance matrix is equal to
        $$H \coloneqq \varepsilon \begin{pmatrix}
            F & -A^\top \\ -A & G
        \end{pmatrix}^{-1}.$$
    \end{lemma}

    \begin{proof}
        Using Equation (\ref{eq:optimal_potentials}) and then Equation (\ref{eq:link_F_U}), we have that
        \begin{align*}
            \frac{\dd \pi^*}{\dd x \dd y}(x,y) &= \exp\prt{f(x) + g(y) - \frac{\modu{Qx - Ky}^2}{2 \varepsilon}} \frac{\dd \mu}{\dd x}(x) \frac{\dd \nu}{\dd y}(y) \\
            &\propto \exp \prt{\qcal(U + \Sigma^{-1})(x) +\qcal(V + \Omega^{-1})(y) + \frac{1}{\varepsilon} \qcal\prt{\begin{smallmatrix}
                Q^\top Q & -A^\top \\ -A &K^\top K
            \end{smallmatrix}}\prt{\begin{smallmatrix}
                x \\ y
            \end{smallmatrix}}    } \\
            &\propto \exp\prt{ \frac{1}{\varepsilon}\qcal\prt{\begin{smallmatrix}
                \varepsilon U + \varepsilon \Sigma^{-1} + Q^\top Q & -A^\top \\ -A &\varepsilon V + \varepsilon \Omega^{-1} + K^\top K
            \end{smallmatrix}}\prt{\begin{smallmatrix}
                x \\ y
            \end{smallmatrix}}    } \\
            &\propto \exp \prt{ \frac{1}{\varepsilon} \qcal\prt{\begin{smallmatrix}
                F & -A^\top \\ -A & G
            \end{smallmatrix}} \prt{\begin{smallmatrix}
                x \\ y
            \end{smallmatrix}   } }.
        \end{align*}
        The matrix
        $$H^{-1} \coloneqq \frac{1}{\varepsilon} \begin{pmatrix}
            F & -A^\top \\ -A & G
        \end{pmatrix}$$
        is positive definite, as $G \succ 0 $ and its Schur complement is equal to $\varepsilon^{-1}(F - A^\top G^{-1} A) = \Sigma^{-1} \succ 0$ with Equation (\ref{eq:link_F_G}).
    \end{proof}
    
    The following result paves the way to finding an explicit expression for $H$.
    
    \begin{lemma}
        \label{lem:expression_C}
        Let $F$ and $G$ be two positive definite $d\times d$ matrices satisfying
        \begin{equation*}
            \begin{gathered}
                F = \varepsilon \Sigma^{-1} + A^\top G^{-1} A, \\
                G = \varepsilon \Omega^{-1} + A F^{-1} A^\top.
            \end{gathered}
        \end{equation*}
        Denote
        \begin{gather}
            C_F \coloneqq \Omega A F^{-1} A^\top, \nonumber\\
            C_G \coloneqq \Sigma A^\top G^{-1} A. \label{eq:link_C_G}
        \end{gather}
        We have
        \begin{gather*}
            C_F = \Omega^{1/2}\prt{\Omega^{1/2} A \Sigma A^\top \Omega^{1/2} + \frac{\varepsilon^2}{4}I_d}^{1/2} \Omega^{-1/2} - \frac{\varepsilon}{2}I_d, \\
            C_G = \Sigma^{1/2} \prt{\Sigma^{1/2} A^\top \Omega A\Sigma^{1/2} + \frac{\varepsilon^2}{4}I_d}^{1/2} \Sigma^{-1/2} - \frac{\varepsilon}{2}I_d,
        \end{gather*}
        where the square roots are well-defined, as they are taken on positive definite matrices.
    \end{lemma}

    \begin{proof}
        Let us prove the result for $C_F$ only, by symmetry.
        A very similar computation as in \cite{janati2020entropic} shows that
        $$C_F^2 + \varepsilon C_F = \Omega A \Sigma A^\top.$$
        Noticing that
        $$C_F = \Omega^{1/2} \prt{\Omega^{1/2} A F^{-1} A^\top \Omega^{1/2}}\Omega^{-1/2}$$
        shows that all the eigenvalues of $C_F$ are positive, and it is easy to check that the matrix $\Omega^{1/2}\prt{\Omega^{1/2} A \Sigma A^\top \Omega^{1/2} + \frac{\varepsilon^2}{4}I_d}^{1/2} \Omega^{-1/2} - \frac{\varepsilon}{2}I_d$ is the unique solution of the equation
        $$C^2 + \varepsilon C = \Omega A \Sigma A^\top$$
        with only positive eigenvalues.
    \end{proof}

    Let us now use the block inversion formula on $H^{-1}$ to obtain the desired expression for $H$.
    \begin{align*}
        H &= \varepsilon \begin{pmatrix}
            F & -A^\top \\ -A & G
        \end{pmatrix}^{-1} \\
        &= \varepsilon\begin{pmatrix}
            (F - A^\top G^{-1}A)^{-1} & (F - A^\top G^{-1}A)^{-1} A^\top G^{-1} \\
            (G - A F^{-1} A^\top)^{-1} A F^{-1} & (G - A F^{-1} A^\top)^{-1}
        \end{pmatrix} \\
        &= \begin{pmatrix}
            \Sigma & \Sigma A^\top G^{-1} \\ \Omega A F^{-1} & \Omega
        \end{pmatrix}
    \end{align*}
    using Equation (\ref{eq:link_F_G}).
    We can conclude noticing that $\Omega A F^{-1} =  (\Sigma A^\top G^{-1})^\top$ as $H$ is symmetric, so that 
    $$\Sigma A^\top G^{-1} = (C_F A^{-\top})^\top$$
    by definition of $C_F$.
\end{proof}

\section{Drifting Models with Gaussian Kernels}
\label{appsec:drifting_models}

Drifting models \cite{deng2026generative} form a one-step generative modeling
paradigm in which the model distribution is evolved by a data-dependent drifting
field, and recent work relates this field to Wasserstein gradient flows of
KDE-approximated divergences \cite{cao2026gradient}.
Although drifting models are not the focus of this paper, the Gaussian-kernel
case has a direct connection with the Transformer PDE studied above: its
velocity uses the same normalized Gaussian-kernel vector field as
$\revision{\ell^2}$ self-attention, so the Gaussian calculations of Section
\ref{sec:gaussian} can be reused directly.

For simplicity, we write the unit-bandwidth Gaussian kernel
\[
    k(x,y)\coloneqq \exp\left(-\frac12 \modu{x-y}^2\right).
\]
For a probability measure $\mu\in \pcal(\R^d)$, define the normalized
Gaussian-kernel field
\[
    \Gamma_\mu\satt{G}(x)
    \coloneqq
    \frac{\int_{\R^d} y\, k(x,y)\dd \mu(y)}
    {\int_{\R^d} k(x,z)\dd \mu(z)}
\]
and the associated mean-shift field
\[
    a_\mu(x)\coloneqq \Gamma_\mu\satt{G}(x)-x
    =
    \frac{\int_{\R^d} (y-x) k(x,y)\dd \mu(y)}
    {\int_{\R^d} k(x,z)\dd \mu(z)}.
\]
The drifting PDE from a current distribution $\rho_t$ toward a fixed target
distribution $\nu$ reads
\begin{equation}
    \label{appeq:drifting_pde}
    \partial_t\rho_t
    +
    \mathrm{div}\left(\rho_t\left(a_\nu-a_{\rho_t}\right)\right)
    =0.
\end{equation}
Since the two $-x$ terms cancel, this is equivalently
\begin{equation}
    \label{appeq:drifting_attention_form}
    \partial_t\rho_t
    +
    \mathrm{div}\left(\rho_t\left(\Gamma_\nu\satt{G}-\Gamma_{\rho_t}\satt{G}\right)\right)
    =0.
\end{equation}
Moreover, with the convention of Section \ref{sec:att_variants},
$\Gamma_\mu\satt{G}$ is exactly the $\revision{\ell^2}$ attention field with
$Q=K=I_d/\sqrt{2}$ and $V=I_d$.
Thus, drifting with a Gaussian kernel is governed by the same normalized
Gaussian-kernel mechanism as one of the attention variants analyzed in this
paper, with the only difference that the velocity is the target field minus the
self-field of the current distribution.

This observation makes the Gaussian case explicit.

\begin{prop}[Gaussian drifting dynamics]
\label{appprop:gaussian_drifting}
    Assume that
    $\nu=\ncal(m_\nu,\Sigma_\nu)$ and
    $\rho_0=\ncal(m_0,\Sigma_0)$, with
    $m_\nu,m_0\in\R^d$ and
    $\Sigma_\nu,\Sigma_0\in \mathcal S_d^{++}$.
    Then, for all times for which the solution of
    \eqref{appeq:drifting_pde} exists, it remains Gaussian:
    $\rho_t=\ncal(m_t,\Sigma_t)$.
    Moreover,
    \begin{equation}
        \label{appeq:drifting_mean}
        \dot m_t = (\Sigma_\nu+I_d)^{-1}(m_\nu-m_t)
    \end{equation}
    and
    \begin{equation}
        \label{appeq:drifting_cov}
        \dot \Sigma_t = B_t\Sigma_t+\Sigma_t B_t^\top,
        \qquad
        B_t\coloneqq
        (\Sigma_t+I_d)^{-1}-(\Sigma_\nu+I_d)^{-1}.
    \end{equation}
\end{prop}

\begin{proof}
    Let $\mu=\ncal(m,\Sigma)$.
    Completing the square in the Gaussian reweighting gives
    \[
        \Gamma_\mu\satt{G}(x)
        =
        m+\Sigma(\Sigma+I_d)^{-1}(x-m)
        =
        x-(\Sigma+I_d)^{-1}(x-m).
    \]
    Hence
    $a_\mu(x)=-(\Sigma+I_d)^{-1}(x-m)$.
    If $\rho_t=\ncal(m_t,\Sigma_t)$, then the velocity in
    \eqref{appeq:drifting_pde} is affine:
    \[
        a_\nu(x)-a_{\rho_t}(x)
        =
        B_t(x-m_t)+(\Sigma_\nu+I_d)^{-1}(m_\nu-m_t),
    \]
    with $B_t$ defined in \eqref{appeq:drifting_cov}.
    This affine form closes the dynamics on Gaussian measures.
    Integrating the velocity against $\rho_t$ yields
    \eqref{appeq:drifting_mean}; after subtracting this mean motion, one obtains
    \[
        a_\nu(x)-a_{\rho_t}(x)-\dot m_t=B_t(x-m_t),
    \]
    and differentiating the covariance gives \eqref{appeq:drifting_cov}.
\end{proof}

In dimension $d=1$, writing $\Sigma_t=s_t$ and $\Sigma_\nu=s_\nu$, the covariance
ODE becomes
\[
    \dot s_t
    =
    2s_t\left(\frac{1}{s_t+1}-\frac{1}{s_\nu+1}\right).
\]
Thus $s_\nu$ is the stable stationary variance: if $s_t>s_\nu$, then
$\dot s_t<0$, while if $s_t<s_\nu$, then $\dot s_t>0$.
More generally, when $\Sigma_0$ and $\Sigma_\nu$ are simultaneously
diagonalizable, the dynamics decouple into these scalar equations along their
common eigenbasis.
This gives a simple interpretation of the drifting mechanism in the Gaussian
setting: the target field contracts or expands each covariance direction until
the current kernel-smoothed conditional mean matches the target one.

\clearpage

\begin{figure}[t]
    \centering
    \begin{tabular}{cc}
        \includegraphics[width=0.455\linewidth]{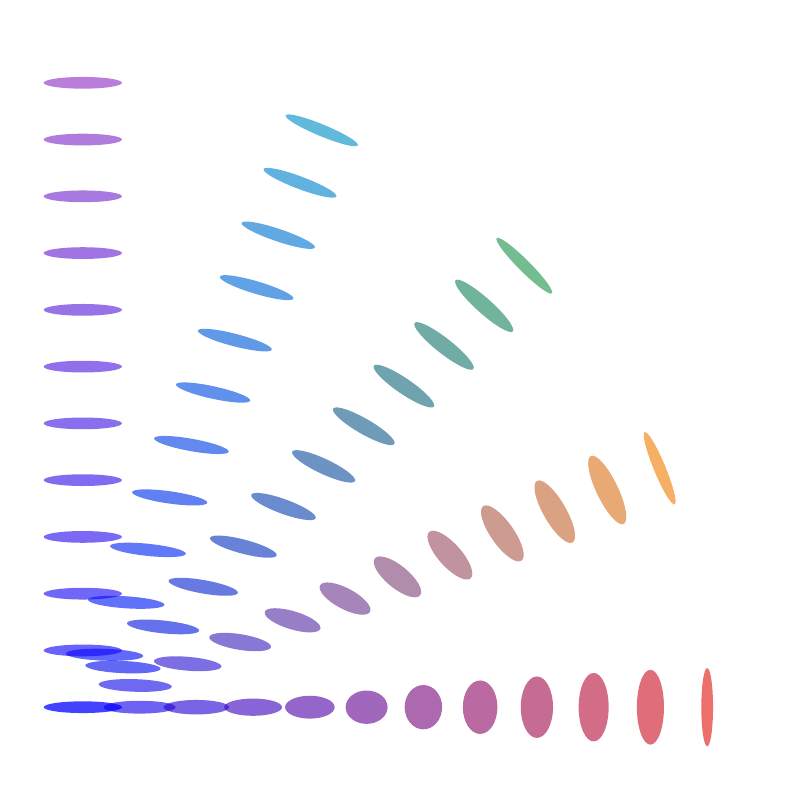}
        &
        \includegraphics[width=0.455\linewidth]{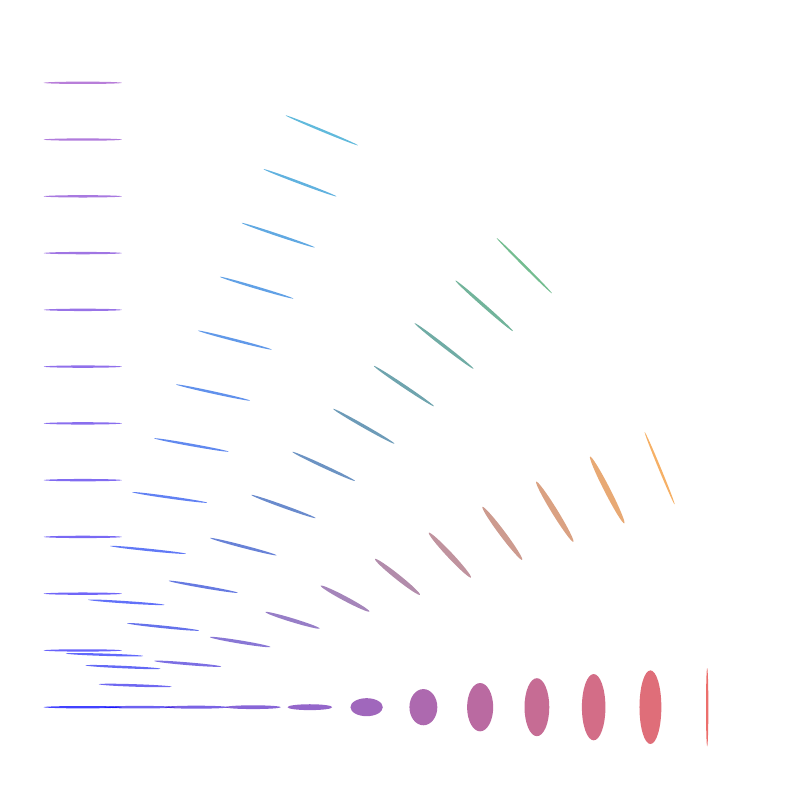}
        \\
        \small (a) Anisotropy $0.15$
        &
        \small (b) Anisotropy $0.03$
    \end{tabular}
    \caption{Gaussian-kernel drifting between anisotropic Gaussian measures in
    dimension two.  The source covariance is horizontal, while the five target
    covariances have principal axes tangent to a quarter circle.  Each path is
    represented by transparent ellipses sampled along the covariance evolution
    and displayed on radial rays only for readability.}
    \label{appfig:drifting_gaussians}
\end{figure}

Figure \ref{appfig:drifting_gaussians} illustrates this covariance dynamics in
dimension two.  The displayed paths connect a fixed anisotropic Gaussian to
target Gaussian covariances whose principal axes rotate along a quarter circle; the
ellipses are sampled along the covariance path and displayed with a common
global scale.  This visualization should be read as a qualitative picture of
the Gaussian reduction above: the drifting path simultaneously changes the
anisotropy and the orientation of the covariance until it matches the target.

\section{Proofs of Section \ref{sec:well_posedness}}
\label{appsec:compact_support}

Let us start with the following useful result.

\begin{lemma}
\label{lem:control_variance}
    Let $R>0$.
    Let $\mu$ be a probability measure on $\R^d$ supported in $B_R$.
    Denote $\alpha \in \R^d$ the expectation of $\mu$.
    Then 
    $$\norm{\var \mu}_2 \le R^2,$$
    where
    $$\var \mu \coloneqq \mathbb E_{X \sim \mu} (X - \alpha)(X - \alpha)^\top.$$
\end{lemma}

\begin{proof}
    Let $X$ be a random variable distributed according to $\mu$.
    Using the triangle inequality, we have:
    \begin{align*}
        \norm{\mathbb E (X - \alpha)(X - \alpha)^\top}_2 &\le \mathbb E \norm{(X - \alpha)(X - \alpha)^\top}_2 \\
        &= \mathbb E \norm{X - \alpha}^2\\
        &= \mathbb E \norm{X}^2 - \norm{\alpha}^2 \\
        &\le \mathbb E \norm{X}^2\\ &\revision{\le R^2,}
    \end{align*}
    \revision{as $\mu$ is supported in $B_R$.}
\end{proof}

\subsection{Estimates With Constant Parameters}
\label{appsubsec:estimates}

We gather in this section the estimates on the velocity field $\Gamma_\mu$ when $Q,K, V$ are constant, for each type of unmasked self-attention, \revision{first in their single-head version, then for the multi-head case}.
Notations are the same as in Section \ref{sec:att_variants}.
We start with Softmax self-attention.

\begin{lemma}[Estimates for Softmax self-attention]
    \label{lem:estimates_gamma}
    Let $p\ge 1$ and $R>0$.
    Let $\mu$ and $\nu$ be two probability measures supported in $B_R$.
    We have the following estimates.
    \begin{enumerate}[label=(\roman*)]
        \item $\sup_{x\in \R^d}\modu{\Gamma_\mu\satt{SM}} \le \norm{V}_2 R$, \label{eq:1_trad}
        \item $\sup_{x\in \R^d}\norm{D_x\Gamma_\mu\satt{SM}}_2 \le \norm{V}_2  \norm{A}_2 R^2$,\label{eq:2_trad}
        \item $\modu{\Gamma_\mu\satt{SM}(x) - \Gamma_\nu\satt{SM}(x)} \le   \norm{V}_2 (1 + 2\norm{A}_2 R \modu{x}) e^{2 \norm{A}_2 R \modu{x}} W_p(\mu, \nu)$.\label{eq:3_trad}
    \end{enumerate}
\end{lemma}

\begin{proof}
    Equation \ref{eq:1_trad} is straightforward. Equation \ref{eq:2_trad} relies on the following remark:
    $$D_x \Gamma_\mu = V \var\prt{\frac{e^{Ax\cdot y}}{\int e^{Ax\cdot z}\dd\mu(z)}\dd \mu(y)} A.$$
    As the probability measure $\frac{e^{Ax\cdot y}}{\int e^{Ax\cdot z}\dd\mu(z)}\dd \mu(y)$ is supported in $B_R$, its variance is bounded by $R^2$, according to Lemma \ref{lem:control_variance}, which proves \ref{eq:2_trad}.
    Finally, \revision{let us derive \ref{eq:3_trad}, with a similar method as in the proof of Lemma 6.5 in \cite{geshkovski2023emergence}.}
    \revision{By adding and subtracting $\frac{\int e^{Ax\cdot y}Vy\dd\nu(y)}{\int e^{Ax\cdot y}\dd\mu(y)}$, we bound
    \begin{align*}
        \lvert \Gamma_\mu(x) - \Gamma_\nu(x)\rvert &\le \frac{\lvert \int e^{Ax\cdot y}Vy\dd(\mu - \nu)(y)\rvert}{\int e^{Ax\cdot y}\dd\mu(y)} \\
        &\quad+ \lvert\int e^{Ax\cdot y}Vy\dd\nu(y) \rvert \frac{\lvert \int e^{Ax\cdot y}\dd(\mu - \nu)(y)\rvert}{\int e^{Ax\cdot y}\dd\mu(y)\int e^{Ax\cdot y}\dd\nu(y)}\\
        &= \frac{\lvert \int e^{Ax\cdot y}Vy\dd(\mu - \nu)(y)\rvert}{\int e^{Ax\cdot y}\dd\mu(y)} +\lvert \Gamma_\nu(x)\rvert \frac{\lvert \int e^{Ax\cdot y}\dd(\mu - \nu)(y)\rvert}{\int e^{Ax\cdot y}\dd\mu(y)}.
    \end{align*}
    We bound each component separately. By the duality formula for $W_1$, we have
    \begin{align*}
        \lvert \int e^{Ax\cdot y}Vy\dd(\mu - \nu)(y)\rvert &\le \mathrm{Lip}_{B_R}(y\mapsto e^{Ax\cdot y}Vy)W_1(\mu, \nu)\\
        &\le e^{\lVert A\rVert_2R\lvert x\rvert}\lVert V\rVert_2 (1 + \lVert A\rVert_2R\lvert x\rvert)W_1(\mu,\nu).
    \end{align*}
    Similarly,
    \begin{align*}
        \lvert \int e^{Ax\cdot y}\dd(\mu - \nu)(y)\rvert &\le \mathrm{Lip}_{B_R}(y\mapsto e^{Ax\cdot y})W_1(\mu, \nu)\\
        &\le e^{\lVert A\rVert_2R\lvert x\rvert} \lVert A\rVert_2\lvert x\rvert W_1(\mu,\nu).
    \end{align*}
    Moreover,
    \begin{align*}
        \int e^{Ax\cdot y}\dd \mu(y)\ge e^{-\lVert A\rVert_2 R\lvert x\rvert},
    \end{align*}
    and the same bound holds for $\nu$. Putting everything together and using \ref{eq:1_trad}, we obtain
    \begin{align*}
        \lvert \Gamma_\mu(x)- \Gamma_\nu(x)\rvert &\le e^{2\lVert A\rVert_2R\lvert x\rvert}\lVert V\rVert_2(1 + \lVert A\rVert_2R\lvert x\rvert)W_1(\mu, \nu)\\&\quad + \lVert V\rVert_2Re^{2\lVert A\rVert_2R\lvert x\rvert}\lVert A\rVert_2\lvert x\rvert W_1(\mu, \nu)\\
        &\le \norm{V}_2 (1 + 2\norm{A}_2 R \modu{x}) e^{2 \norm{A}_2 R \modu{x}} W_p(\mu, \nu)
    \end{align*}
    as $W_1\le W_p$ in $B_R$.}
\end{proof}

We obtain the following $\revision{\ell^2}$ estimates with the same method as for Softmax self-attention.

\begin{lemma}[Estimates for $\revision{\ell^2}$ self-attention]
    Let $p\ge 1$ and $R>0$.
    Let $\mu$ and $\nu$ be two probability measures supported in $B_R$.
    We have the following estimates.
    \begin{enumerate}[label=(\roman*)]
        \item $\sup_{x\in \R^d}\modu{\Gamma_\mu\satt{\revision{\ell^2}}} \le \norm{V}_2 R$,
        \item $\sup_{x\in\R^d}\norm{D_x\Gamma_\mu\satt{\revision{\ell^2}}}_2 \le  \norm{V}_2 \norm{A}_2 R^2$,
        \item $\modu{\Gamma_\mu\satt{\revision{\ell^2}}(x) - \Gamma_\nu\satt{\revision{\ell^2}}(x)} \le  \norm{V}_2 \prt{1 + 4R(\norm{A}_2 \modu{x} + \norm{K^\top K}_2 R)}$ \\ \phantom{$\modu{\Gamma_\mu\satt{\revision{\ell^2}}(x) - \Gamma_\nu\satt{\revision{\ell^2}}(x)} \le  \norm{V}_2$} $\times\, e^{(\norm{Q}_2\modu{x} + \norm{K}_2 R)^2} W_p(\mu, \nu)$.
    \end{enumerate}
\end{lemma}

\revision{\begin{proof}
    The proof follows the same steps as for Lemma \ref{lem:estimates_gamma}, replacing the kernel $e^{Ax\cdot y}$ with $e^{-\lvert Qx-Ky\rvert^2}$, and using the following bounds:
    \begin{align*}
        \mathrm{Lip}_{B_R}(y\mapsto e^{-\lvert Qx-Ky\rvert^2})&\le 2(\norm{A}_2 \modu{x} + \norm{K^\top K}_2 R)\\
        \mathrm{Lip}_{B_R}(y\mapsto e^{-\lvert Qx-Ky\rvert^2} Vy)&\le \norm{V}_2 (1 + 2R(\norm{A}_2\modu{x} + \norm{K^\top K}_2 R))\\
        \int e^{-\lvert Qx-Ky\rvert^2} \dd \mu(y) &\ge e^{-(\norm{Q}_2\modu{x} + \norm{K}_2 R)^2}.
    \end{align*}
\end{proof}}

The estimates for Sinkhorn self-attention are more involved, as they leverage some background on entropic optimal transport (see Section \ref{appsubsec:dual_EOT}).
We only state them for $\varepsilon = 1$ for simplicity, as the $\varepsilon$ can be absorbed in the matrices $Q$, $K$ and $V$.
Let us start by proving estimates on $\kappa_\mu^\infty$.
All notation is defined in Section \ref{appsubsec:dual_EOT}.

\begin{lemma}
    Let $\kappa^0\colon (x,y) \in (\R^d)^2 \mapsto e^{-\frac12\modu{Qx - Ky}^2}$ and $\kappa_\mu^\infty$ defined by the iterates (\ref{eq:sinkhorn_iterations}).
    Then
    \begin{enumerate}[label=(\roman*)]
        \item $\sup_{x, y \in B_R} \kappa_\mu^\infty(x,y) \le e^{(\norm{Q}_2 + \norm{K}_2)^2 R^2},$
        \item $\sup_{x, y \in B_R} \norm{D_y \kappa_\mu^\infty(x,y)}_2 \le 2 (\norm{Q}_2 + \norm{K}_2)^2 R e^{(\norm{Q}_2 + \norm{K}_2)^2 R^2}.$
    \end{enumerate}
\end{lemma}

\begin{proof}
    Let $x,y\in B_R$.
    We have with Lemma \ref{lem:schrödinger_estimates} that
    \begin{align*} 
        \kappa_\mu^\infty(x, y) = e^{f(x) + g(y) - \modu{Qx - Ky}^2} \le e^{(\norm{Q}_2 + \norm{K}_2)^2 R^2},
    \end{align*}
    which proves (i).
    For (ii), differentiating $\kappa_\mu^\infty$ with respect to the second variable gives
    \begin{align*}
        D_y \kappa_\mu^\infty(x, \cdot) = \kappa_\mu^\infty(x, y) (D_y g - D_y c(x, \cdot)) = \kappa_\mu^\infty(x, y) (D_y g - (Ky - Qx)^\top K)
    \end{align*}
    so that
    $$\norm{D_y \kappa_\mu^\infty(x, \cdot)}_2 \le (\norm{Q}_2 + \norm{K}_2 + (\norm{Q}_2 + \norm{K}_2)^{3/2}) R e^{(\norm{Q}_2 + \norm{K}_2)^2 R^2}$$
    with Lemma \ref{lem:schrödinger_estimates} and the first estimate on $\kappa_\mu^\infty$.
\end{proof}

We can then state estimates on $\Gamma$, as for the other types of attention.

\begin{lemma}[Estimates for Sinkhorn self-attention]
    Let $R > 0$ and $p\ge 1$.
    Let $\mu$ and $\nu$ be two probability measures supported in $B_R$.
    We have the following estimates for all $x\in \R^d$.
    \begin{enumerate}[label=(\roman*)]
        \item $\sup_{x\in \R^d} \modu{\Gamma_{\mu}\satt{sink}(x)} \le \norm{V}_2 R$,
        \item $\sup_{x\in \R^d}\norm{D_x \Gamma_{\mu}\satt{sink}}_2 \le \norm{V}_2 \norm{A}_2 R^2$,
        \item $\modu{\Gamma_\mu\satt{sink}(x) - \Gamma_\nu\satt{sink}(x)} \le \norm{V}_2  \prt{1 + c(R, Q, K) R  + 2 (\norm{Q}_2 + \norm{K}_2)^2 R\modu{x} }$\\\phantom{$\modu{\Gamma_\mu\satt{sink}(x) - \Gamma_\nu\satt{sink}(x)} \le \norm{V}_2$}$\times \,e^{ (\norm{Q}_2 +\norm{K}_2)^2 R\modu{x}}W_2(\mu, \nu)$
    \end{enumerate}
    for some function $c(R, Q, K) > 0$.
\end{lemma}

\begin{proof}
    The first estimate is straightforward.
    To prove the second estimate, notice that
    \begin{align*}
        D_x f &= - \int e^{g(y) - c(x,y)}(Ky - Qx)^\top Q \dd \mu(y) / e^{-f(x)} \\
        &=  \int \kappa_\mu^\infty(x,y)(Qx-Ky)^\top Q \dd \mu(y) \\
        &= x^\top Q^\top Q - \int \kappa_\mu^\infty(x,y)y^\top A \dd \mu(y).
    \end{align*}
    Then
    \begin{align*}
        D_x \kappa_\mu^\infty(\cdot, y) &= \kappa_\mu^\infty (x,y)(D_x f + (Ky - Qx)^\top Q) \\
        &= \kappa_\mu^\infty (x,y) \left (y^\top - \int \kappa_\mu^\infty (x, y') y'^\top \dd \mu(y)\right)A,
    \end{align*}
    and finally
    \begin{align*}
        D_x \Gamma_\mu &= V\left (\int yy^\top \kappa_\mu^\infty(x,y) \dd \mu(y) - \int y\kappa_\mu^\infty(x,y) \dd \mu(y) \int y^\top \kappa_\mu^\infty (x,y)\dd\mu(y)\right ) A \\
        & = V \var(\kappa_\mu^\infty(x,y)\dd \mu(y)) A,
    \end{align*}
    which allows us to conclude with Lemma \ref{lem:control_variance}.
    The last estimate builds on Lemma \ref{lem:carlier}.
    Let us first derive a bound on $\sup_{x,y\in B_R} \modu{\kappa_\mu^\infty(x, y) - \kappa_\nu^\infty(x, y)}$.
    For all $x, y \in B_R$, we have
    \begin{align*} 
        \modu{\kappa_\mu^\infty(x, y) - \kappa_\nu^\infty(x, y)} &\le \modu{e^{f^\mu(x) + g^\mu(y) - \frac12\modu{Qx - Ky}^2} - e^{f^\nu(x) + g^\nu(y) - \frac12\modu{Qx - Ky}^2}} \\
        &\le e^{- \frac12\modu{Qx - Ky}^2} e^{\max(f^\mu(x) + g^\mu(y), f^\nu(x) + g^\nu(y))}\\
        &\phantom{\le e^{- \frac12\modu{Qx - Ky}^2}}\times\, \modu{f^\mu(x) - f^\nu(x) + g^\mu(y) - g^\nu(y)}.
    \end{align*}
    We can choose $f^\mu$ and $g^\mu$ such that
    $$\norm{f^\mu - f^\nu}_\infty + \norm{g^\mu - g^\nu}_\infty = \inf_{\eta\in \R} \norm{f^\mu - f^\nu - \eta}_\infty + \norm{g^\mu - g^\nu + \eta}_\infty,$$
    in order to have
    $$\modu{f^\mu(x) - f^\nu(x) + g^\mu(y) - g^\nu(y)}\le \norm{f^\mu - f^\nu}_\infty + \norm{g^\mu - g^\nu}_\infty \le C_\mathrm{stab}(R)W_2(\mu, \nu)$$
    with Lemma \ref{lem:carlier}.
    Finally, we obtain with a variant of Lemma \ref{lem:schrödinger_estimates}
    $$\modu{\kappa_\mu^\infty(x, y) - \kappa_\nu^\infty(x, y)} \le C(R) e^{ (\norm{Q}_2 +\norm{K}_2)^2 R\modu{x}} W_2(\mu, \nu),$$
    where we bounded $e^{- \frac12\modu{Qx - Ky}^2}$ above by 1 and $\max(f^\mu(x) + g^\mu(y), f^\nu(x) + g^\nu(y))$ above by $(\norm{Q}_2 +\norm{K}_2)^2 R\modu{x}$.
    Note that this is a very rough bound, but it is enough to prove the continuity of $\Gamma$.
    Then, for all $x\in \R^d$ we have
    \begin{align*}
        \modu{\Gamma_\mu(x) - \Gamma_\nu(x)} &\le \norm{V}_2 \Big (\modu{\int y \kappa^\infty_\mu(x,y) \dd \mu(y) - \int y \kappa^\infty_\nu(x,y) \dd \mu(y)} \\
        &~~~~~~~~~~~~~~~~~~~~~~~~~~~~~ + \modu{\int y \kappa^\infty_\nu(x,y) \dd \mu(y) - \int y \kappa^\infty_\nu(x,y) \dd \nu(y)}\Big ) \\
        &\le \norm{V}_2 \Big (R\sup_{y \in B_R}\modu{\kappa_\mu^\infty (x,y) - \kappa_\nu^\infty(x,y)} \\ 
        &\phantom{\le \norm{V}_2}+ \sup_{y\in B_R}\norm{D_y(y\mapsto y \kappa^\infty_\nu(x,y))}_2 W_2(\mu, \nu)\Big ) \\
        &\le \norm{V}_2 e^{(\norm{Q}_2 +\norm{K}_2)^2 R\modu{x}} \big ( 1 + C_\mathrm{stab}(R) R  +  \\
        &~~~~~~~(\norm{Q}_2 + \norm{K}_2 + (\norm{Q}_2 + \norm{K}_2)^{3/2}) R\modu{x} \big ) W_2(\mu, \nu).
    \end{align*} 
\end{proof}

\revision{Next}, we have the following estimates for Sigmoid attention.

\begin{lemma}[Estimates for Sigmoid self-attention]
    Let $R > 0$ and $p\ge 1$.
    Let $\mu$ and $\nu$ be two probability measures supported in $B_R$.
    We have the following estimates for all $x\in \R^d$.
    \begin{enumerate}[label=(\roman*)]
        \item $\sup_{x\in \R^d} \modu{\Gamma_\mu\satt{\sigma}(x)} \le \norm{V}_2 R$,\label{eq:sigm_1}
        \item $\sup_{x\in \R^d}\norm{D_x \Gamma_\mu\satt{\sigma}}_2 \le \frac14 \norm{V}_2 \norm{A}_2 R^2$,\label{eq:sigm_2}
        \item $\modu{\Gamma_\mu\satt{\sigma}(x) - \Gamma_\nu\satt{\sigma}(x)} \le \norm{V}_2 (1 + \norm{A}_2 R\modu{x}/4) W_p(\mu, \nu)$.\label{eq:sigm_3}
    \end{enumerate}
\end{lemma}

\begin{proof}
    Equation \ref{eq:sigm_1} is straightforward.
    For Equation \ref{eq:sigm_2}, we calculate
    $$D_x\Gamma\satt{\sigma}_\mu = \int Vyy^\top\! A\, \sigma(Ax\cdot y)(1 - \sigma(Ax\cdot y))\dd \mu(y),$$
    and the bound follows, noticing that $\modu{\sigma(Ax\cdot y)(1 - \sigma(Ax\cdot y))}\le 1 / 4$ as $\modu{\sigma(Ax\cdot y)}\in [0,1]$.
    Finally, we have
    $$\modu{\Gamma_\mu(x) - \Gamma_\nu(x)} = \modu{\int Vy\sigma(Ax\cdot y) \dd(\mu - \nu)(y)},$$
    and 
    \begin{align*}
    D_y (y\mapsto Vy\sigma(Ax\cdot y))
    =
    V\big(&\sigma(Ax\cdot y)I_d \\
    &+\sigma(Ax\cdot y)(1 - \sigma(Ax\cdot y))y x^\top A^\top\big),
    \end{align*}
    so that the Lipschitz constant of this map is bounded by
    $\norm{V}_2(1 + \norm{A}_2 R\modu{x}/4)$ if $\mu$ and $\nu$ are supported in
    $B_R$.
    We conclude for Equation \ref{eq:sigm_3} with the duality formula for $W_1$, and with the inequality $W_1(\mu, \nu) \le W_p(\mu, \nu)$, as $\mu$ and $\nu$ are compactly supported.
\end{proof}

\revision{Estimates for multi-head attention can be derived for each attention variant from the following result, that is a simple application of the triangle inequality.}

\revision{\begin{lemma}[Estimates for multi-head self-attention]
\label{applem:mh_estimates}
    Let $p\ge 1$ and $R>0$.
    Let $\mu$ and $\nu$ be two probability measures supported in $B_R$.
    Let $\Gamma$ be a velocity field satisfying the following estimates:
    \begin{enumerate}[label=(\roman*)]
        \item $\sup_{x\in \R^d}\modu{\Gamma_\mu} \le \norm{V}_2 R$,
        \item $\sup_{x\in \R^d}\norm{D_x\Gamma_\mu}_2 \le \norm{V}_2\norm{A}_2R^2$,
        \item $\modu{\Gamma_\mu(x) - \Gamma_\nu(x)} \le  c(\modu{x},R)  W_p(\mu, \nu)$,
    \end{enumerate}
    where $c(\modu{x},R)$ is a continuous function that depends on $\modu{x}, R$ and $Q,K,V$.
    Denote $\Gamma\satt{MH}$ the multi-head version of $\Gamma$. Then,
    \begin{enumerate}[label=(\roman*)]
        \item $\sup_{x\in \R^d}\modu{\Gamma_\mu\satt{MH}} \le \sum_{h=1}^H\norm{V^{(h)}}_2 R$,
        \item $\sup_{x\in \R^d}\norm{D_x\Gamma_\mu\satt{MH}}_2 \le \sum_{h=1}^H\norm{V^{(h)}}_2 \norm{A^{(h)}}_2 R^2$,
        \item $\modu{\Gamma\satt{MH}_\mu(x) - \Gamma\satt{MH}_\nu(x)} \le  c(\modu{x},R) W_p(\mu, \nu)$,
    \end{enumerate}
    where $c(\modu{x},R)$ is a continuous function that depends on $\modu{x}, R$ and on the parameters $(Q^{(h)},K^{(h)},V^{(h)})_{1\le h\le H}$.
\end{lemma}}

\revision{Finally, we derive estimates for single-head Softmax masked self-attention. Similar estimates can be obtained for multi-head Softmax masked self-attention, and $\ell^2$, Sinkhorn and Sigmoid masked self-attention in their single-head or multi-head version.
\begin{lemma}
\label{lem:softmax_masked_estimates}
    Let $p\ge 1$ and $R>0$. Let $\bar \mu, \bar \nu \in \pcal_c([0,1]\times B_R)$, with the same position marginal $\theta$. Assume that $\theta(\{0\}) > 0$. Let $\Gamma\satt{m}$ be Softmax masked self-attention, and let $\tilde\Gamma\satt{m}$ be defined as $\Gamma\satt{m} \eqqcolon (0, \tilde\Gamma\satt{m})$.
    We have the following estimates.
    \begin{enumerate}[label=(\roman*)]
        \item $\sup_{(\sigma,x)\in [0,1]\times\R^d}\modu{\tilde\Gamma_{\bar\mu}\satt{m}(\sigma,x)} \le \norm{V}_2 R$,
        \item $\sup_{(\sigma, x)\in [0,1]\times \R^d}\norm{\partial_x\tilde\Gamma_{\bar\mu}\satt{m}(\sigma,x)}_2 \le \norm{V}_2 \norm{A}_2 (R^2+1)$,
        \item $\modu{\tilde\Gamma\satt{m}_{\bar\mu}(\sigma, x) - \tilde\Gamma\satt{m}_{\bar\nu}(\sigma, x)} \le  \frac{c(\modu{x},R)}{\int_0^\sigma\dd\theta(\tau)} d(\bar\mu, \bar\nu)$.
    \end{enumerate}
\end{lemma}
\begin{proof}
    Let $\sigma\in[0,1]$.
    Recall that $\tilde\Gamma_{\bar\mu}\satt{m}(\sigma,x) = \frac{\int_0^\sigma\int_{\R^d}Vye^{Ax\cdot y}\dd\bar\mu(\tau, y)}{\int_0^\sigma\int_{\R^d}e^{Ax\cdot y}\dd\bar\mu(\tau, y)}$. For estimate $(i)$, we bound
    \begin{align*}
        \modu{\tilde \Gamma_{\bar\mu}\satt{m}(\sigma,x)} \le \frac{\int_0^\sigma\int_{\R^d}\modu{Vy}e^{Ax\cdot y}\dd\bar\mu(\tau, y)}{\int_0^\sigma\int_{\R^d}e^{Ax\cdot y}\dd\bar\mu(\tau, y)}\le \norm{V}_2R.
    \end{align*}
    Estimate $(ii)$ follows from:
        \begin{align*}
            \partial_x\tilde \Gamma_{\bar\mu}\satt{m}(\sigma,x) &= V\mathrm{Var}\left ( \frac{\mathbf{1}_{\tau\in [0,\sigma]}e^{Ax\cdot y}\dd\bar\mu(\tau,y)}{\int_{[0,1]\times \R^d}\mathbf{1}_{\omega\in [0,\sigma]}e^{Ax\cdot z}\dd\bar\mu(\omega, z)}\right ) A.
        \end{align*}
    Hence, with Lemma \ref{lem:control_variance}, $\norm{\partial_x\tilde\Gamma_{\bar\mu}\satt{m}(\sigma,x)}_2 \le \norm{V}_2 \norm{A}_2 (R^2+1)$, as the measure $\mathbf{1}_{\tau\in [0,\sigma]}e^{Ax\cdot y}\dd\bar\mu(\tau,y)$ is supported in $[0,1]\times B_R$.
    Finally, let us prove estimate $(iii)$.
    \begin{multline*}
        \lvert \tilde \Gamma_{\bar\mu}\satt{m}(\sigma, x) - \tilde \Gamma_{\bar\nu}\satt{m}(\sigma,x)\rvert \le \frac{\lvert \int_0^\sigma\int_{\R^d} e^{Ax\cdot y}Vy\dd(\bar\mu - \bar\nu)(\tau, y)\rvert}{\int_0^\sigma\int_{\R^d} e^{Ax\cdot y}\dd\bar \mu(\tau, y)} \\
        \quad+ \lvert\int_0^\sigma\int_{\R^d} e^{Ax\cdot y}Vy\dd\bar \nu(\tau, y) \rvert \frac{\lvert \int_0^\sigma\int_{\R^d} e^{Ax\cdot y}\dd(\bar \mu - \bar \nu)(\tau, y)\rvert}{\int_0^\sigma\int_{\R^d} e^{Ax\cdot y}\dd\bar \mu(\tau, y)\int_0^\sigma\int_{\R^d} e^{Ax\cdot y}\dd\bar \nu(\tau, y)},
    \end{multline*}
    which implies
    \begin{multline*}
         \lvert \tilde \Gamma_{\bar\mu}\satt{m}(\sigma, x) - \tilde \Gamma_{\bar\nu}\satt{m}(\sigma,x)\rvert \le \frac{ \int_0^\sigma\lvert\int_{\R^d} e^{Ax\cdot y} Vy\dd(\bar\mu^\tau - \bar\nu^\tau)(y)\rvert\dd\theta(\tau)}{\int_0^\sigma\int_{\R^d} e^{Ax\cdot y}\dd\bar\mu^\tau(y)\dd\theta(\tau)} \\+\lvert \tilde \Gamma_{\bar\nu}(\sigma, x)\rvert \frac{\lvert \int_0^\sigma\int_{\R^d} e^{Ax\cdot y}\dd(\bar\mu^\tau - \bar\nu^\tau)(y)\dd\theta(\tau)\rvert}{\int_0^\sigma\int_{\R^d} e^{Ax\cdot y}\dd\bar\mu^\tau(y)\dd\theta(\tau)}
    \end{multline*}
    with the triangle inequality and the disintegration theorem $\dd\bar\mu(\tau, x) = \dd\theta(\tau)\dd\bar\mu^\tau(x)$.
    We bound each component separately. By the duality formula for $W_1$, we have
    \begin{align*}
        \lvert \int_{\R^d} e^{Ax\cdot y}Vy\dd(\bar\mu^\tau - \bar\nu^\tau)(y)\rvert &\le \mathrm{Lip}_{B_R}(y\mapsto e^{Ax\cdot y}Vy)W_1(\bar\mu^\tau, \bar\nu^\tau)\\
        &\le e^{\lVert A\rVert_2R\lvert x\rvert}\lVert V\rVert_2 (1 + \lVert A\rVert_2R\lvert x\rvert)W_1(\bar\mu^\tau,\bar\nu^\tau).
    \end{align*}
    Similarly,
    \begin{align*}
        \lvert \int_{\R^d} e^{Ax\cdot y}\dd(\bar\mu^\tau - \bar\nu^\tau)(y)\rvert &\le \mathrm{Lip}_{B_R}(y\mapsto e^{Ax\cdot y})W_1(\bar\mu^\tau, \bar\nu^\tau)\\
        &\le e^{\lVert A\rVert_2R\lvert x\rvert} \lVert A\rVert_2\lvert x\rvert W_1(\bar\mu^\tau,\bar\nu^\tau).
    \end{align*}
    Moreover,
    \begin{align*}
        \int_{\R^d} e^{Ax\cdot y}\dd \mu^\tau(y)\ge e^{-\lVert A\rVert_2 R\lvert x\rvert}.
    \end{align*}
    Putting everything together and using \ref{eq:1_trad}, we obtain
    \begin{multline*}
        \lvert \tilde \Gamma_{\bar\mu}\satt{m}(\sigma, x) - \tilde \Gamma_{\bar\nu}\satt{m}(\sigma,x)\rvert \le e^{2\norm{A}_2R\modu{x}}(1 + \norm{A}_2R\modu{x})\frac{\int_0^\sigma W_1(\bar\mu^\tau, \bar\nu^\tau)\dd\theta(\tau)}{\int_0^\sigma \dd\theta(\tau)} \\+ \norm{V}_2\norm{A}_2Re^{2\norm{A}_2R\modu{x}}\frac{\int_0^\sigma W_1(\bar\mu^\tau, \bar\nu^\tau)\dd\theta(\tau)}{\int_0^\sigma \dd\theta(\tau)}.
    \end{multline*}
    We conclude noticing that $\int_0^\sigma W_1(\bar\mu^\tau, \bar\nu^\tau)\dd\theta(\tau)\le \int_0^1 W_1(\bar\mu^\tau, \bar\nu^\tau)\dd\theta(\tau)=d(\bar\mu, \bar\nu)$.
\end{proof}}

\subsection{Technical Lemmas}
\label{appsubsec:technical_lemmas}

We list here the technical lemmas used in the proof of Theorem \ref{thm:compact_support}.

\revision{\begin{lemma}
    \label{lem:improved_gronwall}
    Let $u\colon \R\to\R$, $v\colon \R\to\R$ and $b\colon \R\to\R$ be integrable functions such that, for all $t\in [0,T]$, it holds
    $$u(t)\le \int_0^t L(s) u(s) \dd s +\int_0^t b(s) \dd s.$$
    Then, we have the following upper bound on $u$:
    $$u(t)\le \int_0^t \exp\left (\int_s^tL(\tau)\dd \tau \right ) b(s) \dd s.$$
\end{lemma}}

\revision{\begin{proof}
    Denote $v(t)\coloneqq \int_0^t L(s) u(s) \dd s +\int_0^t b(s) \dd s$.
    Then
    \begin{align*}v'(t) &= L(t)u(t) + b(t)\\ &\le L(t) v(t) + b(t),\end{align*}
    so
    $$v'(t) - L(t) v(t) \le b(t).$$
    Multiplying both sides by $e^{-\int_0^t L(\tau)\dd\tau}$ and integrating between 0 and $t$ gives
    $$e^{-\int_0^tL(\tau)\dd \tau}v(t)\le \int_0^t e^{-\int_0^sL(\tau)\dd\tau}b(s)\dd s.$$
    We conclude noticing that $v(t)\ge u(t)$ and multiplying both sides by $e^{\int_0^tL(\tau)\dd \tau}$.
\end{proof}}

\begin{lemma}
    \label{lem:support_control}
    Let $\mu_0 \in \pcal_c(\R^d)$ be an initial condition, and assume that the associated problem (\ref{eq:transf_eq}) has an equi-compactly supported solution $\mu$ on the time interval $[0,T]$.
    For all $0\le t\le T$, denote $R(t)>0$ the smallest radius such that
    $$\supp \mu(\tau) \subset B_{R(t)},$$
    for all $0\le \tau\le t$.
    Then, it holds
    $$R(t) \le e^{\norm{V}_2 t} R(0)$$
    for all $0\le t\le T$.
\end{lemma}

\begin{proof}
    Let $x\in \R^d$ belong to the support of $\mu_0$.
    For any $t\in [0,T]$ we have by Lemma \ref{lem:estimates_gamma}:
    $$\modu{x - \phi_t(\mu)(x)} \le \int_0^t \modu{\Gamma_{\mu(s)}(\phi_s(\mu)(x))}\dd s \le \norm{V}_2 \int_0^tR(s)\dd s.$$
    By taking a supremum over $x$ in the support of $\mu_0$ we get 
    $$R(t) - R(0) \le \norm{V}_2 \int_0^tR(s)\dd s$$
    and Grönwall's inequality implies the claim.
\end{proof}

Lemma \ref{lem:support_control} allows us to prove that $\fcal$ preserves \revision{the set $X$ of equi-compactly supported curves $\mu\in \ccal_\text{co}([0,T], \pcal_c(\R^d))$ that satisfy $\supp \mu(t)\subset \bbar (0, e^{\lvert V\rVert_2 t}R_0)$ for any $0\le t\le T$}.

\begin{lemma}
    We have $\fcal(X)\subset X$.
\end{lemma}

\begin{proof}
    The proof is the same as for Lemma \ref{lem:support_control}, replacing $\mu_0$ with $\bar \mu_0$.
\end{proof}

We have the following estimate on $\fcal$.

\begin{lemma}
    \label{lem:contraction_estimate}
    Let $T>0$ and $\bar \mu_0 \in \pcal_c(\R^d)$.
    Denote $R_0$ the smallest radius such that $\supp \bar \mu_0 \subset B_{R_0}$, and set
    $$R(t) \coloneqq e^{\norm{V}_2t}R_0$$
    for all $0\le t\le T$.
    Consider $X$ the set of curves $\mu \in \ccal_{\mathrm{co}}([0,T], \pcal_c(\R^d))$ satisfying
    $$\supp \mu(t) \subset B_{R(t)} $$
    for all $0\le t\le T$, and define $\fcal \colon X \to X$ by
    $$\fcal(\mu) \coloneqq \phi_t(\mu)_\sharp \bar \mu_0.$$
    Then, for all $\mu,\nu \in X$, we have
    $$ \dcal_{p,T}(\fcal(\mu),\fcal(\nu)) \le f(T,R_0) \dcal_{p,T}(\mu,\nu) $$
    with $f$ a positive function such that $f(T,R_0)\to 0$ when $T\to 0^+$.
\end{lemma}

\begin{proof}
    Let us first prove that
    \beq
    \label{eq:smart_gronwall}
    \norm{\phi_t(\mu) - \phi_t(\nu)}_{L^\infty(\supp \bar \mu_0)} \le C_1(R(T))\int_0^t e^{C_2 R(T)^2(t-s)}W_p(\mu(s),\nu(s))\dd s
    \eeq
    for all $0\le t\le T$, with
    $$C_1(R(T)) \coloneqq c(R(T), R(T))$$
    (see \ref{eq:ref_3}) and
    $$C_2 \coloneqq \norm{V}_2 \norm{A}_2.$$
    For all $x\in \supp \bar \mu_0$ and $t\in [0,T]$, it holds:
    \begin{align*}
        \modu{\phi_t(\mu)(x) - \phi_t(\nu)(x)} &\le \int_0^t \modu{\Gamma_\mu(s,\phi_s(\mu)(x)) - \Gamma_\nu(s,\phi_s(\nu)(x))}\dd s \\
        &\le \int_0^t \modu{\Gamma_\mu(s,\phi_s(\mu)(x)) - \Gamma_\mu(s,\phi_s(\nu)(x))}\dd s
        \\ &\phantom{\le} + \int_0^t \modu{\Gamma_\mu(s, \phi_s(\nu)(x)) - \Gamma_\nu(s, \phi_s(\nu)(x))}\dd s \\
        &\le C_2 R(T)^2 \int_0^t \modu{\phi_s(\mu)(x) - \phi_s(\nu)(x)} \dd s \\
        &\phantom{\le C_2 R(T)^2}+ C_1(R(T))\int_0^t W_p(\mu(s), \nu(s))\dd s 
    \end{align*}
    where we used Equations \ref{eq:ref_1}, \ref{eq:ref_2}, \ref{eq:ref_3} for the last inequality, noticing with Lemma \ref{lem:support_control} that $\modu{\phi_s(\nu)(x)} \le R(T)$ for all $s\in [0,T]$.
    An application of Grönwall's inequality \revision{(Lemma \ref{lem:improved_gronwall})} leads to Equation (\ref{eq:smart_gronwall}).
    Then, by a standard result in optimal transport \revision{(see for instance \cite[Lemma A.6]{castin2024smooth})}:
    \begin{align*}
        W_p(\fcal(\mu)(t), \fcal(\nu)(t)) &= W_p( \phi_t(\mu)_\sharp \bar \mu_0, \phi_t(\nu)_\sharp \bar \mu_0) \\
        &\le \norm{\phi_t(\mu) -\phi_t(\nu)}_{L^\infty(\supp \bar \mu_0)} \\
        &\le C_1(R(T))\prt{\int_0^t e^{C_2 R(T)^2(t-s)}\dd s} \dcal_{p,T}(\mu,\nu)\\
        &=\frac{C_1(R(T))}{C_2 R(T)^2}\prt{e^{C_2 R(T)^2 t} - 1} \dcal_{p,T}(\mu, \nu),
    \end{align*}
    for all $0\le t\le T$, thanks to Equation (\ref{eq:smart_gronwall}).
    Taking the supremum over $t\in [0,T]$ gives the desired estimate with
    $$f(T, R_0) \coloneqq \frac{C_1(R(T))}{C_2 R(T)^2}\prt{e^{C_2 R(T)^2 T} - 1}.$$
    We can see that $f(T,R_0)\to 0$ when $T\to 0^+$, which concludes the proof.
\end{proof}

We also need the following result to apply a Banach fixed point argument.

\begin{lemma}
    \label{lem:complete_space}
    $X$ equipped with the distance $\dcal_{p,T}$ is a complete metric space.
\end{lemma}

\begin{proof}
    Let $(\mu^n)_{n\in \N}$ be a Cauchy sequence taking values in $X$.
    Seeing $X$ as a subspace of the complete metric space $\ccal([0,T], \pcal_p(\R^d))$, we know that $\mu^n$ converges to some $\mu^*\in \ccal([0,T], \pcal_p(\R^d))$.
    Assume by contradiction that $\mu^*$ does not belong to $X$.
    Then, there exists some $t\in [0,T]$ such that $\supp \mu^*(t)$ is not included in $B_{R(t)}$, with $R(t)$ defined in Lemma \ref{lem:contraction_estimate}.
    Let $x\in \supp \mu^*(t)\setminus B_{R(t)}$.
    There is a small closed ball $\bcal$ centered at $x$, containing a neighborhood of $x$, and included in the complement of $B_{R(t)}$, such that $\mu^*(t)(\bcal)>0$.
    Let $\delta > 0$ such that for all $x'\in \bcal$ and $y\in B_{R(t)}$, it holds $\modu{x'-y} \ge \delta$.
    Then, for all $n\in \N$, we have
    $$W_p(\mu^*(t), \mu^n(t))^p \ge \delta^p \mu^*(t)(\bcal),$$
    which contradicts the fact that $\dcal_{p,T}(\mu^n,\mu^*)\to 0$ when $n\to +\infty$.
\end{proof}

\

\subsection{Stability estimates}
\label{apppar:stab_estim}
As for the stability estimate, let $t\ge 0$.
We have
\begin{align*}
    W_p(\mu(t), \nu(t)) &= W_p( \phi_t(\mu)_\sharp \mu_0, \phi_t(\nu)_\sharp \nu_0)\\
    &\le W_p( \phi_t(\mu)_\sharp \mu_0, \phi_t(\mu)_\sharp \nu_0) \\
    &\phantom{\le} + W_p( \phi_t(\mu)_\sharp \nu_0, \phi_t(\nu)_\sharp \nu_0) \\
    &\le \lip(x\mapsto \phi_t(\mu)(x)) W_p(\mu_0, \nu_0) + \norm{\phi_t(\mu) - \phi_t(\nu)}_{L^\infty(\supp \nu_0)}.
\end{align*}
Let us bound the Lipschitz constant $\lip(x\mapsto \phi_t(\mu)(x))$.
For all $x,y\in \R^d$:
\begin{align*}
    \modu{ \phi_t(\mu)(x)- \phi_t(\mu)(y)} &\le \modu{x-y} +\int_0^t \modu{\Gamma_\mu(s, \phi_s(\mu)(x)) - \Gamma_\mu(s, \phi_s(\mu)(y))} \dd s \\
    &\le \modu{x - y} + C_2 R(T)^2 \int_0^t \modu{ \phi_s(\mu)(x)- \phi_s(\mu)(y)} \dd s
\end{align*}
by Lemma \ref{lem:estimates_gamma}, with $C_2 = \norm{V}_2\norm{A}_2$.
Then, Grönwall's inequality entails that
$$\modu{ \phi_t(\mu)(x)- \phi_t(\mu)(y)} \le e^{C_2 R(T)^2 t} \modu{x - y}.$$
Plugging this into our previous bound for $W_p(\mu(t), \nu(t))$, together with Equation (\ref{eq:smart_gronwall}), we get:
$$W_p(\mu(t), \nu(t)) \le e^{C_2 R(T)^2 t} W_p(\mu_0, \nu_0) + C_1(R(T))\int_0^t e^{C_2 R(T)^2(t-s)}W_p(\mu(s), \nu(s))\dd s.$$
We can then apply Grönwall's inequality to $t\mapsto e^{-C_2 R(T)^2 t} W_p (\mu(t), \nu(t))$ to obtain
$$W_p(\mu(t), \nu(t)) \le e^{\prt{C_1(R(T))+C_2 R(T)^2}t}W_p(\mu_0, \nu_0),$$
which concludes the proof, recalling that $R(T) = e^{\norm{V}_2T}R_0$.

\subsection{Proof \revision{for Unmasked Self-Attention} with Varying Parameters}
\label{appsubsec:varying_params}

The proof of Theorem \ref{thm:compact_support} with time-dependent parameters is similar to the case of constant parameters, with a few additional computations\revision{, and replacing Cauchy's theorem with Carathéodory's existence theorem}.
For any equi-compactly supported curve $\mu\in \ccal_\mathrm{co}([0,T], \pcal_c(\R^d))$, consider again the Cauchy problem
\begin{equation}
\label{appeq:caratheodory_system}
\begin{cases}
    \dot r(t) = \Gamma_{\mu}(t, r(t)) \mbox{ for }0\le t \le T\\
    r(0) = x \in \R^d
\end{cases}.
\end{equation}
\revision{As seen in the proof of Theorem \ref{thm:compact_support} with constant parameters, if $\Gamma_\mu$ is associated with the parameters $Q, K,V$ and $\supp \mu \subset B_R$, then, denoting $A \coloneqq K^\top Q$, we have the following estimates.
\begin{enumerate}[label=(\roman*)]
    \item $\sup_{x\in \R^d}\modu{\Gamma_\mu(x)} \le \norm{V}_2 R$, \label{appeq:ref_1}
    \item $\sup_{x\in \R^d}\norm{D_x\Gamma_\mu}_2 \le \norm{V}_2  \norm{A}_2 R^2$,\label{appeq:ref_2}
    \item $\modu{\Gamma_\mu(x) - \Gamma_\nu(x)} \le   c(\modu{x}, R, Q, K, V) W_p(\mu, \nu), $\label{appeq:ref_3}
\end{enumerate}
where $c(\modu{x}, R, Q, K, V)$ is a continuous function that depends on $\modu{x}, R$ and $Q, K, V$.
Hence, the system (\ref{appeq:caratheodory_system}) satisfies the assumptions of Carathéodory's existence theorem: $t\mapsto \Gamma_\mu(t,x)$ is measurable for every $x\in \R^d$, $x\mapsto \Gamma_\mu(t, x)$ is continuous for every $t\in [0,T]$, and $\lvert\Gamma_\mu(t,x)\rvert$ is bounded above by the measurable function $t\mapsto \norm{V(t)}_2R(t)$. Therefore, (\ref{appeq:caratheodory_system}) has at least one solution. Moreover, this solution is unique as, thanks to estimate \ref{appeq:ref_2}, we have for any $x_1, x_2\in \R^d$,
$$\lvert \Gamma_\mu(t, x_1) - \Gamma_\mu(t, x_2)\rvert\le \norm{V(t)}_2\norm{A(t)}_2 R^2 \modu{x_1- x_2},$$
where $t\mapsto \norm{V(t)}_2\norm{A(t)}_2 R^2$ is integrable.
We can then define the flow $\phi_t(\mu)(x)$ associated to (\ref{appeq:caratheodory_system})—note that we have not used equation \ref{appeq:ref_3} yet.
}

%

Now set a compactly supported initial condition $\bar \mu_0$ supported in $B_{R_0}$ with $R_0$ minimal, define $X$ the set of curves $\mu \in \ccal_\mathrm{co}([0,T], \pcal_c(\R^d))$ such that for all $t\in [0,T]$, we have
$$\supp \mu(t) \subset \bbar\prt{0, e^{\int_0^t \norm{V(s)}_2 \dd s}R_0},$$
and consider the map
$$\fcal\colon \mu \in X \mapsto \phi_t(\mu)_\sharp \bar \mu_0.$$

\begin{lemma}
\label{applem:support_control}
    We have $\fcal(X) \subset X$.
\end{lemma}

\begin{proof}
    Let $\mu \in X$ and $x$ in the support of $\bar \mu_0$.
    For all $t\in [0, T]$, denote $R(t)$ the smallest radius such that $\mu(s)$ is supported in $B_{R(t)}$ for all $s\in [0, t]$.
    Then
    \begin{align*}
        \modu{x - \phi_t(\mu)(x)} &\le \int_0^t \modu{\Gamma_\mu(s, \phi_s(\mu)(x))} \dd s\\
        &\le \int_0^t \norm{V(s)}_2 R(s) \dd s,
    \end{align*}
    using Lemma \ref{lem:estimates_gamma}.
    Taking the supremum over $x$ in the support of $\bar \mu_0$, we get
    $$R(t) - R(0) \le \int_0^t \norm{V(s)}_2 R(s) \dd s.$$
    Grönwall's inequality then proves the claim.
\end{proof}

Then, following a similar strategy as for Lemma \ref{lem:contraction_estimate}, we obtain the following result.

\begin{lemma}
    Let $T>0$ and $\bar \mu_0 \in \pcal_c(\R^d)$.
    Denote $R_0$ the smallest radius such that $\supp \bar \mu_0 \subset B_{R_0}$, and set
    $$R(t) \coloneqq e^{\int_0^t \norm{V(s)}_2 \dd s}R_0$$
    for all $0\le t\le T$.
    Consider $X$ the set of curves $\mu \in \ccal_{\mathrm{co}}([0,T], \pcal_c(\R^d))$ satisfying
    $$\supp \mu(t) \subset B_{R(t)} $$
    for all $0\le t\le T$, and define $\fcal \colon X \to X$ by
    $$\fcal(\mu) \coloneqq \phi_t(\mu)_\sharp \bar \mu_0.$$
    Then, for all $\mu,\nu \in X$, we have
    $$ \dcal_{p,T}(\fcal(\mu),\fcal(\nu)) \le f(T,R_0) \dcal_{p,T}(\mu,\nu) $$
    with $f$ a positive function such that $f(T,R_0)\to 0$ when $T\to 0^+$.
\end{lemma}

\revision{\begin{proof}
    Let us first prove that
    \begin{equation}
    \label{appeq:smart_gronwall}
    \norm{\phi_t(\mu) - \phi_t(\nu)}_{L^\infty(\supp \bar \mu_0)}
     \le \int_0^t e^{ \int_s^t \norm{V(\tau)}_2 \norm{A(\tau)}_2 R(\tau)^2\dd\tau}c_1(s) W_p(\mu(s), \nu(s))\dd s
    \end{equation}
    for all $0\le t\le T$, with the notation of \ref{appeq:ref_3} and with
    $$c_1(s)\coloneqq c(R(s), R(s), Q(s), K(s), V(s)).$$
    For all $x\in \supp \bar \mu_0$ and $t\in [0,T]$, it holds:
    \begin{align*}
        \modu{\phi_t(\mu)(x) - \phi_t(\nu)(x)} &\le \int_0^t \modu{\Gamma_\mu(s,\phi_s(\mu)(x)) - \Gamma_\nu(s,\phi_s(\nu)(x))}\dd s \\
        &\le \int_0^t \modu{\Gamma_\mu(s,\phi_s(\mu)(x)) - \Gamma_\mu(s,\phi_s(\nu)(x))}\dd s
        \\ &\phantom{\le} + \int_0^t \modu{\Gamma_\mu(s, \phi_s(\nu)(x)) - \Gamma_\nu(s, \phi_s(\nu)(x))}\dd s \\
        &\le \int_0^t \norm{V(s)}_2\norm{A(s)}_2 R(s)^2\modu{\phi_s(\mu)(x) - \phi_s(\nu)(x)} \dd s \\
        &\quad + \int_0^t c(R(s), R(s), Q(s), K(s), V(s)) W_p(\mu(s), \nu(s))\dd s 
    \end{align*}
    where we used Equations \ref{appeq:ref_1}, \ref{appeq:ref_2}, \ref{appeq:ref_3} for the last inequality, noticing with Lemma \ref{applem:support_control} that $\modu{\phi_s(\nu)(x)} \le R(T)$ for all $s\in [0,T]$.
    An application of Grönwall's inequality (Lemma \ref{lem:improved_gronwall}) leads to Equation (\ref{appeq:smart_gronwall}).
    Then, by a standard result in optimal transport \revision{(see for instance \cite[Lemma A.6]{castin2024smooth})}:
    \begin{align*}
        W_p(\fcal(\mu)(t), \fcal(\nu)(t)) &= W_p( \phi_t(\mu)_\sharp \bar \mu_0, \phi_t(\nu)_\sharp \bar \mu_0) \\
        &\le \norm{\phi_t(\mu) -\phi_t(\nu)}_{L^\infty(\supp \bar \mu_0)} \\
        &\le \int_0^t e^{ \int_s^t \norm{V(\tau)}_2 \norm{A(\tau)}_2 R(\tau)^2\dd\tau} c_1(s)\dd s\ \dcal_{p,T}(\mu,\nu),\\
    \end{align*}
    for all $0\le t\le T$, thanks to Equation (\ref{appeq:smart_gronwall}).
    Taking the supremum over $t\in [0,T]$ gives the desired estimate with
    $$f(T, R_0) \coloneqq \int_0^T e^{ \int_s^T \norm{V(\tau)}_2 \norm{A(\tau)}_2 R(\tau)^2\dd\tau} c_1(s)\dd s.$$
    We can see that $f(T,R_0)\to 0$ when $T\to 0^+$, which concludes the proof.
\end{proof}}

As $X$ is a complete metric space, which derives from the proof of Lemma \ref{lem:complete_space}, we can follow the exact same steps as in the previous subsection, and conclude that problem (\ref{eq:transf_eq}) is well-posed even for \revision{integrable} time-dependent parameters $A$ and $V$.

\medskip
\revision{\textbf{Stability estimate.} As for the stability estimate, let $t\ge 0$.
We have
\begin{align*}
    W_p(\mu(t), \nu(t)) &= W_p( \phi_t(\mu)_\sharp \mu_0, \phi_t(\nu)_\sharp \nu_0)\\
    &\le W_p( \phi_t(\mu)_\sharp \mu_0, \phi_t(\mu)_\sharp \nu_0) \\
    &\phantom{\le} + W_p( \phi_t(\mu)_\sharp \nu_0, \phi_t(\nu)_\sharp \nu_0) \\
    &\le \lip(x\mapsto \phi_t(\mu)(x)) W_p(\mu_0, \nu_0) + \norm{\phi_t(\mu) - \phi_t(\nu)}_{L^\infty(\supp \nu_0)}.
\end{align*}
Let us bound the Lipschitz constant $\lip(x\mapsto \phi_t(\mu)(x))$.
For all $x,y\in \R^d$:
\begin{align*}
    \modu{ \phi_t(\mu)(x)- \phi_t(\mu)(y)} &\le \modu{x-y} +\int_0^t \modu{\Gamma_\mu(s, \phi_s(\mu)(x)) - \Gamma_\mu(s, \phi_s(\mu)(y))} \dd s \\
    &\le \modu{x - y} +  \int_0^t \norm{V(s)}_2\norm{A(s)}_2R(s)^2\modu{ \phi_s(\mu)(x)- \phi_s(\mu)(y)} \dd s
\end{align*}
by Lemma \ref{lem:estimates_gamma}.
Then, Grönwall's inequality entails that
$$\modu{ \phi_t(\mu)(x)- \phi_t(\mu)(y)} \le e^{\int_0^t \norm{V(s)}_2\norm{A(s)}_2R(s)^2\dd s} \modu{x - y}.$$
Plugging this into our previous bound for $W_p(\mu(t), \nu(t))$, together with Equation (\ref{appeq:smart_gronwall}), we get:
\begin{multline*} W_p(\mu(t), \nu(t)) \le e^{\int_0^t \norm{V(s)}_2\norm{A(s)}_2R(s)^2\dd s} W_p(\mu_0, \nu_0) \\+ \int_0^t e^{\int_s^t\norm{V(\tau)}_2\norm{A(\tau)}_2R(\tau)^2\dd\tau }c_1(s) W_p(\mu(s), \nu(s))\dd s.\end{multline*}
We then apply Grönwall's inequality to $t\mapsto e^{-\int_0^t \norm{V(s)}_2\norm{A(s)}_2R(s)^2\dd s} W_p (\mu(t), \nu(t))$ to obtain
\begin{multline*}W_p(\mu(t), \nu(t)) \le W_p(\mu_0, \nu_0)  \exp\Bigg ( \int_0^t \norm{V(s)}_2\norm{A(s)}_2R(s)^2 \dd s\\ +\int_0^t c_1(s)e^{\int_0^t\norm{V(\tau)}_2\norm{A(\tau)}_2R(\tau)^2\dd\tau} \dd s\Bigg ),\end{multline*}
which concludes the proof, recalling that $R(T) = e^{\int_0^t\norm{V(s)}_2\dd s}R_0$.}

\subsection{Masked Attention and Wasserstein Distance}
\label{appsubsec:masked}

The use of the conditional Wasserstein distance for the masked Transformer PDE circumvents the impossibility of proving an estimate of the form
$$\modu{\Gamma_{\bar \mu}\satt{m}(\sigma, x) - \Gamma_{\bar \nu}\satt{m}(\sigma, x)} \le C(x) W_2(\bar \mu, \bar \nu).$$
Indeed, we have the following negative result.

\begin{prop}
    \revision{Let $\acal\subset \R^d$ containing at least two points. If $V\neq 0$, the masked attention map $$\Gamma\colon \bar \mu \in (\pcal_c([0, 1]\times \acal), W_2) \mapsto \Gamma_{\bar \mu}\in (([0, 1]\times\R^d)^{[0, 1]\times\R^d}, \norm{\cdot}_\infty)$$} defined in Section \ref{subsec:masked} is not continuous for the Wasserstein 2 distance $W_2$.
\end{prop}

\revision{\begin{proof}
    Let $\sigma\in[0,1)$.
    It suffices to prove that no estimate of the form
    \begin{equation}\modu{\Gamma_{\bar \mu}\satt{m}(\sigma, x) - \Gamma_{\bar \nu}\satt{m}(\sigma, x)} \le C(x) W_2(\bar \mu, \bar \nu)\end{equation}
    can hold, for any $x\in\R^d$.
    Let $x\neq y\in \acal$.
    For each $n\in \mathbb{N}^*$, denote
    $$\begin{cases}
        \bar\mu_n = \frac13\delta_{(0,0)} + \frac13 \delta_{(\sigma, x)} + \frac13 \delta_{(\sigma+\frac1n,y)}\\
        \bar \nu_n = \frac13\delta_{(0,0)} + \frac13 \delta_{(\sigma,y)} + \frac13 \delta_{(\sigma+\frac1n, x)}.
    \end{cases}$$
    \begin{figure}[ht]
        \definecolor{violetfonce}{rgb}{0.4, 0.0, 0.4}
        \centering
        \begin{tikzpicture}[>=stealth, scale=2]
            \def\s{0.8}
            \def\ninv{0.5} 
            \def\xval{0.6}
            \def\yval{1.4}
            \draw[->] (-0.3, 0) -- (2.2, 0) node[right] {$[0, 1]$};
            \draw[->] (0, -0.3) -- (0, 2) node[above] {$\mathcal{A}$};
            \draw[dashed, gray!40] (0, \xval) node[left, black] {$x$} -- (1.8, \xval);
            \draw[dashed, gray!40] (0, \yval) node[left, black] {$y$} -- (1.8, \yval);
            \draw[dashed, gray!40] (\s, 0) node[below, black] {$\sigma$} -- (\s, 1.8);
            \draw[dashed, gray!40] (\s + \ninv, 0) node[below, black] {$\sigma + \frac{1}{n}$} -- (\s + \ninv, 1.8);
            \filldraw[violetfonce] (0,0) circle (1.2pt) node[below left, black] {$(0,0)$};
            \filldraw[blue] (\s, \xval) circle (1.2pt);
            \filldraw[blue] (\s + \ninv, \yval) circle (1.2pt);
            \filldraw[red] (\s, \yval) circle (1.2pt);
            \filldraw[red] (\s + \ninv, \xval) circle (1.2pt);
            \draw[thick, black] (\s, \xval) -- (\s + \ninv, \xval);
            \draw[thick, black] (\s, \yval) -- (\s + \ninv, \yval);
            \begin{scope}[shift={(2.3, 1)}]
                \node[blue, anchor=west] at (0, 0.4) {$\bullet \ \bar\mu_n = \frac{1}{3}\delta_{(0,0)} + \frac{1}{3}\delta_{(\sigma,x)} + \frac{1}{3}\delta_{(\sigma+\frac{1}{n},y)}$};
                \node[red, anchor=west] at (0, 0) {$\bullet \ \bar\nu_n = \frac{1}{3}\delta_{(0,0)} + \frac{1}{3}\delta_{(\sigma,y)} + \frac{1}{3}\delta_{(\sigma+\frac{1}{n},x)}$};
            \end{scope}
        \end{tikzpicture}
        \caption{Visualization of $\mu_n$ and $\nu_n$. The optimal coupling for $W_2$ is represented in black.}
        \label{appfig:mu_n_nu_n}
    \end{figure}
    For $n$ large enough, one has $W_2(\bar\mu_n, \bar\nu_n) = \frac{\sqrt{2}}{n}\to_{n\to +\infty} 0$ (see Figure \ref{appfig:mu_n_nu_n}).
    However, $\lvert \Gamma_{\bar\mu}\satt{m}(\sigma,x) - \Gamma_{\bar\nu}\satt{m}(\sigma,x)\rvert$ is non-zero and independent of $n$, as $\bar\mu_n$ and $\bar\nu_n$ do not change on $[0,\sigma]\times \acal$, which proves the claim.
\end{proof}}

\revision{\textbf{Proof of Theorem \ref{thm:mask_compact_support}: well-posedness for masked self-attention.}
It suffices to do the proof for single-head masked self-attention, as we have the same type of estimates in the single-head and multi-head cases (Lemma \ref{applem:mh_estimates}).
The proof follows the same steps as for Theorem \ref{thm:compact_support}.
For the sake of clarity, we will use in this proof the Latin alphabet ($t, s$) to indicate the time of the PDE—which can also be seen as the depth inside the Transformer, while we will use the Greek alphabet ($\sigma, \tau$) for time variables in the masked attention—these variables encode the order of tokens in the input.
For any equi-compactly supported curve $\bar \mu\in \ccal_\mathrm{co}([0,T], \pcal_c([0,1]\times \R^d))$, and any $(\sigma, x)\in [0,1]\times \R^d$, consider the Cauchy problem
\begin{equation}
\label{appeq:full_flow}
    \begin{cases}
        (\dot \tau,\dot r)(t) = \Gamma_{\bar\mu(t)}(\sigma, r(t)) \mbox{ for }0\le t \le T\\
        (\tau, r)(0) = (\sigma, x),
    \end{cases}
\end{equation}
where $\Gamma = \Gamma\satt{m}$ is the single-headed masked attention velocity field (we omit the exponent to lighten notation).
As the first component of $\Gamma_{\bar\mu(t)}(\sigma, r(t))$ is zero, $\tau$ is constant over time and, denoting $\Gamma_{\bar\mu(t)} \eqqcolon (0, \tilde \Gamma_{\bar\mu})$, the Cauchy problem can be simplified as
\begin{equation}
\label{appeq:caratheodory_system_masked}
\begin{cases}
    \dot r(t) = \tilde \Gamma_{\bar\mu(t)}(\sigma, r(t)) \mbox{ for }0\le t \le T\\
    r(0) = x \in \R^d.
\end{cases}
\end{equation}}
\revision{According to Lemma \ref{lem:softmax_masked_estimates}, we have the following estimates, for any $\bar\mu$ and $\bar\nu$ in $\pcal([0,1]\times B_R)$.
\begin{enumerate}[label=(\roman*)]
    \item $\sup_{(\sigma,x)\in [0,1]\times\R^d}\modu{\tilde\Gamma_{\bar\mu}(\sigma,x)} \le \norm{V}_2 R$,\label{appeq:ref_1_masked}
    \item $\sup_{(\sigma, x)\in [0,1]\times \R^d}\norm{\partial_x\tilde\Gamma_{\bar\mu}(\sigma,x)}_2 \le \norm{V}_2 \norm{A}_2 (R^2+1)$,\label{appeq:ref_2_masked}
    \item $\modu{\tilde\Gamma_{\bar\mu}(\sigma, x) - \tilde\Gamma_{\bar\nu}(\sigma, x)} \le  \frac{c(\modu{x},R, Q, K, V)}{\int_0^\sigma\dd\theta(\tau)} d(\bar\mu, \bar\nu)$,\label{appeq:ref_3_masked}
\end{enumerate}
where $c(\modu{x}, R, Q, K, V)$ is a continuous function that depends on $\modu{x}, R$ and $Q, K, V$.
Hence, the system (\ref{appeq:caratheodory_system_masked}) satisfies the assumptions of Carathéodory's existence theorem: $t\mapsto \tilde\Gamma_{\bar\mu(t)}(\sigma, x)$ is measurable for every $x\in \R^d$, $x\mapsto \tilde\Gamma_{\bar\mu(t)}(\sigma, x)$ is continuous for every $t\in [0,T]$, and $\lvert\tilde\Gamma_{\bar\mu(t)}(\sigma, x)\rvert$ is bounded above by the measurable function $t\mapsto \norm{V(t)}_2R(t)$. Therefore, (\ref{appeq:caratheodory_system_masked}) has at least one solution. Moreover, this solution is unique thanks to estimate \ref{appeq:ref_2_masked}.
We can then define the flow $\phi_t^\sigma(\bar\mu)(x)$ associated with (\ref{appeq:caratheodory_system_masked}), and the flow $\phi_t(\bar\mu)(\sigma,x)$ associated with (\ref{appeq:full_flow})—note that we have not used equation \ref{appeq:ref_3_masked} yet.
}

\revision{Now set a compactly supported initial condition $\dbbar{\mu}_0$ supported in $[0,1]\times B_{R_0}$ with $R_0$ minimal, and define $X$ as the set of curves $\bar \mu \in \ccal_\mathrm{co}([0,T], \pcal_c([0,1]\times \R^d))$ satisfying the following two conditions.
\begin{enumerate}
    \item Their position marginal is constant: there exists $\theta\in \pcal([0,1])$ such that for all $t\in [0,T]$, it holds $\int_{x\in\R^d}\dd\bar\mu_t(\sigma, x) = \dd\theta(\sigma)$.
    \item For all $t\in [0,T]$, we have
$$\supp \mu(t) \subset \bbar\prt{0, e^{\int_0^t \norm{V(s)}_2 \dd s}R_0},$$
where $\mu$ is the space marginal of $\bar\mu$, defined as $\dd\mu(x) \coloneqq \int_{\sigma\in[0,1]}\dd\bar\mu(\sigma, x)$.
\end{enumerate}
Then, consider the map
$$\fcal\colon \mu \in X \mapsto \phi_t(\bar\mu)_\sharp \dbbar{\mu}_0.$$}

\revision{
Following the same steps as for Lemma \ref{applem:support_control}, we obtain the following result.
\begin{lemma}
    \label{applem:support_control_masked}
    We have $\fcal(X) \subset X$.
\end{lemma}}

\revision{Then, we control the Lipschitz constant of $\fcal$ as follows.}

\revision{\begin{lemma}
    For all $\mu,\nu \in X$, we have
    $$ \sup_{t\in[0,T]}d(\fcal(\bar\mu),\fcal(\bar\nu)) \le f(T,R_0) \sup_{t\in[0,T]}d(\bar\mu,\bar\nu),$$
    with $f$ a positive function such that $f(T,R_0)\to 0$ when $T\to 0^+$.
\end{lemma}}

\revision{\begin{proof}
    Let us first prove that for any $(\sigma, x)\in \supp \dbbar{\mu}_0$, it holds
    \begin{multline}
    \label{appeq:smart_gronwall_masked}
    \modu{\phi_t^\sigma(\bar\mu)(x) - \phi_t^\sigma(\bar\nu)(x)}\\
     \le \int_0^t e^{ \int_s^t \norm{V(u)}_2 \norm{A(u)}_2 (R(u)+1)^2\dd u}\frac{c_1(s)}{\int_0^\sigma\dd\theta(\tau)} d(\bar\mu(s), \bar\nu(s))\dd s
    \end{multline}
    for all $0\le t\le T$, where
    $$c_1(s)\coloneqq c(R(s), R(s), Q(s), K(s), V(s))$$
    with the notation of \ref{appeq:ref_3_masked}.
    For all $x\in \supp \bar \mu_0$ and $t\in [0,T]$, it holds:
    \begin{align*}
        \modu{\phi_t^\sigma(\bar\mu)(x) - \phi_t^\sigma(\bar\nu)(x)} &\le \int_0^t \modu{\tilde\Gamma_{\mu(s)}(\sigma,\phi_s^\sigma(\bar\mu)(x)) - \tilde \Gamma_{\nu(s)}(\sigma,\phi_s^\sigma(\bar\nu)(x))}\dd s \\
        &\le \int_0^t \modu{\tilde\Gamma_{\mu(s)}(\sigma,\phi_s^\sigma(\bar\mu)(x)) - \tilde\Gamma_{\mu(s)}(\sigma,\phi_s^\sigma(\bar\nu)(x))}\dd s
        \\ &\phantom{\le} + \int_0^t \modu{\tilde\Gamma_{\mu(s)}(\sigma, \phi_s^\sigma(\bar\nu)(x)) - \tilde\Gamma_{\nu(s)}(\sigma, \phi_s^\sigma(\bar\nu)(x))}\dd s \\
        &\le \int_0^t \norm{V(s)}_2\norm{A(s)}_2 (R(s)^2+1)\modu{\phi_s^\sigma(\bar\mu)(x) - \phi_s^\sigma(\bar\nu)(x)} \dd s \\
        &\quad + \int_0^t c_1(s) \frac{\int_0^\sigma W_1(\bar\mu^\tau(s), \bar\nu^\tau(s))\dd\theta(\tau)}{\int_0^\sigma\dd\theta(\tau)}\dd s 
    \end{align*}
    where we used Equations \ref{appeq:ref_1_masked}, \ref{appeq:ref_2_masked}, \ref{appeq:ref_3_masked} for the last inequality, noticing with Lemma \ref{applem:support_control_masked} that $\modu{\phi_s^\sigma(\bar\nu)(x)} \le R(s)$ for all $s\in [0,T]$.
    An application of Grönwall's inequality (Lemma \ref{lem:improved_gronwall}) leads to Equation (\ref{appeq:smart_gronwall_masked}).
    Then, by a standard result in optimal transport \revision{(see for instance \cite[Lemma A.6]{castin2024smooth})}:
    \begin{align*}
        d(\fcal(\bar\mu)(t), \fcal(\bar\nu)(t)) &= \int_0^1  W_1( \phi_t^\sigma(\bar\mu)_\sharp \dbbar{\mu}_0, \phi_t^\sigma(\bar\nu)_\sharp \dbbar{\mu}_0) \dd\theta(\sigma)\\
        &\le \int_0^1\sup_{x:(\sigma,x)\in \supp \dbbar\mu_0} \modu{\phi_t^\sigma(\bar\mu)(x) -\phi_t^\sigma(\bar\nu)(x)}\dd\theta(\sigma) \\
        &\le \int_0^t e^{ \int_s^t \norm{V(u)}_2 \norm{A(u)}_2 R(u)^2\dd u}c_1(s) d(\bar\mu(s), \bar\nu(s))\dd s \\
        &\phantom{blabla}\times \int_0^1\frac{\dd\theta(\sigma)}{\int_0^\sigma\dd\theta(\tau)},\\
    \end{align*}
    for all $0\le t\le T$, thanks to Equation (\ref{appeq:smart_gronwall_masked}).
    Taking the supremum over $t\in [0,T]$, bounding $d(\bar\mu(s), \bar\nu(s))$ above by $\max_{t\in[0,T]}d(\bar\mu(t), \bar\nu(t))$ and bounding $\int_0^\sigma\dd\theta(\tau)$ below by $\theta(\{0\})$ gives the desired estimate with
    $$f(T, R_0) \coloneqq \int_0^T e^{ \int_s^T \norm{V(u)}_2 \norm{A(u)}_2 (R(u)^2+1)\dd u} \frac{c_1(s)}{\theta(\{0\})}\dd s.$$
    We see that $f(T,R_0)\to 0$ when $T\to 0^+$, which concludes the proof.
\end{proof}}

\revision{As $X$ is a complete metric space, which derives from the proof of Lemma \ref{lem:complete_space}, we can follow the exact same steps as for unmasked self-attention, and conclude that problem (\ref{eq:mask_transf_eq}) is well-posed.}

\medskip
\revision{\textbf{Stability estimate.} As for the stability estimate, let $t\ge 0$.
We have
\begin{align*}
    d(\bar\mu(t), \bar\nu(t)) &=\int_0^1 W_1(\bar\mu^\sigma(t), \bar\nu^\sigma(t))\dd\theta(\sigma) \\
    &= \int_0^1 W_1( \phi_t^\sigma(\bar\mu)_\sharp \bar\mu_0^\sigma, \phi_t^\sigma(\bar\nu)_\sharp \bar\nu_0^\sigma) \dd\theta(\sigma)\\
    &\le \int_0^1 W_1( \phi_t^\sigma(\bar\mu)_\sharp \bar\mu_0^\sigma, \phi_t^\sigma(\bar\mu)_\sharp \bar\nu_0^\sigma) \dd\theta(\sigma)\\
    &\phantom{\le} + \int_0^1 W_1( \phi_t^\sigma(\bar\mu)_\sharp \bar\nu_0^\sigma, \phi_t^\sigma(\bar\nu)_\sharp \bar\nu_0^\sigma)\dd\theta(\sigma) \\
    &\le \int_0^1 \lip(x\mapsto \phi_t^\sigma(\bar\mu)(x)) W_1(\bar\mu_0^\sigma, \bar\nu_0^\sigma)\dd\theta(\sigma) \\
    &\phantom{\le}+\int_0^1 \sup_{x\in\supp\bar\nu_0^\sigma}\modu{\phi_t^\sigma(\bar\mu)(x) - \phi_t^\sigma(\bar\nu)(x)}\dd\theta(\sigma).
\end{align*}
Let us bound the Lipschitz constant $\lip(x\mapsto \phi_t^\sigma(\bar\mu)(x))$.
For all $x,y\in \R^d$:
\begin{align*}
    \lvert \phi_t^\sigma(\bar\mu)(x)- &\phi_t^\sigma(\bar\mu)(y)\rvert \le \modu{x-y} +\int_0^t \modu{\tilde\Gamma_{\bar\mu(s)}(\sigma, \phi_s^\sigma(\bar\mu)(x)) - \tilde\Gamma_{\bar\mu(s)}(\sigma, \phi_s^\sigma(\bar\mu)(y))} \dd s \\
    &\le \modu{x - y} +  \int_0^t \norm{V(s)}_2\norm{A(s)}_2(R(s)^2+1)\modu{ \phi_s^\sigma(\bar\mu)(x)- \phi_s^\sigma(\bar\mu)(y)} \dd s
\end{align*}
using estimate \ref{appeq:ref_2_masked}.
Then, Grönwall's inequality entails that
$$\modu{ \phi_t^\sigma(\bar\mu)(x)- \phi_t^\sigma(\bar\mu)(y)} \le e^{\int_0^t \norm{V(s)}_2\norm{A(s)}_2(R(s)^2+1)\dd s} \modu{x - y}.$$
Plugging this into our previous bound for $d(\bar\mu(t), \bar\nu(t))$, together with Equation (\ref{appeq:smart_gronwall_masked}), we get:
\begin{multline*} d(\bar\mu(t), \bar\nu(t)) \le e^{\int_0^t \norm{V(s)}_2\norm{A(s)}_2(R(s)^2+1)\dd s} d(\bar\mu_0, \bar\nu_0) \\+ \int_0^t e^{\int_s^t\norm{V(u)}_2\norm{A(u)}_2(R(u)^2+1)\dd u }\frac{c_1(s)}{\theta(\{0\})} d(\bar\mu(s), \bar\nu(s))\dd s.\end{multline*}
We then apply Grönwall's inequality to $t\mapsto e^{-\int_0^t \norm{V(s)}_2\norm{A(s)}_2R(s)^2\dd s} d (\bar\mu(t), \bar\nu(t))$ to obtain
\begin{multline*}d(\bar\mu(t), \bar\nu(t)) \le d(\bar\mu_0, \bar\nu_0)  \exp\Bigg ( \int_0^t \norm{V(s)}_2\norm{A(s)}_2(R(s)^2+1) \dd s\\ +\int_0^t \frac{c_1(s)}{\theta(\{0\})} \dd s\Bigg ),\end{multline*}
which concludes the proof, recalling that $R(T) = e^{\int_0^t\norm{V(s)}_2\dd s}R_0$.}

\section{Proofs and Experiments of Section \ref{sec:gaussian}}
\label{appsec:gaussian}

\subsection{Proofs and Additional Lemmas}
\label{appsubsec:gaussian_proofs}

\begin{proof}[Proof of Proposition \ref{prop:gaussian_dynamics_trad}]
Gaussian measures stay Gaussian along the dynamics, as the pushforward of a Gaussian measure by an affine function is still Gaussian.
Let us then derive the ODEs on the expectation $\alpha$ and the covariance $\Sigma$.
For all $x\in \R^d$ and $1\le i\le d$, denote $x_i$ the $i$-th coordinate of $x$.
Set two integers $1\le i,j\le d$.
For all $t\in [0,T]$, we have
$$\Sigma(t)_{i,j} = \int (x_i - \alpha(t)_i)(x_j - \alpha(t)_j).$$
In the rest of the proof, we omit the dependence of $\alpha$ and $\Sigma$ on $t$ and write $\Gamma_\mu$ for $\Gamma_\mu\satt{SM}$.
Recall that $\mu$ satisfies
$$\partial_t \mu + \nabla_x \cdot (\mu \Gamma_\mu) = 0.$$
By multiplying this equation by $(x_i - \alpha_i) (x_j - \alpha_j)$ and then integrating by parts, we get
\begin{align*}
    \dot \Sigma_{i,j} &= \int \sum_{k=1}^d \partial_{x_k}\prt{(x_i - \alpha_i) (x_j - \alpha_j)}\Gamma_{\mu(t)}(x)_k \dd \mu(t)(x) \\
    &= \int \prt{(x_j-\alpha_j) \Gamma_{\mu(t)}(x)_i + (x_i-\alpha_i) \Gamma_{\mu(t)}(x)_j}\dd \mu(t)(x).
\end{align*}
Replacing $\Gamma_{\mu(t)}$ with its expression, given in Lemma \ref{lem:gaussian_formula_gamma}, gives
$$\dot \Sigma_{i,j} = \int (x_j - \alpha_j)(V\alpha + V\Sigma A x)_i\dd \mu(t)(x) + \int (x_i - \alpha_i)(V\alpha + V\Sigma Ax)_j\dd \mu(t)(x).$$
We compute
\begin{align*}
    \int (x_j - \alpha_j)(V\alpha + V\Sigma A x)_i\dd \mu(t)(x) &=
    \int (x_j - \alpha_j)\Big ((V\alpha)_i \\ &\phantom{blank space}+ \sum_{k=1}^d (V\Sigma A)_{i,k}x_k\Big )\dd \mu(t)(x) \\
    &= \sum_{k=1}^d (V\Sigma A)_{i,k}\int (x_j - \alpha_j)x_k \dd \mu(t)(x) \\
    &= \sum_{k=1}^d (V\Sigma A)_{i,k} \Sigma_{j,k} \\
    &= (V\Sigma A \Sigma)_{i,j}.
\end{align*}
Exchanging $i$ and $j$ we get
$$\int (x_i - \alpha_i)(V\alpha + V\Sigma Ax)_j\dd \mu(t)(x) = (V\Sigma A \Sigma )_{j,i},$$
so that
$$\dot \Sigma_{i,j} = (V\Sigma A \Sigma + \Sigma A^\top \Sigma V^\top)_{i,j}.$$
This gives the equation on the covariance.
The same strategy, but this time multiplying the PDE by $x_i - \alpha_i$ before integrating, gives the equation on $\alpha$.
\end{proof}

\

\begin{proof}[Proof of Proposition \ref{prop:trad_theoretical_result}]
To check the closed form, let us derive an ODE on $\Omega(t) \coloneqq \Sigma(t)^{-1}$.
We have that $V$ and $V^\top$ commute with $\Sigma_0$, so they also commute with $\Omega_0\coloneqq \Sigma_0^{-1}$.
Using that $\dot \Omega = - \Omega \dot \Sigma \Omega$, it is easy to check that $\Omega(t)$ satisfies the following differential equation:
$$\dot \Omega = - \Omega V \Omega^{-1} A - A^\top \Omega^{-1} V^\top \Omega.$$
Thanks to our commutation assumptions, if $\Omega(t)$ commutes with $V$ and $V^\top$, then $\Omega(t + \dd t) = \Omega(t) - \dd t (VA + A^\top V^\top)$ commutes with both $V$ and $V^\top$ as well.
Therefore, if $\Omega_0$ commutes with $V$ and $V^\top$, then the matrix $\Omega$ satisfies the equation
$$\dot \Omega = - (VA + A^\top V^\top).$$
Hence, the solution reads $\Omega(t) = \Omega_0 -t(VA + A^\top V^\top)$, which gives the expected formula for $\Sigma$.
\begin{itemize}
    \item Assume first that $VA + A^\top V^\top$ has at least one positive eigenvalue.
    For any symmetric matrix $M \in \R^{d\times d}$, denote $\lambda_1(M) \ge \dots \ge \lambda_d(M)$ its ordered eigenvalues.
    According to Weyl's inequality, we have
    $$0\le \lambda_d(\Omega(t)) \le \lambda_1(\Omega_0) - t\lambda_1(VA + A^\top V^\top),$$
    so that $\lambda_d(\Omega(t))$ has to be zero for $t\ge \lambda_1(\Omega_0) / \lambda_1(VA + A^\top V^\top) \ge 0$.
    Therefore, the matrix $\Omega$ becomes non-invertible in finite time, and $\Sigma = \Omega^{-1}$ blows up in finite time.
    \item Assume now that $VA + A^\top V^\top \preceq 0$.
    Let $1\le i\le d$.
    Again with Weyl's inequality, we have
    \begin{equation}
    \label{eq:weyl}
        \lambda_i(\Omega) \ge \lambda_d(\Omega_0) - t\lambda_{d - i}(VA + A^\top V^\top) > 0,
    \end{equation}
    so that $\Omega$ is invertible for all times, and the equation on $\Sigma = \Omega^{-1}$ has a global solution.
    Moreover, Equation (\ref{eq:weyl}) shows that if $\lambda_{d - i}(VA + A^\top V^\top) < 0$, then $\lambda_i(\Omega) \to +\infty$ when $t\to +\infty$.
    Thus, if $\lambda_i(VA + A^\top V^\top) < 0$, then $\lambda_i(\Sigma(t)) = \lambda_{d-i}(\Omega(t)) \to 0$ when $t\to +\infty$.
    Finally, if $\lambda_i(VA + A^\top V^\top) = 0$, then $\lambda_i(\Sigma) = \lambda_i(\Omega)^{-1}$ stays bounded along the dynamics.
    To prove convergence of $\Sigma(t)$, notice that the coefficients of $\Sigma$ are rational fractions of $t$, so they converge in $\R\cup\{+\infty, -\infty\}$ as $t\to +\infty$.
    Moreover, we have seen that they have to be bounded over time, as $\sum_i \lambda_i(\Sigma(t))^2 = \norm{\Sigma(t)}_F^2$ is bounded, which proves convergence of $\Sigma$ to a positive semidefinite matrix $\Sigma^*$.
\end{itemize}
\end{proof}

\

\revision{\begin{lemma}
    \label{applem:rank_preservation}
    Let $\Sigma(t)$ be a solution of the equation $\dot \Sigma = V\Sigma A\Sigma + \Sigma A^\top \Sigma V^\top$ on $[0,T)$, with initial condition $\Sigma(0) = \Sigma_0$.
    Then, for all $t\in [0,T)$, the matrix $\Sigma(t)$ has the same rank as $\Sigma_0$.
\end{lemma}}

\revision{\begin{proof}
    Denote $M_\Sigma(t)\coloneqq V\Sigma(t) A$ for all $t\in [0, T)$.
    Let $\Phi$ be the solution of
    $$\begin{cases}
        \dot \Phi(t) = M_{\Sigma}(t)\Phi(t)\\
        \Phi(0) = I_d.
    \end{cases}$$
    Then, one checks easily that $\Sigma(t) = \Phi(t) \Sigma_0\Phi(t)^\top$.
    Moreover, $\Phi(t)$ is invertible for all $t\in [0, T)$, according to Liouville's formula:
    $$\det \Phi(t) = \det(\Phi(0))\exp\left ( \int_0^t \tr(M_\Sigma(\tau))\dd \tau\right ) > 0,$$
    which allows us to conclude.
\end{proof}}

\

\begin{proof}[Proof of Proposition \ref{prop:stat_points}]
Let us first consider the case where $A$ is a diagonal matrix $\diag(\lambda_1, \dots, \lambda_d)$ with $\lambda_1 = \dots = \lambda_k = 0$, $\lambda_{k+1},\dots, \lambda_{k+ \ell} = 1$ and $\lambda_{k + \ell + 1} = \dots = \lambda_d = 0$ where $k$ and $\ell$ are respectively the multiplicity of $0$ and $1$ as eigenvalues of $A$.
Then
$$\Sigma A \Sigma = \sum_{i = k+1}^{k + \ell } C_i(\Sigma) C_i(\Sigma)^\top - \sum_{i=k+\ell + 1}^d C_i(\Sigma) C_i(\Sigma)^\top,$$
where $C_i(\Sigma)$ is the $i$-th column of $\Sigma$, seen as a column vector.
If $\Sigma A \Sigma = 0$ then
$$\sum_{i = k+1}^{k + \ell } C_i(\Sigma) C_i(\Sigma)^\top = \sum_{i=k+\ell + 1}^d C_i(\Sigma) C_i(\Sigma)^\top$$
so that
$$ \Sigma^2 = \sum_{i = 1}^d C_i(\Sigma) C_i(\Sigma)^\top = \sum_{i = 1}^k C_i(\Sigma) C_i(\Sigma)^\top + 2\sum_{i=k+1}^{k + \ell} C_i(\Sigma) C_i(\Sigma)^\top,$$
which means that the rank of $\Sigma^2$ cannot exceed $k + \ell$.
If $\ell \le d - (k + \ell)$, this proves the result.
If $\ell > d - (k + \ell)$, writing
$$\Sigma^2 = \sum_{i = 1}^d C_i(\Sigma) C_i(\Sigma)^\top = \sum_{i = 1}^k C_i(\Sigma) C_i(\Sigma)^\top + 2\sum_{i=k+\ell +1}^{d} C_i(\Sigma) C_i(\Sigma)^\top$$
allows us to conclude.

Now if $A$ is a general symmetric matrix, write $A = ODO^\top$ with $O$ an orthogonal matrix and $D = \diag(\lambda_1, \dots, \lambda_d)$ with $\lambda_1 = \dots = \lambda_k = 0$, $\lambda_{k+1}, \dots, \lambda_{k+\ell } > 0$ and $\lambda_{k + \ell + 1}, \dots, \lambda_d < 0$ (the case where one of these groups of eigenvalues is empty can be solved with the exact same method).
Then $\Sigma A \Sigma =  0$ if and only if $\tilde \Sigma \tilde D \tilde \Sigma = 0$ with $\tilde D \coloneqq \mathrm{sign}(D)$ and $$\tilde \Sigma \coloneqq \diag(1,\dots, 1,\sqrt{|\lambda_{k+1}|}, \dots, \sqrt{|\lambda_d|}) O^\top \Sigma O \diag(1, \dots, 1,\sqrt{|\lambda_{k+1}|}, \dots, \sqrt{|\lambda_d|}).$$
The first part of the proof implies that
$$\mathrm{rk} \tilde\Sigma \le \dim \ker A + \min(\sharp\{ \mathrm{positive\ eigenvalues\ of\ }A\}, \sharp\{\mathrm{negative\ eigenvalues\ of\ }A \}).$$
We conclude by noticing that the matrices $\Sigma$ and $\tilde \Sigma$ have the same rank, as $\lambda_i \neq 0$ for $i\ge k+ 1$.
\end{proof}

\

\begin{proof}[Proof of Lemma \ref{lem:well_posed_L2}]
\label{apppar:well_posed_L2}
Consider the map
$$\varphi(M) = \prt{M^{-1} + 2K^\top K}^{-1} M = (I_d + 2  K^\top K M )^{-1} M^2,$$
defined on the set $\{ M \in \R^{d\times d} : M \mbox{ symmetric s.t. } M\succeq 0 \mbox{ and } MK^\top K = K^\top K M \}$ of nonnegative matrices that commute with $K^\top K$, which is a Banach space when equipped with Frobenius norm.
This map is locally Lipschitz continuous, so Cauchy-Lipschitz theorem gives us the local existence and uniqueness of the considered Cauchy problem.

To recover global existence (and uniqueness), let us show that $\norm{\varphi(M)}_F$ grows at most linearly with $\norm{M}_F$.
Using that the squared Frobenius norm of a matrix is equal to the sum of its squared eigenvalues, we get
\begin{align*}
    \norm{\prt{M^{-1} + 2K^\top K}^{-1}}_F^2 &= \sum_{i=1}^d \frac{1}{\lambda_i(M^{-1} + 2K^\top K)^2} \\
    &= \sum_{i=1}^d \frac{1}{\prt{\lambda_i(M^{-1})+ \lambda_{\sigma(i)}(2K^\top K)}^2}
\end{align*}
for some permutation $\sigma \in \mathfrak{S}_d$, where $\lambda_i$ denotes the $i$-th eigenvalue, as the matrices $M^{-1}$ and $2K^\top K$ commute.
Then
\begin{align*}
    \norm{\prt{M^{-1} + 2K^\top K}^{-1}}_F^2 &= \sum_{i=1}^d \frac{1}{\prt{\lambda_i(M)^{-1}+ \lambda_{\sigma(i)}(2K^\top K)}^2} \\
    &= \sum_{i=1}^d \frac{\lambda_i(M)^2}{\prt{1+\lambda_{\sigma(i)}(2K^\top K)\lambda_i(M)}^2}.
\end{align*}
Now, the function 
$$x\in \R_+ \mapsto \frac{x^2}{(1 + \lambda_{\sigma(i)}(2K^\top K) x)^2}$$
is bounded on $\R_+$, so that there exists a constant $C>0$ such that
$$\norm{\prt{M^{-1} + 2K^\top K}^{-1}}_F^2 \le C^2$$ for all nonnegative symmetric matrices $M$ that commute with $K^\top K$.
Finally, as the Frobenius norm is sub-multiplicative:
$$\norm{\varphi(M)}_F \le \norm{\prt{M^{-1} + 2K^\top K}^{-1}}_F\norm{M}_F \le C \norm{M}_F,$$
which proves the claim and allows us to apply the global Cauchy theorem.
\end{proof}

\

\begin{proof}[Proof of Lemma \ref{lem:expression_gamma_sinkhorn}]
\label{apppar:sink_gaussian}
Let $c_\varepsilon(x,y)\coloneqq \frac{1}{2\varepsilon}\modu{Qx - Ky}^2 $.
Denote $\pi^*$ the minimizer of the entropic optimal transport problem
$$\min_{\pi\in \Pi(\mu, \mu)} \int c_\varepsilon(x,y)\dd \pi(x,y) + \KL(\pi \| \mu\otimes \mu),$$
with the notation of Section \ref{sec:att_variants}.
Recall that $\kappa_{\mu,\varepsilon}^\infty$, simply denoted $\kappa_\varepsilon$ in the rest of the proof, is defined as the density of $\pi^*$ with respect to the probability measure $\mu \otimes \mu$.
According to our generalization of \cite{janati2020entropic} (see Section \ref{appsec:eot_background}, Theorem \ref{appthm:extension_janati}), it holds
$$\pi^* = \ncal\prt{\begin{pmatrix}
        \alpha \\ \alpha
    \end{pmatrix}, \begin{pmatrix}
        \Sigma & A^{-1} C^\top \\ C A^{-\top} & \Sigma
    \end{pmatrix}},$$
with
$$C \coloneqq \Sigma^{1/2} \prt{\Sigma^{1/2} A \Sigma A^\top \Sigma^{1/2} + \frac{\varepsilon^2}{4}I_d}^{1/2}\Sigma^{-1/2} - \frac{\varepsilon}{2} I_d.$$
Let us compute the density of the probability measure $\kappa_\varepsilon(x, y)\dd \mu(y)$, parameterized by $x\in \R^d$.
This density is proportional to
$$e^{-\frac12 \left (\begin{smallmatrix}
    x - \alpha \\ y - \alpha
\end{smallmatrix}\right )^\top \left ( \left (\begin{smallmatrix}
        \Sigma & A^{-1}C^\top \\ CA^{-\top} & \Sigma
\end{smallmatrix}\right )^{-1} - \left (\begin{smallmatrix}
    \Sigma^{-1} & 0 \\ 0 & 0
\end{smallmatrix}\right )\right ) \left (\begin{smallmatrix}
    x - \alpha \\ y - \alpha
\end{smallmatrix}\right )}.$$
With the notation of the proof of Theorem \ref{appthm:extension_janati}, i.e.,
$$H = \begin{pmatrix}
        \Sigma & A^{-1} C^\top \\ C A^{-\top} & \Sigma
    \end{pmatrix} \quad \mbox{and} \quad \begin{cases}
        F &= \varepsilon \Sigma^{-1} + A^\top G^{-1} A \\
        G &= \varepsilon \Sigma^{-1} + A^\top F^{-1} A,
    \end{cases}$$
we have with Lemma \ref{lem:expression_H_inverse}
\begin{align*}
    \prt{ H^{-1} - \begin{pmatrix}
        \Sigma^{-1} & 0 \\ 0 & 0
    \end{pmatrix}}  &= \frac1\varepsilon \begin{pmatrix}
        F - \varepsilon\Sigma^{-1} & -A^\top \\ -A & G
    \end{pmatrix} \\
    & = \frac1\varepsilon \begin{pmatrix}
        A^\top G^{-1} A & -A^\top \\ -A & G
    \end{pmatrix}
\end{align*}
noticing that $A^\top G^{-1} A = F^{-1} - \varepsilon\Sigma^{-1}$ with Equation (\ref{eq:link_F_G}).
Then, the density of $\kappa_\varepsilon(x, y) \dd \mu(y)$ is proportional, up to factors that do not depend on $y$, to
\begin{align*}
    \exp\left (-\frac{1}{2\varepsilon} (y - \alpha - G^{-1}A(x- \alpha))^\top G (y - \alpha - G^{-1}A(x- \alpha)) \right ).
\end{align*}
Finally, we have
$$C_G = \Sigma A^\top G^{-1} A$$
according to Equation (\ref{eq:link_C_G}), so that
$$\kappa_\varepsilon(x, \cdot) \dd\mu = \ncal(\alpha + A^{-\top}\Sigma^{-1}C_G(x - \alpha), \varepsilon G^{-1})$$
and therefore
$$\Gamma_\mu(x) = \frac1\varepsilon V \mathbb E[\kappa_\varepsilon(x, \cdot) \dd\mu] = \frac1\varepsilon V(I_d - A^{-\top}\Sigma^{-1}C_G)\alpha + \frac1\varepsilon VA^{-\top}\Sigma^{-1}C_Gx.$$
\end{proof}

\subsection{Experiments}
\label{appsubsec:experiments}

\begin{figure}
\centering
$$\begin{array}{cccc}
\includegraphics[width=0.17\textwidth]{figures/axes.pdf}& \includegraphics[width=0.17\textwidth]{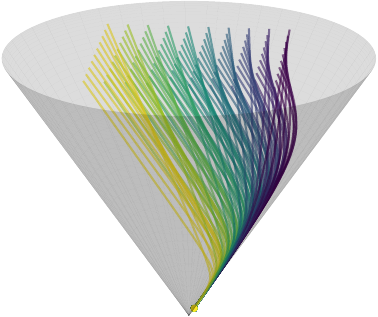} & \includegraphics[width=0.17\textwidth]{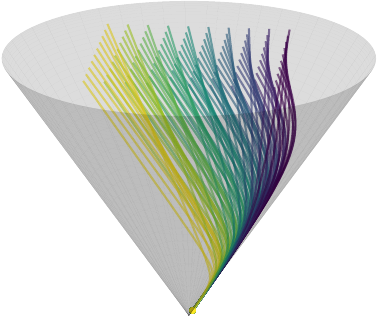}&\includegraphics[width=0.17\textwidth]{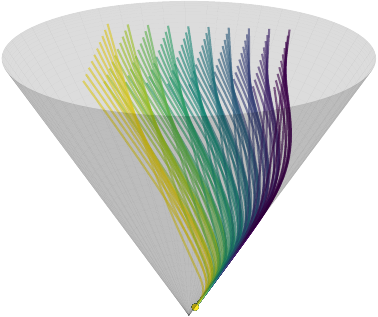}
\\
&\mbox{SM att.} & \mbox{$\revision{\ell^2}$ att.} & \mbox{MH att.} 
\end{array}$$
\caption{Comparison of the behavior of Softmax, $\revision{\ell^2}$ and Multi-head attention in the setting of Figure \ref{fig:well_posed_dynamics}. All plots correspond to the same parameters, with $V$ random and $A + A^\top \prec 0$. We observe very similar behaviors.}
\label{appfig:to_zero}
\end{figure}

\begin{figure}
\centering
$$\begin{array}{cccc}
&\mbox{(a) Convergence} & \mbox{(b) Convergence} & \mbox{(c) Convergence} \\
&\mbox{to zero} & \mbox{to a line} & \mbox{to two lines}\\
\includegraphics[width=0.17\textwidth]{figures/axes.pdf}& \includegraphics[width=0.17\textwidth]{figures/L2_neg_def.pdf} & \includegraphics[width=0.17\textwidth]{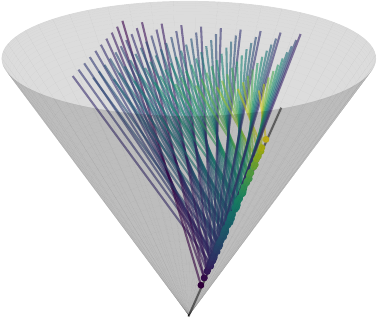}&\includegraphics[width=0.17\textwidth]{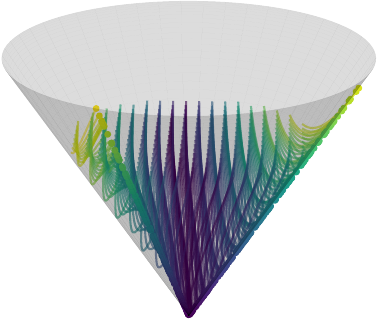}
\end{array}$$
\caption{Evolution of the covariance matrix of a 2-dimensional Gaussian measure that goes through the $\revision{\ell^2}$ Transformer PDE. All plots were obtained with $\revision{\ell^2}$ self-attention, with the same parameters as in Figure \ref{fig:well_posed_dynamics} (a, b, d). The behavior looks extremely similar as for Softmax self-attention.}
\label{appfig:well_posed_dynamics}
\end{figure}

\begin{figure}
    \centering
    $$\begin{array}{cccc}
    \Sigma \mapsto \Sigma / \tr\Sigma&\mbox{(a) Conv./blow-up} & \mbox{(b) Conv./div.} & \mbox{(c) Conv./blow-up} \\
    \includegraphics[width=0.17\textwidth]{figures/cut_cone.pdf}~~~~& \includegraphics[width=0.13\textwidth]{figures/trad_both_V.pdf} & \includegraphics[width=0.13\textwidth]{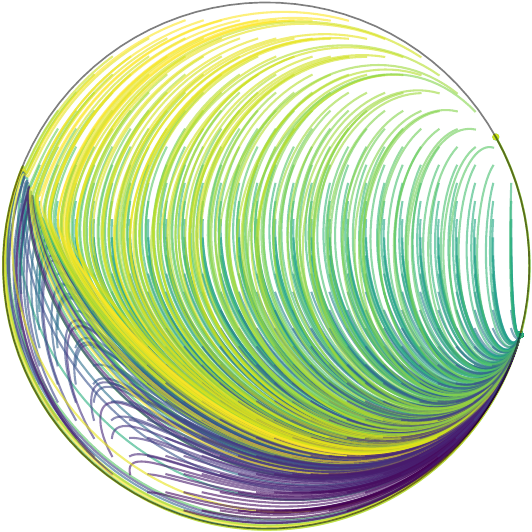}&\includegraphics[width=0.13\textwidth]{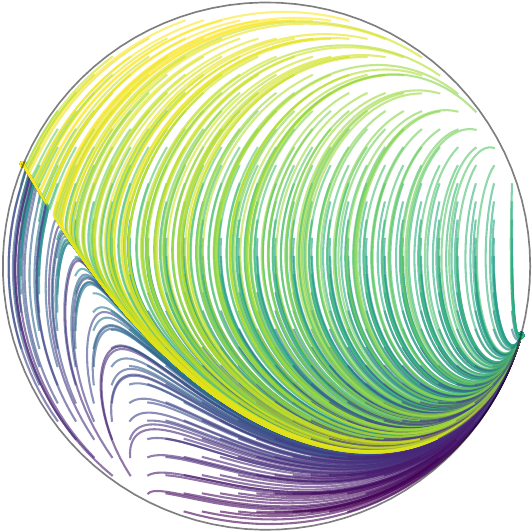}\\
    &\mbox{SM att.} & \mbox{$\revision{\ell^2}$ att.} & \mbox{MH att.} 
    \end{array}$$
    \caption{Projection on the set of trace-1 matrices of the dynamics of the covariance matrix of a Gaussian measure that goes through the Transformer PDE, in cases where curves blow up or diverge. We obtained the plots (a), (b) and (c) with the same parameters, chosen specifically to observe a division of the behavior between convergence (yellow curves) and blow-up or divergence (purple curves). $A + A^\top$ has one positive and one negative eigenvalue.}
    \label{appfig:both}
\end{figure}

\begin{figure}
    \centering
    $$\begin{array}{cc}
    \mbox{(a) Convergence} & \mbox{(b) Divergence} \\
    \includegraphics[width=0.17\textwidth]{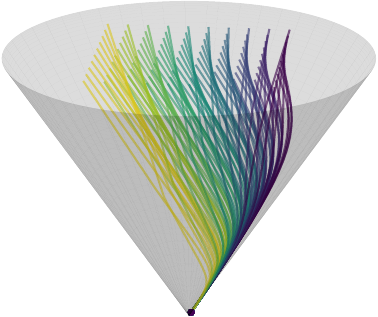}~~~~& \includegraphics[width=0.15\textwidth]{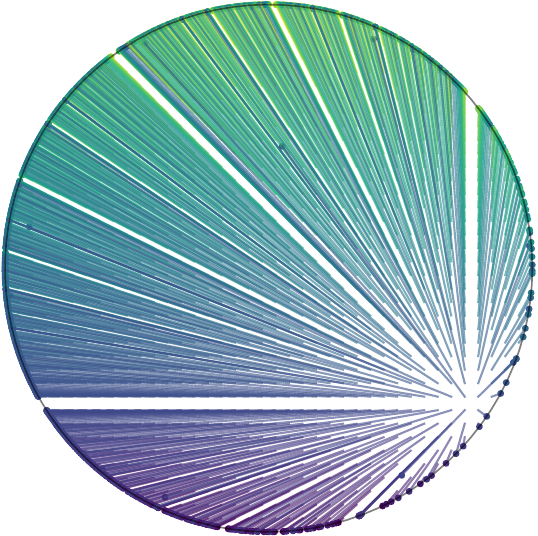}  
    \end{array}$$
    \caption{Evolution of the covariance matrix of a 2-dimensional Gaussian measure that goes through the Sinkformer PDE. Figure (a) was obtained with the same parameters as in Figure \ref{fig:well_posed_dynamics}, and Figure (b) with the same parameters as in Figure \ref{fig:explosion}. In both cases, the behavior is very similar to what happens with Softmax self-attention.}
    \label{appfig:sinkhorn}
\end{figure}

Figure \ref{appfig:well_posed_dynamics} is the parallel of Figure \ref{fig:well_posed_dynamics} for $\revision{\ell^2}$ self-attention.
Figure \ref{appfig:both} complements Figure \ref{fig:explosion} (d), and highlights that behaviors are very similar for Softmax and multi-head self-attention.
The case of $\revision{\ell^2}$ self-attention replaces finite-time divergence with infinite-time divergence.
Figure \ref{appfig:sinkhorn} plots the behavior of the Sinkformer PDE on Gaussians, in two cases already investigated for Softmax, $\revision{\ell^2}$ and multi-head self-attention. We observe a similar behavior as with Softmax self-attention.
Figure \ref{fig:histograms_L2} was obtained with the same procedure as Figure \ref{fig:histograms_trad} but for $\revision{\ell^2}$ self-attention.
The conclusion is the same: we observe that limiting points have a low rank.

\begin{figure}
    \centering
    \includegraphics[width=0.8\linewidth]{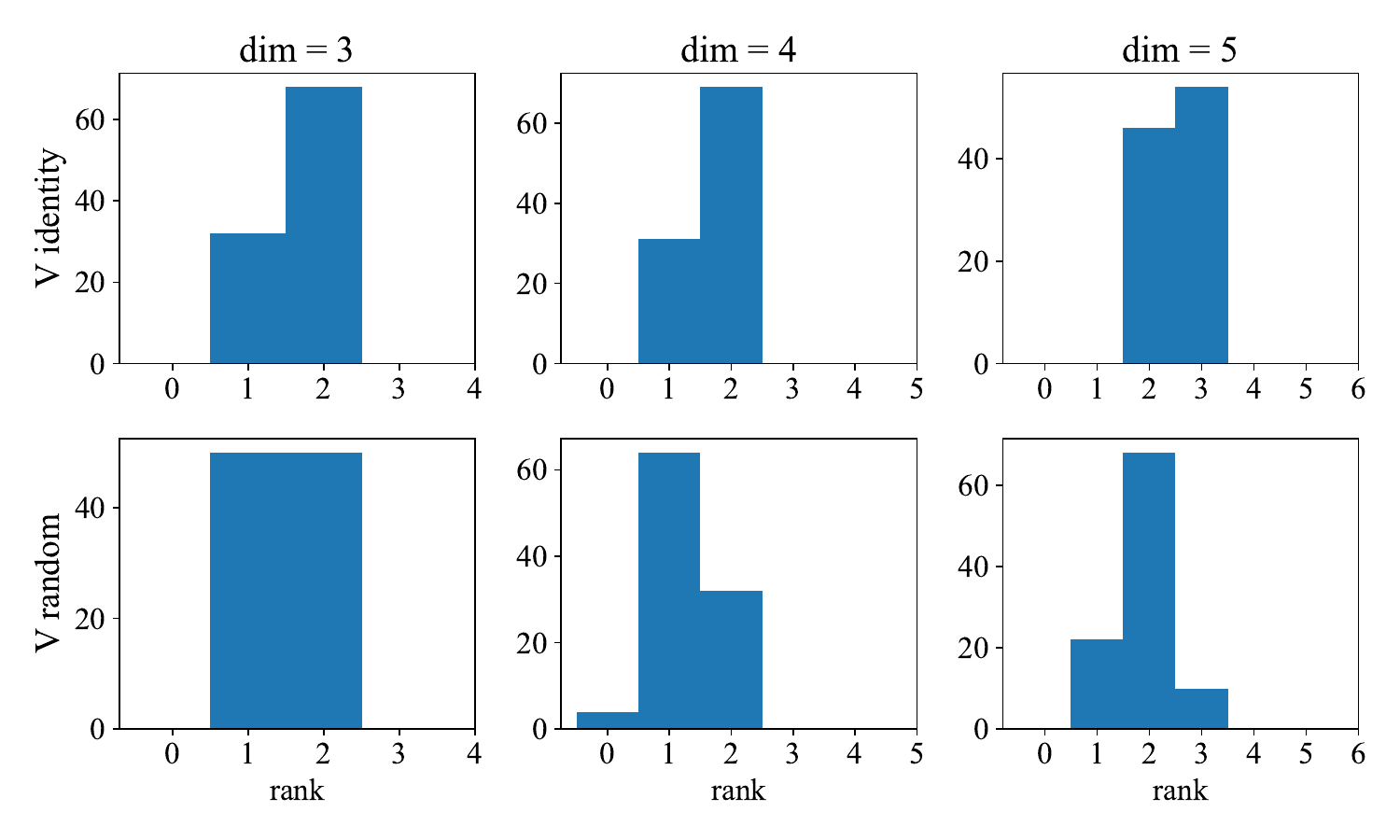}
    \caption{Histogram of the rank of limiting points of the covariance equation for $\revision{\ell^2}$ self-attention, in dimensions 3, 4, and 5. The matrix $V$ has full rank ($V = I_d$ in the upper row and $V$ random and different for each point in the lower row) and the matrix $A$ has rank $\lfloor d / 2\rfloor$, is random negative semidefinite, and is different for each point. Limiting points have a low rank (smaller than $\lceil d / 2\rceil$).}
    \label{fig:histograms_L2}
\end{figure}

\section{Proofs of Section \ref{sec:gradient_flow}}

\subsection{The Bures-Wasserstein Gradient Flow Induced by Sinkformers}
\label{appsubsec:BW_flow}

We have seen in Section \ref{subsec:gaussian_restriction} that the Sinkformer PDE initialized with a Gaussian probability measure induces a Bures-Wasserstein gradient flow in the space of Gaussian measures, associated with the energy functional
$$F_\varepsilon(\alpha, \Sigma) = \fcal_\varepsilon(\ncal(\alpha, \Sigma))$$
where $\fcal$ is defined in Equation (\ref{eq:sink_energy_functional}).
Let us compute a more explicit expression for $F_\varepsilon$.
First, for any \revision{compactly supported or Gaussian} $\mu \in \pcal(\R^d)$, we have
$$\fcal_\varepsilon(\mu) \coloneqq -\frac12 \int \kappa_{\mu, \varepsilon}^\infty\log \left (\frac{\kappa_{\mu, \varepsilon}^\infty}{\kappa_{\mu, \varepsilon}^0}\right )\dd (\mu\otimes \mu) + \frac{1}{4\varepsilon}\int (\modu{Qx}^2 + \modu{Kx}^2)\dd \mu(x)$$
and, according to Equation (\ref{eq:OT_kappa_formulation}) we have
$$OT_\varepsilon (\mu, \mu) = \frac{1}{2\varepsilon} \int \modu{Qx - Ky}^2 \kappa_{\mu, \varepsilon}^\infty(x, y)\dd \mu(x)\dd \mu(y) + \int \log(\kappa_{\mu, \varepsilon}^\infty)\kappa_{\mu, \varepsilon}^\infty \dd \mu(x) \dd \mu(y).$$
Therefore, it holds, recalling that $\kappa_{\mu, \varepsilon}^0 (x, y) = e^{-\frac{1}{2\varepsilon}\modu{Qx - Ky}^2}$:
$$\fcal_\varepsilon(\mu) = -\frac12 OT_\varepsilon(\mu, \mu) + \frac{1}{4\varepsilon}\int (\modu{Qx}^2 + \modu{Kx}^2) \dd \mu(x).$$
Now, by definition of the entropy-regularized Bures distance (\ref{eq:entropy_regularized_bures}) we have
$$OT_\varepsilon(\ncal(\alpha, \Sigma), \ncal(\alpha, \Sigma)) = \frac{1}{2\varepsilon}\mathfrak{B}_\varepsilon^2(\Sigma, \Sigma),$$
so that
$$F_\varepsilon(\alpha, \Sigma) = -\frac{1}{4\varepsilon}\mathfrak{B}_\varepsilon^2(\Sigma, \Sigma) + \frac{1}{4\varepsilon}\int (\modu{Qx}^2 + \modu{Kx}^2) \dd \mu(x).$$
Finally, writing $\modu{Qx}^2 = \tr(Qxx^\top Q^\top)$ we compute
$$\int \modu{Qx}^2 \dd \ncal(\alpha, \Sigma)(x) = \tr(Q\int xx^\top \dd \ncal(\alpha, \Sigma)(x) Q^\top) = \tr(Q\Sigma Q^\top) + \alpha^\top Q^\top Q\alpha.$$
The same computation for $\int \modu{Kx}^2 \dd \ncal(\alpha, \Sigma)(x)$ leads to
$$F_\varepsilon(\alpha, \Sigma) =  \frac{1}{4\varepsilon} \left (- \mathfrak{B}_\varepsilon^2(\Sigma, \Sigma) + \tr(Q\Sigma Q^\top) + \tr(K\Sigma K^\top) + \alpha^\top (Q^\top Q +K^\top K)\alpha \right ).$$

\subsection[The Transformer PDE as a dAV Gradient Flow]{The Transformer PDE as a $d_{A, V}$ Gradient Flow}
\label{appsubsec:attention_distance}

\paragraph{Energy Functional On Gaussians}
\label{appparagraph:gaussian_functional_trad}

We have seen in Section \ref{subsec:twisted_flow} that the Transformer PDE associated with Softmax self-attention is a gradient flow for the twisted distance $d_{A, V}$, the energy functional being
$$\fcal(\mu) = \frac12 \int e^{Qx\cdot Ky} \dd \mu(x) \dd \mu(y).$$
Let us compute the expression of $F(\alpha, \Sigma)\coloneqq \fcal(\ncal(\alpha, \Sigma))$ for $\alpha\in \R^d$ and $\Sigma$ a $d\times d$ covariance matrix.
First, we have
\begin{multline*}
    x^\top A^\top y - \frac12 (x - \alpha)^\top \Sigma^{-1} (x - \alpha) = -\frac12 (x - \alpha - \Sigma A^\top y)^\top \Sigma^{-1} (x - \alpha - \Sigma A^\top y) \\+ \alpha^\top A^\top y + \frac12 y^\top A\Sigma A^\top y.
\end{multline*}
This gives
\begin{align*}
    \int e^{x^\top A^\top y} \dd \ncal(\alpha, \Sigma) (x) &= \int e^{x^\top A^\top y - \frac12 (x - \alpha)^\top \Sigma^{-1} (x-\alpha)} \frac{\dd x}{(2\pi)^{d/2}\modu{\det \Sigma}^{1/2}}\\
    &= e^{\alpha^\top A^\top y + \frac12 y^\top A\Sigma A^\top y}.
\end{align*}
In view of computing the integral of this term in $y$, we write
\begin{multline*}
    -\frac12 (y-\alpha)^\top \Sigma^{-1}(y - \alpha) + \alpha^\top A^\top y + \frac12 y^\top A\Sigma A^\top y = \\ -\frac12 (y - (\Sigma^{-1} - A\Sigma A^\top)^{-1}(A + \Sigma^{-1})\alpha)^\top (\Sigma^{-1} - A\Sigma A^\top)(y - (\Sigma^{-1} - A\Sigma A^\top)^{-1}(A + \Sigma^{-1})\alpha) \\+ \frac12 \alpha^\top (A^\top + \Sigma^{-1})(\Sigma^{-1} - A\Sigma A^\top)^{-1}(A + \Sigma^{-1})\alpha - \frac12 \alpha^\top \Sigma^{-1}\alpha.
\end{multline*}
A similar computation as for the integral in $x$ leads to
\begin{align*} F(\alpha, \Sigma) &= \frac{e^{ \frac12 \alpha^\top ((A + \Sigma^{-1})^\top (\Sigma^{-1} - A \Sigma A^\top)^{-1}(A + \Sigma^{-1}) - \Sigma^{-1})\alpha }}{2\lvert \det(\Sigma^{-1} - A \Sigma A^\top)\rvert^{1/2}\modu{\det \Sigma}^{1/2}} \\
  &= \frac{e^{ \frac12 \alpha^\top ((A + \Sigma^{-1})^\top (\Sigma^{-1} - A \Sigma A^\top)^{-1}(A + \Sigma^{-1}) - \Sigma^{-1})\alpha }}{2\lvert \det(I_d - A \Sigma A^\top \Sigma)\rvert^{1/2}}.
\end{align*}

\

\paragraph{Geodesics of $d_{A, V}$}
\label{appparagraph:geodesics_of_twisted_distance}
We look for a characterization of geodesics for $d_{A, V}$.
We reformulate it via a Lagrange multiplier $\psi(t, x)$:
\begin{multline*}
d_{A, V}(\mu, \nu)^2 = \inf_{(\rho, v)} \sup_{\psi} \int_0^1 \int_{\R^d} \frac{v\cdot Bv}{G * \rho} \dd \rho \dd s  \\ - \int_0^1 \int_{\R^d} \left ( \partial_s \psi(s, x) + \frac{Bv(s,x)\cdot\nabla_x \psi(s, x)}{G * \rho(x)}\right )\dd \rho(x) \dd s \\+ \int_{\R^d} \psi(1, x) \dd\nu(x) - \int_{\R^d} \psi(0, x) \dd\mu(x)
\end{multline*}
where $(\rho, u)\in\ccal([0,1],\pcal_2(\R^d))\times \ccal([0,1], \ccal(\R^d, \R^d))$ in the infimum is now unconstrained.
The optimality condition on $v$ is
$v = \frac{1}{2} \nabla_x \psi.$
Incorporating this in the equation, the problem becomes
\begin{multline*} 
d_{A, V}(\mu, \nu)^2 = \inf_{\rho} \sup_\psi  \int_0^1 \int_{\R^d} - \left ( \partial_s \psi(s, x) + \frac{1}{4}\frac{B\nabla_x\psi\cdot\nabla_x \psi}{G * \rho} \right ) \dd \rho (x) \dd s \\ + \int_{\R^d} \psi(1, x) \dd\nu(x) - \int_{\R^d} \psi(0, x) \dd\mu(x).
\end{multline*}
Then, the optimality in $\rho$ gives
$$ \partial_s \psi(s, x) + \frac{1}{4} \frac{B \nabla_x \psi\cdot\nabla_x \psi}{G * \rho(x)} = \int \frac{1}{4} \frac{B\nabla_y\psi\cdot\nabla_y \psi}{(G * \rho(y))^2} G(y, x) \dd \rho(y).$$
We thus obtain the following equations characterizing geodesics for $d_{A, V}$:
\begin{equation*}
    \begin{cases}
        \partial_s \rho + \mathrm{div} \left ( \frac{B\nabla_x \psi}{2G*\rho}\rho \right ) = 0 \\
        \partial_s \psi + \frac14 \frac{\nabla_x\psi\cdot B\nabla_x\psi}{G*\rho} - \frac14 \int_{\R^d} \frac{\nabla_y\psi\cdot B\nabla_y\psi}{(G*\rho(y))^2}G(y,x)\rho(y)\dd y = 0.
    \end{cases}
\end{equation*}

\

\begin{proof}[Proof of Proposition \ref{prop:non_geo_convexity}]
\label{par:non_geo_conv}
If $\rho$ is a geodesic for $d_{A, V}$, associated with the test function $\psi$, we have
\begin{align*}\partial_t \fcal(\rho) &= \frac12 \int\int \nabla_xG(x,y) \cdot (B\nabla_x\psi) \frac{\rho(x)\rho(y)}{G*\rho(x)}\dd x\dd y \\
&= -\frac12 \int \Gamma_\rho(x) \cdot \nabla_x \psi\, \rho(x) \dd x.
\end{align*}
In the whole proof, we write $\dd \rho(y) = \rho(y) \dd y$ for simplicity—this is only a formal computation.
In order to compute $\partial^2_t \fcal(\rho)$, let us compute
\begin{align*}
    (\partial_t\Gamma_\rho(x))_i &= \int (Vy)_i G(x, y)\left(\frac{\partial_t \rho(y)}{G * \rho(x)} - \frac{\rho(y)}{(G * \rho(x))^2} \int G(x, z) \partial_t \rho (z) \dd z\right) \dd y \\
    &=\frac12 \int \nabla_y ((Vy)_i\cdot G(x, y))\cdot (B\nabla_y\psi) \frac{\rho(y)}{G*\rho(x) G*\rho(y)} \dd y \\
    &~~~~-\frac12 \int (Vy)_i G(x, y)\frac{\rho(y)}{(G*\rho(x))^2}\int \nabla_z G(x, z)\cdot (B\nabla_z\psi)\frac{\rho(z)}{G*\rho(z)}\dd z  \dd y \\
    &= \frac12 \int \prt{L_i(V) G(x, y) + (Vy)_i G(x, y) Ax} \cdot (B\nabla_y \psi) \frac{\rho(y)}{G*\rho(y)G*\rho(x)}\dd y \\
    &~~~~-\frac12 \int (Vy)_i G(x, y)\frac{\rho(y)}{(G*\rho(x))^2}\int G(x, z) (Ax)\cdot (B\nabla_z \psi) \frac{\rho(z)}{G*\rho(z)}\dd z\dd y,
\end{align*}
where $L_i(V)$ is the $i$-th row of $V$, seen as a column vector.
Then
\begin{align*}
    \partial_t \Gamma_\rho(x) &= \frac12 \int G(x, y)\prt{VB\nabla_y\psi + (Ax)\cdot (B\nabla_y\psi)Vy} \frac{\rho(y)}{G*\rho(y)G*\rho(x)}\dd y \\
    &~~~~-\frac12 \int G(x, z) (Ax)\cdot (B\nabla_z\psi) \frac{\rho(z)}{G*\rho(z)}\dd z \int Vy G(x,y) \frac{\rho(y)}{(G*\rho(x))^2}\dd y \\
    &= \frac12 \int G(x, y) \prt{VB\nabla_y\psi - x^\top V^\top\nabla_y\psi Vy}\frac{\rho(y)}{G*\rho(y)G*\rho(x)}\dd y \\
    &~~~~+\frac12 \int G(x, z)x^\top V^\top \nabla_z\psi \frac{\rho(z)}{G*\rho(x) G*\rho(z)}\dd z\Gamma_\rho(x)
\end{align*}
as $AB = -V^\top$.
Finally,
\begin{align*}
    \partial_t \Gamma_\rho(x) &= \frac12 \int G(x, y) VB\nabla_y\psi \frac{\rho(y)}{G*\rho(y)G*\rho(x)}\dd y \\
    &~~~~+\frac12 \int G(x,y)(Vx)\cdot \nabla_y\psi (\Gamma_\rho(x) - Vy)\frac{\rho(y)}{G*\rho(x)G*\rho(y)}\dd y \\
    &=\frac12 \int G(x,y)\prt{VB + (\Gamma_\rho(x) - Vy)(Vx)^\top} \nabla_y \psi \frac{\rho(y)}{G*\rho(x)G*\rho(y)}\dd y.
\end{align*}
We use this formula to compute $\partial_t^2\fcal(\rho)$:
\begin{align*}\partial_t^2 \fcal(\rho) &= -\frac12 \int (\nabla_x \partial_t \psi) \cdot \Gamma_\rho(x) \rho(x)\dd x \quad \mathrm{(1)} \\&\quad - \frac12 \int (\nabla_x\psi)\cdot(\Gamma_\rho(x))\partial_t\rho(x)\dd x \quad \mathrm{(2)} \\&\quad-\frac12\int(\nabla_x\psi)\cdot (\partial_t\Gamma_\rho)\rho(x)\dd x. \quad \mathrm{(3)}
\end{align*}
\revision{Let us compute (1). We have, using symmetry of $B$:
\begin{align*}
    \nabla_x (\partial_t \psi) &= \nabla_x \left( -\frac{1}{4} \frac{\nabla_x \psi \cdot (B \nabla_x \psi)}{G * \rho(x)} + \frac{1}{4} \int \frac{\nabla_y \psi \cdot (B \nabla_y \psi)}{(G * \rho(y))^2} G(x,y) \rho(y) \, \dd y \right) \\
    &=  - \frac{1}{4}B^{-1}\Gamma_\rho(x) \frac{\nabla_x \psi \cdot (B \nabla_x \psi) }{G * \rho(x)}  -\frac{1}{2} \frac{D_x^2 \psi B \nabla_x \psi}{G*\rho(x)} \\
    &\quad + \frac{1}{4} \int \frac{\nabla_y \psi \cdot (B \nabla_y \psi)}{(G * \rho(y))^2} G(x,y) Ay \rho(y) \, \dd y \\
    &= -\frac{1}{4} \frac{(B^{-1}\Gamma_\rho(x)(\nabla_x\psi)^\top + 2 D_x^2 \psi)B \nabla_x \psi}{G * \rho(x)} \quad \mathrm{(a)} \\
    &\quad + \frac{1}{4} \int \nabla_y \psi \cdot (B \nabla_y \psi) A y \frac{G(x,y)}{(G * \rho(y))^2} \rho(y) \, \dd y. \quad \mathrm{(b)}
    \end{align*}
We inject separately each term inside the integral that defines (1):
    \begin{align*}
    &-\frac12 \int \mathrm{(a)}\cdot \Gamma_\rho(x)\rho(x)\dd x \\&= \frac{1}{8} \int  (B \nabla_x \psi)^\top \left[ B^{-1} \Gamma_\rho(x) (\nabla_x\psi)^\top\! + 2 D_x^2\psi \right]^\top\Gamma_\rho(x) \frac{\rho(x)}{G * \rho(x)} \, \dd x  \\
    &=\frac{1}{8} \int \left[ (B \nabla_x \psi) \cdot \nabla_x \psi \ \Gamma_\rho(x) \cdot (B^{-1} \Gamma_\rho(x)) + 2 (\nabla_x \psi)^\top B (D_x^2 \psi) \Gamma_\rho(x) \right] \frac{\rho(x)}{G * \rho(x)} \dd x 
    \end{align*}
    and
    \begin{align*}
    -\frac12 \int \mathrm{(b)}&\cdot \Gamma_\rho(x)\rho(x)\dd x \\&= -\frac{1}{8} \int \left( \int \nabla_y \psi \cdot (B \nabla_y \psi) Ay \frac{G(x,y)}{(G * \rho(y))^2} \rho(y) \dd y \right) \cdot \Gamma_\rho(x) \rho(x) \dd x.
    \end{align*}
    Finally,
    \begin{align*}
    \mathrm{(1)} &= -\frac12 \int (\nabla_x \partial_t \psi) \cdot \Gamma_\rho(x) \rho(x)\dd x \\ &= \frac{1}{8} \int \left [ \Gamma_\rho(x)\cdot (B^{-1}\Gamma_\rho(x)) \nabla_x \psi + 2D_x^2\psi \Gamma_\rho(x)\right ]\cdot (B\nabla_x\psi)\frac{\rho(x)}{G*\rho(x)}\dd x\\ &\quad - \frac{1}{8} \iint \nabla_y \psi \cdot (B \nabla_y \psi) (Vy) \cdot (B^{-1}\Gamma_\rho(x)) \frac{G(x,y)}{(G * \rho(y))^2} \rho(x) \rho(y) \dd x \dd y.
\end{align*}
Let us now compute (2). We have, using that $A = -B^{-1}V= -V^\top B^{-1}$ by symmetry of $B$:
\begin{align*}
    D_x \Gamma_\rho(x) &= \frac{\int Vyy^\top A\, G(x,y)\dd\rho(y)}{G*\rho(x)} - \Gamma_\rho(x)\frac{\int y^\top A\, G(x,y)\dd \rho(y)}{G*\rho(x)}\\
    &=-\frac{\int Vyy^\top V^\top G(x,y)\dd\rho(y)}{G*\rho(x)}B^{-1} + \Gamma_\rho(x)\Gamma_\rho(x)^\top B^{-1}.
\end{align*}
Then
\begin{align*}
    (2) &= - \frac12 \int (\nabla_x\psi)\cdot(\Gamma_\rho(x))\partial_t\rho(x)\dd x \\
    &= -\frac14 \int \nabla_x(\nabla_x\psi \cdot \Gamma_\rho(x))\cdot (B\nabla_x\psi)\frac{\rho(x)}{G*\rho(x)}\dd x\\
    &=-\frac14 \int \Bigg (D_x^2\psi \Gamma_\rho(x) + B^{-1}\Gamma_\rho(x)\Gamma_\rho(x)^\top \nabla_x\psi\\ &\quad \quad- B^{-1}\frac{\int Vyy^\top V^\top G(x,y)\dd\rho(y)}{G*\rho(x)}\nabla_x\psi\Bigg )\cdot (B\nabla_x\psi)\frac{\rho(x)}{G*\rho(x)}\dd x\\
    &=-\frac14\int(D_x^2 \psi \Gamma_\rho(x))\cdot (B\nabla_x\psi)\frac{\rho(x)}{G*\rho(x)}\dd x\\
    &\quad -\frac14 \int (\nabla_x\psi\cdot \Gamma_\rho(x))^2\, \frac{\rho(x)}{G*\rho(x)}\dd x\\
    &\quad +\frac14 \iint (\nabla_x\psi\cdot Vy)^2\, G(x,y)\frac{\rho(x)\rho(y)}{(G*\rho(x))^2}\dd x\dd y.
\end{align*}
Finally, we compute (3). We have, using symmetry of $A$, so that $A^\top B = -V^\top$:
\begin{align*}
    \partial_t \Gamma_\rho &= \frac12 \int G(x,y)\left [ VB\nabla_y\psi + (Ax)\cdot(B\nabla_y\psi)Vy\right]\frac{\rho(y)}{G*\rho(y)G*\rho(x)}\dd y\\
    &\quad -\frac12 \int G(x,z)(Ax)\cdot (B\nabla_z\psi)\frac{\rho(z)}{G*\rho(z)}\dd z\int Vy G(x,y)\frac{\rho(y)}{(G*\rho(x))^2}\dd y\\
    &=\frac12 G(x,y) \left [ VB\nabla_y\psi - x^\top V^\top \nabla_y \psi Vy\right]\frac{\rho(y)}{G*\rho(x)G*\rho(y)}\dd y\\
    &\quad +\frac12 \int G(x,z)x^\top V^\top \nabla_z\psi \frac{\rho(z)}{G*\rho(x)G*\rho(z)}\dd z\ \Gamma_\rho(x)\\
    &=\frac12 \int G(x,y) VB\nabla_y\psi\frac{\rho(y)}{G*\rho(x)G*\rho(y)}\dd y\\
    &\quad +\frac12\int G(x,y)(Vx)\cdot \nabla_y\psi (\Gamma_\rho(x) - Vy)\frac{\rho(y)}{G*\rho(x)G*\rho(y)}\dd y\\
    &=\frac12 \int G(x,y) \left [ VB + (\Gamma_\rho(x)-Vy)(Vx)^\top\right] \nabla_y\psi\frac{\rho(y)}{G*\rho(x)G*\rho(y)}\dd y.
\end{align*}
Hence
\begin{align*}
    \mathrm{(3)} &= -\frac12 \int\nabla_x\psi\cdot \partial_t\Gamma_\rho \rho(x)\dd x \\
    &=-\frac14\int\nabla_x\psi \cdot \int G(x,y)[VB + (\Gamma_\rho(x) - Vy)(Vx)^\top]\nabla_y\psi \frac{\rho(x)\rho(y)}{G*\rho(x)G*\rho(y)}\dd x\dd y \\
    &=-\frac14 \iint \nabla_x\psi\cdot (VB\nabla_y\psi) G(x,y)\frac{\rho(x)\rho(y)}{G*\rho(x)G*\rho(y)}\dd x\dd y \\
    &\quad -\frac14 \iint (\nabla_x \psi)\cdot (\Gamma_\rho(x) - Vy)(Vx)\cdot \nabla_y\psi \frac{\rho(x)\rho(y)}{G*\rho(x)G*\rho(y)} \dd x\dd y.
\end{align*}}
Putting everything together leads to
\begin{align*}
    \partial_t^2\fcal(\rho) &= \frac18 \int \nabla_x\psi \cdot (B\nabla_x\psi) \Gamma_\rho(x)\cdot(B^{-1}\Gamma_\rho(x))\frac{\rho(x)}{G*\rho(x)}\dd x \\
    &~~~+\frac18 \iint \nabla_y\psi \cdot (B\nabla_y\psi)(Vy)\cdot(B^{-1}\Gamma_\rho(x))G(x,y)\frac{\rho(x)\rho(y)}{(G*\rho(y))^2}\dd x\dd y\\
    &~~~-\frac14 \int (\nabla_x \psi \cdot \Gamma_\rho(x))^2 \frac{\rho(x)}{G*\rho(x)}\dd x\\
    &~~~+\frac14 \iint (\nabla_x\psi\cdot Vy)^2 G(x,y)\frac{\rho(x)\rho(y)}{(G*\rho(x))^2}\dd x\dd y \\
    &~~~-\frac14 \iint \nabla_x\psi \cdot (VB\nabla_y\psi)G(x,y)\frac{\rho(x)\rho(y)}{G*\rho(x)G*\rho(y)}\dd x\dd y\\
    &~~~-\frac14 \iint \nabla_x\psi\cdot (\Gamma_\rho(x) - Vy)(Vx)\cdot\nabla_y\psi \frac{\rho(x)\rho(y)}{G*\rho(x)G*\rho(y)}\dd x\dd y.
\end{align*}
Let us now compute $\partial_t^2\fcal(\rho)$ when $\rho = \delta_z$ is the Dirac measure at $z\in \R^d$.
We obtain, as $\Gamma_{\delta_z}(x) = Vz$, that
$$\partial_t^2\fcal(\delta_z) = \frac14 e^{-Az\cdot z}(\nabla_z\psi)^\top \left ( (z^\top V^\top B^{-1} V z) I_d - V \right ) B\nabla_z\psi.$$
Under the assumptions of Proposition \ref{prop:non_geo_convexity}, \revision{$V=-BA$ is symmetric and commutes with $B$. Hence}
$$\partial_t^2\fcal(\delta_z) = \frac14 e^{-Az\cdot z}(\nabla_z\psi)^\top B^{1/2}\left ( (z^\top V^\top B^{-1} V z) I_d - V \right ) B^{1/2}\nabla_z\psi.$$
As $V$ has a positive eigenvalue, take $z=0$ and choose $\nabla_z\psi$ so that $B^{1/2}\nabla_z\psi$ belongs to a positive eigenspace of $V$.
Then $\partial_t^2\fcal(\delta_0)<0$.
Now, consider any (for example compactly supported) measure $\mu$, and denote \revision{$\rho$} the geodesic between $\delta_0$ and $\mu$ for $d_{A, V}$.
As $\partial_{t=0}^2\fcal(\delta_0) < 0$, the functional $\fcal$ is not convex along $\rho$, which proves the claim.
\end{proof}

\bibliographystyle{siamplain}
\bibliography{references}
\end{document}